%% file: paper-main.tex
\documentclass[ssy,nonblindrev]{informs3_homepage}

\OneAndAHalfSpacedXI

\usepackage{natbib}
\bibpunct[, ]{(}{)}{,}{a}{}{,}%
\def\bibfont{\small}%
\TheoremsNumberedBySection

\ECRepeatTheorems

\EquationsNumberedBySection 

\MANUSCRIPTNO{MS-0001-1922.65}

\usepackage[pdfencoding=auto]{hyperref}
\usepackage{subfigure}
\usepackage{comment}
\usepackage{algorithmic}
\usepackage{algorithm}
\usepackage{bm}
\usepackage{multirow}
\usepackage{booktabs}
\usepackage[svgnames,table]{xcolor}
\usepackage[tableposition=above]{caption}
\usepackage{pifont}

\newcommand{\Unif}{\mathcal{U}}
\newcommand{\vep}{\varepsilon}
\newcommand{\Ind}{\mathbf{1}}
\newcommand{\bb}[1]{\left[#1\right]}
\newcommand{\bp}[1]{\left(#1\right)}
\newcommand{\bc}[1]{\left\{#1\right\}}
\renewcommand{\P}{\mathbb{P}}
\newcommand{\E}{\mathbb{E}}
\newcommand{\hmu}{\hat{\mu}}

\def\eps{\epsilon}

\def\bmu{\bm{\mu}}

\def\betaF{{\beta_{\mathcal{F}}}}

\newcommand\IR{\mathbb{R}}

\newcommand\BR{\text{BR}}

\newcommand{\Normal}[3][xxx]{\mathcal{N}\utilsparen{#1}{#2,#3}}
\newcommand{\Bern}[2][xxx]{\mathcal{B}\utilsparen{#1}{#2}}

\begin{document}
	
	
	
	\RUNTITLE{
		Greedy Algorithms for Many Armed Bandit
	}
	
	\TITLE{
		The Unreasonable Effectiveness of Greedy Algorithms in Multi-Armed Bandit with Many Arms
	}
	
	
	\ARTICLEAUTHORS{%
		\AUTHOR{Mohsen Bayati}\AFF{Graduate School of Business, Stanford University,\EMAIL{bayati@stanford.edu}}
		\AUTHOR{Nima Hamidi}\AFF{Department of Statistics, Stanford University, \EMAIL{hamidi@stanford.edu}}
		\AUTHOR{Ramesh Johari}\AFF{Department of Management Science and Engineering, Stanford University, \EMAIL{rjohari@stanford.edu}}
		\AUTHOR{Khashayar Khosravi}\AFF{Department of Electrical Engineering, Stanford University, NYC, \EMAIL{khosravi@stanford.edu}}
	} 
	
	\ABSTRACT{
We investigate a Bayesian $k$-armed bandit problem in the \emph{many-armed} regime, where $k \geq \sqrt{T}$ and $T$ represents the time horizon. Initially, and aligned with recent literature on many-armed bandit problems, we observe that subsampling plays a key role in designing optimal algorithms; the conventional UCB algorithm is sub-optimal, whereas a subsampled UCB (SS-UCB), which selects $\Theta(\sqrt{T})$ arms for execution under the UCB framework, achieves rate-optimality. However, despite SS-UCB's theoretical promise of optimal regret, it empirically underperforms compared to a greedy algorithm that consistently chooses the empirically best arm. This observation extends to contextual settings through simulations with real-world data. Our findings suggest a new form of \emph{free exploration} beneficial to greedy algorithms in the many-armed context, fundamentally linked to a tail event concerning the prior distribution of arm rewards. This finding diverges from the notion of free exploration, which relates to covariate variation, as recently discussed in contextual bandit literature. Expanding upon these insights, we establish that the subsampled greedy approach not only achieves rate-optimality for Bernoulli bandits within the many-armed regime but also attains sublinear regret across broader distributions. Collectively, our research indicates that in the many-armed regime, practitioners might find greater value in adopting greedy algorithms.}
	
\KEYWORDS{Greedy algorithms, free exploration, multi-armed bandit}
	

	
	\maketitle

	

\input{intro.tex}

%
\input{related.tex}
%
\input{model.tex}
%
\input{lowerbounds-and-optimal-algs.tex}
%
%
\input{greedy.tex}
\input{simulations.tex}
\input{generalizations.tex}

\section{Conclusion}\label{sec:conclusion}

We first prove that in multi armed problems with many arms subsampling is a critical step in designing optimal policies. This is achieved by proving a lower bound for Bayesian regret of any policy and showing that it can be achieved when subsampling is integrated with a standard UCB algorithm. But surprisingly, through both empirical investigation and theoretical development we found that greedy algorithms, and a subsampled greedy algorithm in particular, can outperform many other approaches that depend on active exploration.  In this way our paper identifies a novel form of free exploration enjoyed by greedy algorithms, due to the presence of many arms.
As noted in the introduction, prior literature has suggested that in contextual settings, greedy algorithms can exhibit low regret as they obtain free exploration from diversity in the contexts.  An important question concerns a unified theoretical analysis of free exploration in the contextual setting with many arms, that provides a complement to the empirical insights we obtain in the preceding sections.  Such an analysis can serve to illuminate both the performance of Greedy and the relative importance of context diversity and the number of arms in driving free exploration; we leave this for future work.

%
%
%


\bibliographystyle{ormsv080}
\bibliography{refs}



\begin{APPENDICES}
	\input{appendix.tex}
\end{APPENDICES}

\end{document}

%% file: intro.tex
\section{Introduction}\label{sec:intro}

In this paper, we consider the standard stochastic multi-armed bandit (MAB) problem, in which the decision-maker takes actions sequentially over $T$ time periods (the {\em horizon}).
At each time period, the decision-maker chooses one of $k$ arms (or decisions), and receives an uncertain reward. The goal is to maximize cumulative rewards attained over the horizon.  Crucially, in the typical formulation of this problem, the set of arms $k$ is assumed to be ``small'' relative to the time horizon $T$; in particular, in standard asymptotic analysis of the MAB setting, the horizon $T$ scales to infinity while $k$ remains constant.
In practice, however, there are many situations where the number of arms is large relative to the time horizon of interest.  For example, drug development typically considers many combinations of basic substances; thus MABs for adaptive drug design inherently involve a large set of arms.  Similarly, when MABs are used in recommendation engines for online platforms, the number of choices available to users is enormous: this is the case in e-commerce (many products available); media platforms (many content options); online labor markets (wide variety of jobs or workers available); dating markets (many possible partners); etc.


Formally, we say that an MAB instance is in the \emph{many-armed regime} where $k\ge \sqrt{T}$.  In our theoretical results, we show that the threshold $\sqrt{T}$ is in fact the correct point of transition to the many-armed regime, at which behavior of the MAB problem becomes qualitatively different than the regime where $k < \sqrt{T}$.  Throughout our paper, we consider a Bayesian framework, i.e., where the arms' reward distributions are drawn from a prior.

In \S \ref{subsec:low}, we first use straightforward arguments to establish a fundamental lower bound of $\Omega(\sqrt{T})$ on Bayesian regret in the many-armed regime.
We note that prior Bayesian lower bounds for the stochastic MAB problem require $k$ to be fixed while $T \to \infty$ (see, e.g., \citealp{kaufmann2018bayesian, lai1987adaptive, lattimore2018bandit}), and hence, are not applicable in the many-armed regime.

Our first observation (see \S \ref{subsec:upp}), aligned with recent \citep{katzsamuels2020true,ren2019exploring} and concurent literature \citep{zhu2020onregret}, is the importance of subsampling.  The standard UCB algorithm can perform quite poorly in the many-armed regime, because it over-explores arms: even trying every arm once leads to a regret of $\Omega(k)$. Instead, we see that the $\Omega(\sqrt{T})$ bound is achieved (up to logarithmic factors) by a subsampled upper confidence bound (SS-UCB) algorithm, where we first select $m$ (e.g., $m=\sqrt{T}$) arms uniformly at random, and then run  a standard UCB algorithm \citep{lai1985asymptotically,auer2002finite} with just these arms.

Given the primary goal of this paper, which is to examine empirical performance, our numerical investigation unveils intriguing behaviors. In Figure \ref{fig:alph_75}, we simulate several different algorithms across 400 simulations for two pairs of $T, k$ in the many-armed regime.\footnote{Our code is available at \href{http://github.com/khashayarkhv/many-armed-bandit}{http://github.com/khashayarkhv/many-armed-bandit}.} Notably, the greedy algorithm (Greedy)---an algorithm that pulls each arm once and thereafter pulls the empirically best arm for all remaining times—performs exceptionally well. This is despite the well-known fact that the Greedy algorithm can suffer linear regret in the standard Multi-Armed Bandit (MAB) problem, as it might too early fixate on a suboptimal arm. Consistent with the initial observation, subsampling enhances the performance of all algorithms, including UCB, Thompson sampling (TS), and Greedy. Specifically, the subsampled greedy algorithm (SS-Greedy) surpasses all other algorithms in performance.
\begin{figure*}[htpb]
  \centering
  \subfigure{\includegraphics[width=0.47\textwidth]{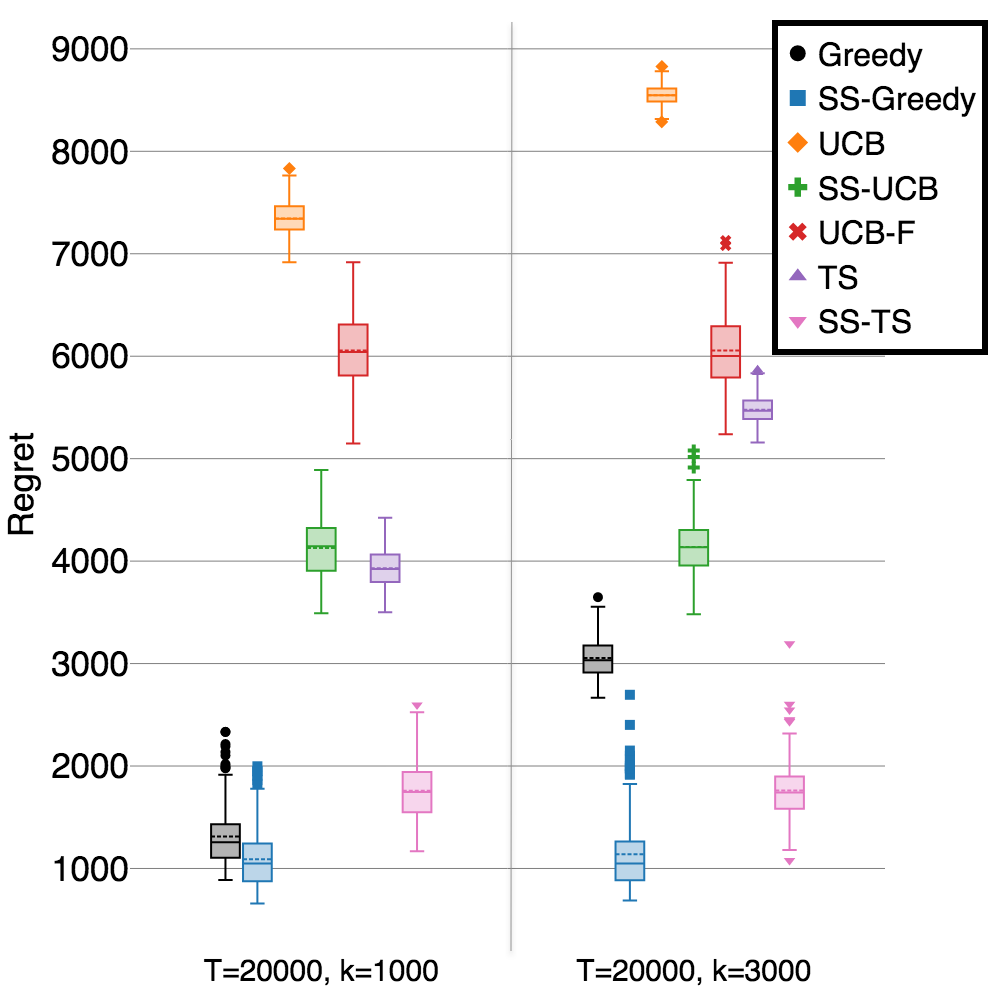}}
\hfill
  \subfigure{\includegraphics[width=0.47\textwidth]{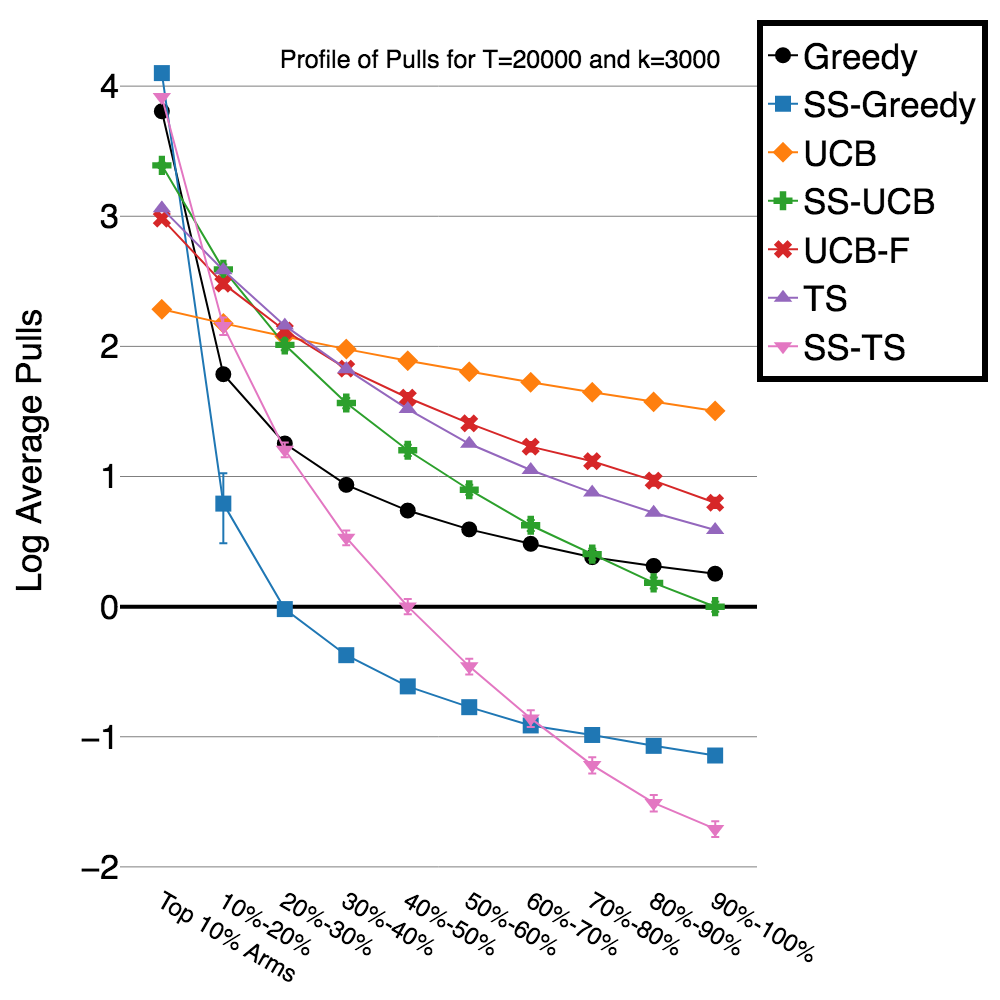}}
  \caption{
  	Distribution of the per-instance regret (on left) and profile of arm pulls in logarithmic scale based on arms index (on right). Rewards are generated according to $\mathcal{N}(\mu_i,1)$, with $\mu_i$ are iid uniform samples from $[0,1]$. The list of algorithms included is as follows. (1) UCB: Algorithm \ref{alg:ucb-asymp}, (2) SS-UCB: Algorithm \ref{alg:subs-ucb} with $m = \sqrt{T}$,
  	(3) Greedy: Algorithm \ref{alg:greedy}, (4) SS-Greedy: Algorithm \ref{alg:subs-greedy} with $m = T^{2/3}$ (see Theorem \ref{thm:greedy-reg}),
  	(5) UCB-F: UCB-F algorithm of \cite{wang2009algorithms} with the choice of confidence set $\mathcal{E}_t = 2 \log (10 \log t)$,
  	(6) TS: Thompson Sampling algorithm \cite{thompson1933likelihood,russo2014learning,agrawal2012analysis}, and (7) SS-TS: subsampled TS with $m = \sqrt{T}$.
}
  \label{fig:alph_75}
\end{figure*}

The right panel in Figure \ref{fig:alph_75} shows that Greedy and SS-Greedy benefit from a novel form of \emph{free exploration}, that arises due to the availability of a large number of near-optimal arms. This free exploration helps the greedy algorithms to quickly discard sub-optimal arms that are substantially over-explored by
algorithms with ``active exploration'' (i.e., UCB, TS, and their subsampled versions).  We emphasize that this source of free exploration is distinct from that observed in recent literature on contextual bandits (see, e.g., \citealp{bastani2017mostly, kannan2018smoothed, raghavan2018externalities, hao2019adaptive}), where free exploration arises due to diversity in the context distribution.
Our extensive simulations in \S \ref{sec:simulations} and in Appendix \ref{app:add-sim} show that these insights are robust to varying rewards and prior distributions. Indeed, similar results are obtained with Bernoulli rewards and general beta priors.
Further, using simulations, we also observe that the same phenomenon arises in the contextual MAB setting, via simulations with synthetic and real-world data.

Motivated by these observations, in \S \ref{sec:greedy} and \S \ref{sec:generalizations} we embark on a theoretical analysis of Greedy in the many-armed regime to complement our empirical investigation and clarify the source of free-exploration. We demonstrate that, with high probability, one of the arms on which Greedy focuses is likely to have a high mean reward (as also observed in the right panel of Figure \ref{fig:alph_75}). Our proof technique employs the Lundberg inequality to relate the likelihood of this event to the distribution of the ruin event of a random walk, which may be of independent interest for analyzing the performance of greedy algorithms in other contexts. Using this result, we show that for Bernoulli rewards, the regret of Greedy is ${O}\big(\max(k, T/k)\big)$; notably, for $k \geq \sqrt{T}$, SS-Greedy is optimal (and for $k = \sqrt{T}$, Greedy is optimal). For more general reward distributions, we demonstrate that, under a mild condition, an upper bound on the regret of Greedy is ${O}\big(\max(k,T/\sqrt{k})\big)$. Thus, theoretically, for general reward distributions in the many-armed regime, Greedy achieves \emph{sublinear}, though not optimal, regret.

Our theoretical findings shed light on why Greedy and SS-Greedy algorithms exhibit strong performance in our numerical experiments, attributed to the novel form of free exploration identified in the many-armed regime. Although our theoretical analyses do not prove the universal rate optimality of SS-Greedy, the discrepancy between regret bounds and empirical outcomes suggests that regret bounds do not capture the entire narrative. Indeed, the robust empirical performance of Greedy and SS-Greedy algorithms underscores, from a practical perspective, the potential preference for greedy algorithms in the many-armed regime, as supported by both our empirical and theoretical insights. This recommendation becomes even more compelling when considering contextual settings, where such algorithms likely benefit further from free exploration due to the diversity of contexts, as previously mentioned.



%% file: related.tex
\subsection{Related Work}\label{sec:related}

The literature on stochastic MAB problems with a finite number of arms is vast; we refer the reader to recent monographs by \cite{lattimore2018bandit} and \cite{slivkins2019intro} for a thorough overview.  Much of this work carries out a frequentist regret analysis. In this line, our work is most closely related to work on the {\em infinitely} many-armed bandit problem, first studied by \cite{berry1997bandit} for Bernoulli rewards. They provided algorithms with $O(\sqrt{T})$ regret, and established a $\sqrt{2 T}$ lower-bound in the Bernoulli setting while a matching upper bound is proved by \cite{bonald2013two}. In \citep{wang2009algorithms}, the authors studied more general reward distributions and proposed an optimal (up to logarithmic factors) algorithm called UCB-F that is constructed based on the UCB-V algorithm of \cite{audibert2007tuning}. In fact, our results in \S \ref{sec:LB-and-opt} also leverage ideas from \citep{wang2009algorithms}. The analysis of the infinitely many-armed bandit setting was later extended to simple regret \citep{carpentier2015simple} and quantile regret minimization \citep{chaudhuri2018quantile}. Adjustments of confidence bounds in the setting where number of arms scales as $k = T^{\zeta}$ was studied in \citep{chan2019optimal}. In a related work, \cite{russo2018satisficing} proposed using a variant of Thompson Sampling for finding ``satisficing'' actions in the complex settings where finding the optimal arm is difficult.

Recent works by \cite{katzsamuels2020true}, \cite{ren2019exploring}, and \cite{zhu2020onregret} have studied the concept of subsampling in many-armed settings. \cite{katzsamuels2020true} demonstrated the advantage of integrating subsampling with the UCB algorithm to enhance the sample complexity for identifying high-performing arms in multi-armed bandits with a large number of arms. Similarly, \cite{ren2019exploring} investigated sampling from a potentially infinite set of arms, combined with confidence bounds, to identify a fixed number of arms with rewards exceeding a specific quantile threshold of the underlying reward distribution, while aiming to reduce the sample complexity. The concurrent paper of \cite{zhu2020onregret}, akin to the first part of our paper, examined a bandit problem with a large number of arms that grows with the time horizon and focused on subsampling arms to concentrate the bandit problem on them. However, the focus diverges thereafter. \cite{zhu2020onregret} investigated a frequentist setting and merged subsampling with the MOSS algorithm from \cite{audibert2009minimax} to derive adaptive algorithms that are agnostic to the hardness of the problem. In contrast, our paper studies a Bayesian setting and focuses on the advantages of greedy algorithms.

Our results complement the existing literature on Bayesian regret analysis of the stochastic MAB.  The literature on the Bayesian setting goes back to
index policies of \cite{gittins1979bandit} that are optimal for the infinite-horizon discounted reward setting. Bayesian bounds for a similar problem like ours, but when $k$ is fixed and $T \rightarrow \infty$ were established in \cite{kaufmann2018bayesian}; their bounds generalized the earlier results of \cite{lai1987adaptive}, who obtained similar results under more restrictive assumptions. 

Several other papers provide fundamental bounds in the fixed $k$ setting. Bayesian regret bounds for the Thompson Sampling algorithm were provided in \citep{russo2014learning} and information-theoretic lower bounds on Bayesian regret for fixed $k$ were established in \citep{russo2016information}. 
Finally, \cite{russo2014inf} proposed to choose policies that maximize information gain, and provided
regret bounds based on the entropy of the optimal action distribution. 

Very recently, and building on our results, \cite{Jedor2021greedy}  provide additional support for our main recommendation that employing greedy algorithms in practice may be preferred when the number of arms is large.


%% file: model.tex
\section{Problem Setting and Notation}\label{sec:model}

We consider a Bayesian $k$-armed stochastic bandit setting where a decision-maker sequentially pulls from a set of unknown arms, and aims to maximize the expected cumulative reward generated.  In this section we present the technical details of our model and problem setting.  Throughout, we use the shorthand that $[n]$ denotes the set of integers $\{1,\ldots,n\}$. We also use the notations $O(\cdot)$, $\Theta(\cdot)$, and $\Omega(\cdot)$ to represent asymptotic scaling of different parameters \citep{de1981asymptotic}. And we use $\tilde{O}(\cdot)$ instead of $O(\cdot)$ when the asymptotic relationship holds up to logarithmic factors.

\paragraph{Time.}  Time is discrete, denoted by $t = 1, \ldots, T$; $T$ denotes the time horizon. 
Motivated by recent industrial experimentation platforms that the sample size is fixed, we assume $T$ is known \citep{cooprider2023science}.

dThroughout this paper we assume $T$ is known. 

\paragraph{Arms.}  At each time $t$, the decision-maker chooses an arm $a_t$ from a set of $k$ arms.  

\paragraph{Rewards.}  Each time the decision maker pulls an arm, a random reward is generated.  We assume a Bayesian setting that arm rewards have distributions with parameters drawn from a common prior. Let $\mathcal{F} = \{P_\mu: \mu \in [0,1]\}$ be a collection of reward distributions, where each $P_\mu$ has mean $\mu$.  Further, let $\Gamma$ be a prior distribution on $[0,1]$ for $\mu$; we assume $\Gamma$ is absolutely continuous w.r.t.~Lebesgue measure in $\IR$, with density $g$. For example, $\mathcal{F}$ might be the family of all binomial distributions with parameters $\mu\in[0,1]$, and $\Gamma$ might be the uniform distribution on $[0,1]$.\footnote{
Our results can be extended to the case where the support of $\Gamma$ is a bounded interval $[a,b]$.} The following definition adapted from the infinitely-many armed bandit literature (see, e.g. \citealp{wang2009algorithms, carpentier2015simple}) is helpful in our analysis.
\begin{definition}[$\beta$-regular distribution]\label{def:beta-prior}
Distribution $Q$ defined over $[0,1]$ is called $\beta$-regular if 
\[
\P_Q[\mu > 1- \eps] = \Theta(\eps^\beta)
\] 
when $\eps$ goes to $0$. Equivalently, there exists positive constants $c_0<C_0$ such that
\begin{equation*}
	c_0 \eps^\beta \leq \P_Q (\mu > 1-\eps) \leq \P_Q(\mu \geq 1-\eps) \leq C_0 \eps^\beta\,.
\end{equation*}
\end{definition}
For simplicity, throughout the paper, we assume that $\Gamma$ is $1$-regular. We discuss how our results can be generalized to an arbitrary $\beta$ in \S \ref{sec:generalizations} and \ref{app:gen-prior-details}. It is also noteworthy that, in practical settings where companies conduct a large number of experiments across their business \citep{kohavi2020trustworthy}, access to an estimate of the prior distribution is available, allowing for the estimation of the parameter $\beta$.

\begin{assumption}\label{ass:prior}
	The distribution $\Gamma$ is $1$-regular.
\end{assumption}
Assumption \ref{ass:prior} puts a constraint on $\mathbb{P}[\mu \geq 1-\epsilon]$, which quantifies how many arms are $\epsilon$-optimal. The larger number of $\epsilon$-optimal arm means it is more likely that Greedy concentrates on an $\epsilon$-optimal arm which is one of main components of our theoretical analysis (see Lemma \ref{lem:greedy-gen}). 

We also assume that the reward distributions are $S^2$-subgaussian as defined below and for notation simplicity we assume $S=1$, but generalization of our results to any $S$ is straightforward.
\begin{assumption}\label{ass:reward}
Every $P_\mu \in \mathcal{F}$ is $1$-subgaussian: for any $\mu \in [0,1]$ and any $t$, if $Z_\mu$ is distributed according to $P_\mu$, then 
\[
\E\left[\exp\Big(t(Z_\mu - \mu)\Big)\right] \leq \exp(t^2/2)\,.
\]
\end{assumption}
Given a realization $\bmu = (\mu_1, \mu_2, \ldots, \mu_k)$ of reward means for the $k$ arms, let $Y_{it}$ denote the reward upon pulling arm $i$ at time $t$.  Then $Y_{it}$ is distributed according to $P_{\mu_i}$, independent of all other randomness; in particular, $\E[Y_{it}] = \mu_i$. Note that $Y_{a_t, t}$ is the actual reward earned by the decision-maker.  As is usual with bandit feedback, we assume the decision-maker only observes $Y_{a_t,t}$, and not $Y_{it}$ for $i \neq a_t$.

\paragraph{Policy.}  Let $H_t = (a_1, Y_{a_1,1}, \ldots, a_{t-1}, Y_{a_{t-1}, t-1})$
denote the history of selected arms and their observed rewards up to time $t$. Also, let $\pi$ denote the decision-maker's policy (i.e., algorithm) mapping $H_t$ to a (possibly randomized) choice of arm $a_t \in [k]$.  In particular, $\pi(H_t)$ is a distribution over $[k]$, and $a_t$ is distributed according to $\pi(H_t)$, independently of all other randomness.

\paragraph{Goal.}  Given a horizon of length $T$, a realization of $\bmu$, and the realization of actions and rewards, the realized {\em regret} is 
\[
T \max_{i = 1}^k \mu_i - \sum_{t = 1}^T Y_{a_t, t}\,.
\]
We define $R_T$ to be the expectation of the above quantity with respect to randomness in the rewards and the actions, given the policy $\pi$ and the mean reward vector $\bmu$,
\begin{equation*}
 R_T(\pi \mid \bmu) = T\, \max_{i=1}^k \mu_i - \sum_{t=1}^T \E[\mu_{a_t} | \pi ]\,.
\end{equation*}
Here the notation $\E[\cdot | \pi]$ is shorthand to indicate that actions are chosen according to the policy $\pi$, as described above; the expectation is over randomness in rewards and in the choices of actions made by the policy.  In the sequel, the dependence of the above quantity on $k$ will be important as well; we make this explicit as necessary.

The decision-maker's goal is to choose $\pi$ to minimize her Bayesian expected regret, i.e., where the expectation 
is taken over the prior as well as the randomness in the policy.  In other words, the decision-maker chooses $\pi$ to minimize
\begin{equation}
\BR_{T,k}(\pi) = \E\left[R_T(\pi \mid \bmu)\right]\,.
\end{equation}


%
%
%
%

\paragraph{Many arms.}  In this work, we are interested in the setting where $k$ and $T$ are comparable. In particular, we focus on the scaling of $\BR_{T,k}$ in different regimes for $k$ and $T$.



%% file: lowerbounds-and-optimal-algs.tex

\section{Lower Bound and an Optimal Algorithm}\label{sec:LB-and-opt}

In this section, we first show that Bayesian regret of any policy is larger than the minimum of $k$ and $\sqrt{T}$ up to constants, in \S \ref{subsec:low}. Then, in \S \ref{subsec:upp}, we show how this lower bound can be achieved, up to logarithmic factors, by integrating subsampling into the UCB algorithm.

\subsection{Lower Bound}\label{subsec:low}

Our lower bound on $\BR_{T,k}$ is stated next.
\begin{theorem}\label{thm:lb}
    Consider the model described in \S \ref{sec:model}. Suppose that Assumption \ref{ass:prior} holds. Then, there exist absolute constants $c_D$ and $c_L$ such that for any policy $\pi$ and $T, k \geq c_D$, we have
    \begin{equation*}
        \BR_{T,k}(\pi) \geq c_L \min(\sqrt{T},k) \,.
     \end{equation*}
\end{theorem}
This theorem shows that the Bayesian regret of an optimal algorithm should scale as $\Omega(k)$ when $k < \sqrt{T}$ and as $\Omega(\sqrt{T})$ if $k > \sqrt{T}$. The proof idea is to show that for any policy $\pi$, there is a class of ``bad arm orderings" that occur with constant probability for which a regret better than $\min(\sqrt{T},k)$ is not possible. The complete proof is provided in \S \ref{app:prelim-lb}. Interestingly, this theorem does not require any additional assumption on reward distributions beyond Assumption \ref{ass:prior} on the prior distribution of reward means.

\subsection{An Optimal Algorithm}\label{subsec:upp}
In this section we describe an algorithm that achieves the lower bound of \S \ref{subsec:low}, up to logarithmic factors.  Recall that we expect to observe two different behaviors depending on whether $k < \sqrt{T}$ or $k > \sqrt{T}$; Theorems \ref{thm:upp-small-k} and \ref{thm:upp-large-k} state our result for these two cases, respectively. In particular, Theorem \ref{thm:upp-large-k} shows that subsampling is a necessary step in the design of optimal algorithms in the many-armed regime. Proofs of these theorems are provided in \S \ref{app:prelim-up}. Note that for Theorem \ref{thm:upp-small-k}, instead of Assumption \ref{ass:prior}, we require the density $g$ (of $\Gamma$) to be bounded from above.

We require several definitions.  For $i \in [k]$, define:
\[
N_i(t) := \sum_{s = 1}^t \mathbf{1}(a_s = i)~~~ \text{and}~~~ \hat{\mu}_i(t) := \frac{\sum_{s = 1}^t Y_{is} \mathbf{1}(a_s = i)}{N_i(t)}\,,
\]
where $\mathbf{1}(\cdot)$ is the indicator function. Thus $N_i(t)$ is the number of times arm $i$ is pulled up to time $t$, and $\hat{\mu}_i(t)$ is the empirical mean reward on arm $i$ up to time $t$.  (We arbitrarily define $\hat{\mu}_i(t) = 1$ if $N_i(t) = 0$.) Also define, 
\[
f(t) := 1 + t \log^2(t)\,.
\]

\subsubsection{Case $k < \sqrt{T}$.}

In this case, we show that the UCB algorithm (see, e.g., Chapter 8 of \citep{lattimore2018bandit}) is optimal, up to logarithmic factors. For completeness, this algorithm is restated as Algorithm \ref{alg:ucb-asymp}.
\begin{theorem}\label{thm:upp-small-k}
Consider the setting described in \S \ref{sec:model}. Suppose that Assumptions \ref{ass:reward} holds and that there exists $D_0$ such that for all $x \in [0,1], g(x) \leq D_0$. Then, Bayesian regret of Algorithm \ref{alg:ucb-asymp} satisfies
\begin{equation*}
    \BR_{T,k}(\text{UCB})
\leq k \left[1 + D_0 + D_0\Big(10 + 18\log f(T)\Big) (2+ 2 \log k + \log T)\right]\,.
\end{equation*}
\end{theorem}

\begin{remark}
The assumption that density of $g$ is bounded by a constant $D_0$ is motivated by the fact that the regret of the UCB algorithm is inversely proportional to the gaps between the true mean reward of the top arm and the remaining arms. Smaller gaps lead to worse regret results. Thus, the assumption does not rule out difficult cases. However, by excluding cases with unbounded $g$, the Bayesian analysis becomes cleaner.
\end{remark}

	\begin{algorithm}[H]
		\centering
		\caption{Asymptotically Optimal UCB}\label{alg:ucb-asymp}
		\begin{algorithmic}[1]%
			\FOR {$t \leq k$}
			\STATE Pull $a_t = t$
			\ENDFOR
			\FOR {$t \geq k+1$}
			\STATE Pull $a_t = \argmax_{i\in[k]} \left[\hmu_i(t-1) + \sqrt{\frac{2 \log f(t)}{N_i(t-1)}}\,\right]$.
			\ENDFOR
		\end{algorithmic}
	\end{algorithm}

\subsubsection{Case $k > \sqrt{T}$.}

For large $k$, because of the first $k$ time periods that each arm is pulled once, UCB incurs $\Omega(k)$ regret which is not optimal. This fact suggests that subsampling arms is a required step of any optimal algorithm. In fact, we show that in this case the subsampled UCB algorithm (SS-UCB) is optimal (up to logarithmic factors). 
\begin{theorem}\label{thm:upp-large-k}
Consider the setting described in \S \ref{sec:model}. Let assumptions \ref{ass:prior} and \ref{ass:reward} hold. Then, Bayesian regret of the subsampled UCB (Algorithm \ref{alg:subs-ucb}), when executed with subsampling paramter $m$ equal to $\lceil{ \sqrt{T} \rceil}$, satisfies
\begin{equation*}
	\BR_{T,k} (\text{SS-UCB}) \leq 2 + \sqrt{T} \left[\frac{\log T}{c_0} + \frac{C_0\log^2 T}{c_0^2} + C_0\Big(20+36 \log f(T)\Big)\Big(5 + \log (\sqrt{T} c_0) - \log \log T\Big) \right]\,.
\end{equation*}
\end{theorem}
%
	\begin{algorithm}[H]
		\centering
		\caption{Subsampled UCB (SS-UCB)}\label{alg:subs-ucb}
		\begin{algorithmic}[1]
			\STATE \textbf{Input:} $m$: subsampling size
			\STATE Let $\mathcal{S}$ be a set of $m$ arms, selected uniformly at random without replacement from $[k]$
			\STATE Run UCB (Algorithm \ref{alg:ucb-asymp}) on arms in set $\mathcal{S}$
		\end{algorithmic}
	\end{algorithm}
%
Finally, note that when $m=k$, Algorithms \ref{alg:ucb-asymp} and \ref{alg:subs-ucb} are equivalent. Therefore, Algorithm \ref{alg:subs-ucb} with $m=\min(\lceil\sqrt{T}\rceil,k)$ is asymptotically optimal for both of the above cases. 

%% file: greedy.tex

\section{A Greedy Algorithm}\label{sec:greedy}

While in \S \ref{sec:LB-and-opt} we showed that SS-UCB enjoys the so called ``rate optimal'' performance, the simulations of Figure \ref{fig:alph_75} in \S \ref{sec:intro} suggest that a {\em greedy} algorithm and its subsampled version are superior. Motivated by this observation, in this section we embark on analyzing this greedy algorithm to shed some light on its empirical performance.
The greedy algorithm pulls each arm once and from then starts pulling the arm with the highest estimated mean; the formal definition is shown in Algorithm \ref{alg:greedy}.  We can also define a subsampled greedy algorithm that randomly selects $m$ arms and executes Algorithm \ref{alg:greedy} on these arms. This subsampled form of greedy is stated as Algorithm \ref{alg:subs-greedy}. As in previous section, we denote the estimated mean of arm $i$ at time $t$ by $\widehat{\mu}_{i}(t)$.
%
	\begin{algorithm}[H]
		\centering
		\caption{Greedy}\label{alg:greedy}
		\begin{algorithmic}[1]
			\FOR {$t \leq k$}
			\STATE Pull arm $a_t = t$
			\ENDFOR
			\FOR {$t \geq k+1$}
			\STATE Pull arm $a_t = \argmax_{i\in[k]} \hmu_i(t-1)$
			\ENDFOR
		\end{algorithmic}
	\end{algorithm}
	\begin{algorithm}[H]
		\centering
		\caption{Subsampled Greedy (SS-Greedy)}\label{alg:subs-greedy}
		\begin{algorithmic}[1]
			\STATE \textbf{Input:} $m$: subsampling size

			\STATE Let $\mathcal{S}$ be a set of $m$ arms, selected uniformly at random without replacement from $[k]$
			\STATE Run Greedy (Algorithm \ref{alg:greedy}) on arms in set $\mathcal{S}$
		\end{algorithmic}
	\end{algorithm}

Our road map to analyze the above algorithms is as follows. First, in
\S \ref{subsec:upper-bd-greedy} , we prove a general upper bound on regret of Algorithm \ref{alg:greedy} that depends on probability of a tail event for a certain random walk, with steps that are distributed according to the arms' reward distribution $P_\mu$. This result opens the door to prove potentially sharper regret bounds for specific families of reward distributions, if one can prove tight bounds for the aforementioned tail event. We provide two examples of this strategy in \S \ref{subsec:bernoulli}-\ref{subsec:upper-bd-greedy}. In \S \ref{subsec:bernoulli}, we consider Bernoulli reward functions and prove a rate-optimal regret bound of $\tilde{O}(\sqrt{T})$ for Greedy when $k= \Theta(\sqrt{T})$ and for SS-Greedy when $k\ge \sqrt{T}$. In the remaining parts we obtain sub-linear regrets for more general family of distributions.

\subsection{A General Upper Bound on Bayesian Regret of Greedy}\label{subsec:upper-bd-greedy}

We require the following definition.  For a fix $\theta$ and $\mu > \theta$, let  $\bc{X_i}_{i=1}^\infty$ be a sequence of i.i.d.~random variables with distribution $P_\mu$.  Let $M_n := \sum_{i = 1}^n X_i/n$, and define $q_\theta(\mu)$ as the probability that the sample average always stays above $\theta$,
\begin{equation}\label{eqn:max-cross}
q_\theta(\mu) := \P \left[M_n > \theta \text{ for all } n\right]\,.
\end{equation}
In other words, $q_\theta(\mu)$ is the probability of the tail event that the random walk with i.i.d.~samples drawn from $P_\mu$ never crosses $n\theta$.

The following lemma (proved in \S \ref{app:greedy-general-upper}) gives a general characterization of the Bayesian regret of Greedy, under the assumption that rewards are subgaussian (Assumption \ref{ass:reward}).
\begin{lemma}[Generic bounds on Bayesian regret of Greedy]\label{lem:greedy-gen}
	Let Assumption \ref{ass:reward} hold. For any positive number $\delta$ that is less than $1/3$, Bayesian regret of Greedy (Algorithm \ref{alg:greedy}) satisfies
	\begin{multline}\label{eqn:bayes-greedy}
	\BR_{T,k}(\text{Greedy}) \leq  T \left(
	1 - \E_\Gamma \Big[\Ind \bp{\mu \geq 1-\delta} q_{1-2\delta}(\mu)\Big]\right)^k
	+3T \delta\\ 
+ k\, \E_\Gamma \bb{\Ind \bp{\mu < 1-3\delta} \min \bp{1+ \frac{3}{C_1(1-2\delta-\mu)}\,, T(1-\mu)}} \,.
	\end{multline}
In addition, for SS-Greedy (Algorithm \ref{alg:subs-greedy}) the same upper bound holds, with $k$ on the right hand side is replaced with $m$.
\end{lemma}
Lemma \ref{lem:greedy-gen} is the key technical result in the analysis of Greedy and SS-Greedy. This bound depends on several components, in particular, the choice of $\delta$ and the scaling of $q_{1-2\delta}(\cdot)$. To ensure sublinear regret, $\delta$ should be small, but
that increases the first term as $P(\mu \geq 1-\delta)$ decreases. The scaling of $q_{1-2\delta}(\cdot)$
is also important; in particular, the shape of $q(\cdot)$ will dictate the quality of the upper bound obtained.

Observe that $q_{1-2\delta}(\mu)$ is the only term in \eqref{eqn:bayes-greedy} that depends on the family of reward distributions $\mathcal{F}$. In the remainder of this section, we provide three upper bounds on Bayesian regret of Greedy and SS-Greedy. The first one is designed for Bernoulli rewards; here $q_\theta(\mu)$ has a constant lower bound, leading to optimal regret rates. The second result requires $1$-subgaussian rewards (Assumption \ref{ass:reward}); this leads to a $q_{1-2\delta}(\cdot)$ which is quadratic in $\delta$. The last bound makes an additional (mild) assumption on the reward distribution (covers many well-known rewards, including Gaussians); this leads to a $q_{1-2\delta}(\cdot)$ that is linear in $\delta$ which leads to a better bound on regret compared to $1$-subgaussian rewards.

The bounds that we establish on $q(\cdot)$ rely on Lundberg's inequality, which bounds the ruin probability of random walks and is stated below. For more details on this inequality, see Corollary 3.4 of \citep{asmussen2010ruin}.

\begin{proposition}[Lundberg's Inequality]\label{prop:lundberg}
	Let $X_1, X_2, \ldots$ be a sequence of i.i.d.~samples from distribution $Q$. Let $S_n = \sum_{i=1}^n X_i$ and $S_0 = 0$. For $u>0$ define the stopping time\footnote{Note that in risk theory, usually the random walk $u-S_n$ is considered and $\eta(u)$ is probability of hitting $0$, i.e., ruin.}
	\[
	\eta(u) = \inf \bc{n \geq 0: S_n > u}
	\]
	and let $\psi(u)$ denote the probability $\psi(u) = \P[\eta(u) < \infty]$. Let $\gamma > 0$ satisfy $\E[\exp(\gamma X_1)] = 1$
	and that $S_n \to -\infty$, almost surely, on the set $\bc{\eta(u) = \infty}$. Then, we have
	\[
	\psi(u) \leq \exp(-\gamma u)\,.
	\]
\end{proposition}

\subsection{Optimal Regret for Bernoulli Rewards}\label{subsec:bernoulli}

In this case, we can prove that there exists a constant lower bound on $q(\cdot)$.
\begin{lemma}\label{lem:bern-crossing}	Suppose $P_\mu$ is the Bernoulli distribution with mean $\mu$, and fix $\theta > 2/3$.  Then, for $\mu \geq (1+\theta)/2$,
	\[
	q_\theta(\mu) \geq \frac{\exp(-0.5)}{3}\,.
	\]
\end{lemma}

The preceding lemma reveals that for $\delta < 1/6$ and $\mu \geq 1-\delta$, for the choice $C_{\text{Bern}} = \exp(-0.5)/3$ we have $q_{1-2\delta}(\mu) \geq C_{\text{Bern}}$.
We can now state our theorem.
\begin{theorem}\label{thm:greedy-opt-bern}
	Suppose that $P_\mu$ is Bernoulli distribution with mean $\mu$ and the prior distribution on $\mu$ satisfies Assumption \ref{ass:prior}. Then, for $k \geq (30 \log T)/c_0$
	\[
	\BR_{T,k}(\text{Greedy}) \leq 1 + \frac{15T \log T}{kc_0} + \frac{3C_0 k}{C_1} \left[5+ \log(\frac{c_0\, k}{5}) - \log \log T\right] \,.
	\]
	Furthermore, Bayesian regret of SS-Greedy when executed with $m = \Theta(\sqrt{T})$ is $\tilde{O}(\sqrt{T})$.
\end{theorem}
This theorem shows that for $k=\Theta(\sqrt{T})$, Greedy is optimal (up to log factors). Similarly, for $k \geq \sqrt{T}$, SS-Greedy is optimal. The proofs of both results is in Appendix \ref{app:greedy-bern}.

\subsection{Sublinear Regret for Subgaussian Rewards}\label{subsec:sub-gaussian}

Here we show that the generic bound of Lemma \ref{lem:greedy-gen} can be used to obtain a sublinear regret bound in the general case of $1$-subgaussian reward distributions. Specifically, first we can prove the following bound  on $q_\theta(\mu)$.
\begin{lemma}\label{lem:subg-crossing}
	If $P_\mu$ is $1$-subgaussian, then for $\theta < \mu$ we have
	\[
	q_\theta(\mu) \geq 1 - \exp (-(\mu-\theta)^2/2) \geq e^{-1}(\mu-\theta)/2\,.
	\]
\end{lemma}
Next, from this lemma we obtain
\[
\inf_{\mu \geq 1-\delta} q_{1-2\delta}(\mu) \geq 1-\exp\left(-\frac{\delta^2}{2}\right)
\geq \frac{e^{-1} \delta^2}{2}\,.
\]
Combining this and Lemma \ref{lem:greedy-gen}, we can prove the following result for Greedy and SS-Greedy.
\begin{theorem}\label{thm:greedy-reg-1sg}
	Let assumptions \ref{ass:prior} and \ref{ass:reward} hold. Suppose that the Greedy algorithm (Algorithm \ref{alg:greedy}) is executed. Then, for $k \geq (54e \log T)/(c_0)$
	\begin{align*}
		\BR_{T,k}(\text{Greedy}) \leq 1 + 3T \bb{\frac{2e \log T}{kc_0}}^{1/3} + \frac{C_0 k}{C_1} \left[15+\log(\frac{kc_0}{2e}) - \log \log T\right] \,.
	\end{align*}
In other words, Bayesian regret of Greedy is $\tilde{O}(Tk^{-1/3}+k)$.
	Furthermore, Bayesian regret of SS-Greedy (Algorithm \ref{alg:subs-greedy}) when executed $m = \Theta \bp{T^{3/4}}$ is $\tilde{O}\bp{T^{3/4}}$.
\end{theorem}
Proofs of Lemma \ref{lem:subg-crossing} and Theorem \ref{thm:greedy-reg-1sg} are provided in \S \ref{app:greedy-subg}.
While this upper bound on regret is appealing -- in particular, it is sublinear regret when $k$ is large --- we are motivated by the empirically strong performance of Greedy and SS-Greedy (cf.~Figure \ref{fig:alph_75}) to see if a stronger upper bound on regret is possible.

\subsection{Tighter Regret Bound for Uniformly Upward-Looking Rewards}\label{subsec:unif-upward}

To this end, we make progress by showing that the above results are further improvable for a large family of subgaussian reward distributions, including Gaussian rewards. The following definition describes this family of reward distributions.
\begin{definition}[Upward-Looking Rewards]\label{def:upward-looking}
Let $Q$ be the reward distribution, and by abuse of notation, let $Q$ be also a random variable with the same distribution. Also, assume $\E[Q] = \mu$ and that $Q - \mu$ is $1$-subgaussian. Let $\bc{X_i}_{i=1}^\infty$ be a sequence of i.i.d.~random variables distributed according to $Q$ and $S_n = \sum_{i=1}^n X_i$. For $\theta < \mu$ define $R_n(\theta) = S_n - n\theta$ and
\[
\tau(\theta) = \inf \{n \geq 1: R_n(\theta) < 0 \text{~or~} R_n (\theta) \geq 1\}\,.
\]
We call the distribution $Q$ {\it upward-looking with parameter $p_0$} if for any $\theta < \mu$ one of the following conditions hold:
\begin{enumerate}
	\item $\P[R_{\tau(\theta)}(\theta) \geq 1] \geq p_0$
	\item $\E[(X_1-\theta) \Ind(X_1 \geq \theta)] \geq p_0$.
\end{enumerate}

More generally, a reward family $\mathcal{Q} = \bc{Q_\mu: \mu \in [0,1]}$ with $\E[Q_\mu] = \mu$ is called {\it uniformly upward-looking with parameters $(p_0, \delta_0)$} if for $\mu \geq 1-\delta_0, Q_\mu$ is upward-looking with parameter $p_0$.
\end{definition}
In \S \ref{subsubsec:examples} below, we show that a general class of reward families is uniformly upward-looking. In particular, class of reward distributions $\mathcal{F}$ that for all $\mu \geq 1-\delta_0$ satisfy $\E[(X_\mu-\mu) \Ind(X_\mu \geq \mu)] \geq c_0$ are uniformly upward looking with parameters $(c_0, \delta_0)$. This class includes the Gaussian rewards.

The preceding discussion reveals that many natural families of reward distributions are upward-looking. The following lemma shows that for such distributions, we can sharpen our regret bounds.
\begin{lemma}\label{lem:max-crossing}
Let $Q$ be upward-looking with parameter $p_0$ which satisfies $\E[Q]=\mu$. Let $\bc{X_i}_{i=1}^{\infty}$ be an iid sequence distributed as $Q$, $S_n = \sum_{i=1}^n X_i$ and $M_n= S_n / n$. Then for any $\delta \leq 0.05$,
\[
\P \bb{\exists~n: M_n < \mu - \delta } \leq \exp (-p_0 \delta/4)\,.
\]
\end{lemma}
From this lemma, for $\mu \geq 1-\delta$ we have
\[
q_{1-2\delta}(\mu) \geq 1 - \exp(-p_0 \delta/4) \geq (p_0 e^{-1}/4) \delta\,.
\]
The following theorem shows that in small $\delta$ regime, this {\em linear} $q(\cdot)$ yields a strictly sharper upper bound on regret than a {\em quadratic} $q(\cdot)$.
%
%
\begin{theorem}\label{thm:greedy-reg}
Let Assumptions \ref{ass:prior} and \ref{ass:reward} hold. Suppose that $\mathcal{F}$ is $(p_0, \delta_0)$ uniformly upward-looking. Then for any $k$ that satisfies the inequality,
\[
k\ge \frac{4e \max (400, 1/\delta_0^2)\log T}{c_0 p_0}\,,
\]
Bayesian regret of Greedy satisfies,
\begin{align*}
\BR_{T,k}(\text{Greedy}) \leq 1 + 3T \bb{\frac{4e \log T}{kc_0p_0}}^{1/2} + \frac{3C_0 k}{2C_1} \left[10+\log\left(\frac{kc_0p_0}{4e}\right) - \log \log T\right] \,.
\end{align*}
Furthermore, Bayesian regret of SS-Greedy when executed with $m=\Theta(T^{2/3})$ is $\tilde{O}(T^{2/3})$.
\end{theorem}
Proofs of Lemma \ref{lem:max-crossing} and Theorem \ref{thm:greedy-reg} are given in \S \ref{app:greedy-unif-upw}.

It is worth noting that in the case $k > T$ where subsampling is inevitable, the results presented on SS-Greedy in Theorems \ref{thm:greedy-opt-bern} and \ref{thm:greedy-reg} are still valid. The main reason is that our proof technique presented in Lemma \ref{lem:greedy-gen} bounds the regret with respect to the ``best'' possible reward of $1$ which as stated in Lemma \ref{lem:greedy-gen} allows for an immediate replacement of $k$ with the subsampling size $m$.

\subsubsection{Examples of Uniformly Upward-Looking Distributions}\label{subsubsec:examples}

Here we show that two general family of distributions are uniformly upward-looking.

First, we start by showing that the Gaussian family is uniformly upward looking.
Suppose $\mathcal{F}^{\text{gaussian}} = \bc{\mathcal{N}(\mu,1): \mu \in [0,1]}$. Then, for any $\delta_0 \leq 1$ the family $\mathcal{F}^{\text{gaussian}}$ is uniformly upward-looking with parameters $(1/\sqrt{\pi}, \delta_0)$. To see this, note that for any $\theta < \mu $ the random variable $X_\mu - \theta$ is distributed according to $\mathcal{N}(\mu-\theta,1)$. Hence, using the second condition in Definition \ref{def:upward-looking}, for $X_1 \sim \mathcal{N}(\mu,1)$ we have
	\begin{equation*}
		\E[(X_1-\theta) \Ind (X_1 \geq \theta)] \geq \frac{\E_{Z \sim \mathcal{N}(0,1)}|Z|}{2} = \frac{1}{\sqrt{\pi}}\,.
	\end{equation*}
More generally, suppose that $Z$ is $1$-subgaussian and consider the family of reward distributions $\mathcal{F} = \mu + Z$. In other words, the family $\mathcal{F}$ is shift-invariant, meaning that $P_{\mu'} = P_\mu + (\mu' - \mu)$. The same argument as the above (Gaussian) case shows that for any choice of $\delta_0 \leq 1$, $\mathcal{F}$ is uniformly upward-looking with parameters $(p_0, \delta_0)$ with $p_0 = \E[Z \Ind(Z\geq 0)]$.

Second, we show that a family of reward distributions $\mathcal{F} = \bc{P_\mu: \mu \in [0,1]}$  satisfying the following: there exist constants $0 < \delta_0 < 1, p_0>0$ and $\theta < 1$ such that for any $\mu \geq \theta$ integer $m$ exists such that:
\begin{equation}\label{eqn:power-m}
	\P[X_\mu \geq \theta + 1/m]^m \geq p_0\,.
\end{equation}
Then, it is easy to observe that $\mathcal{F}$ is upward-looking with parameters $(p_0,\delta_0)$. For example, this also can be used to prove that the Gaussian rewards are uniformly upward-looking: take $m=1$, then the above condition translates to
\begin{equation*}
	\P[\mathcal{N}(\mu,1) \geq \theta  + 1] \geq \P[\mathcal{N}(0,1) \geq 1] \geq 1-\Phi(1)\,,
\end{equation*}
where $\Phi$ is the CDF of the standard Gaussian distribution.

The proof of the above claim is as follows. The idea is to show that for any family satisfying Eq. \eqref{eqn:power-m} the first condition for the uniformly upward-looking distribution is satisfied. Indeed, consider $\theta < \mu$ and let $\bc{X_i}_{i=1}^\infty \sim P_\mu$. Then, we claim that the random walk $R_n(\theta) = \sum_{i=1}^n X_i - n\theta$ crosses $1$ before crossing zero, with probability at least $p_0$. Indeed, Eq. \eqref{eqn:power-m} implies that the probability that all observations $X_1, X_2, \ldots, X_m$ are at least $\theta+1/m$ is at least $p_0$. In this case, as $\theta + 1/m > \theta$, with probability $p_0$, the random walk will not cross zero before $m$ and it crosses $1$ at $n = m$. This proves our claim.

%% file: simulations.tex
\section{Simulations}\label{sec:simulations}


Recall Figure \ref{fig:alph_75} with results of simulations for two pairs of $T, k$ in the many-armed regime where rewards were generated according to Gaussian noise and uniform prior. These results are robust when considering a wide range of beta priors as well as both Gaussian and Bernoulli rewards (see \S \ref{app:add-sim}).
In this section, motivated by real-world applications, we consider a \emph{contextual reward} setting and show that our theoretical insights carry to the contextual setting as well.\footnote{Codes and data for reproducing all empirical results of the paper are available in this Github repository \href{https://github.com/khashayarkhv/many-armed-bandit}{https://github.com/khashayarkhv/many-armed-bandit}.}

\paragraph{Data and Setup.} We use the Letter Recognition Dataset \cite{frey1991letter} from the UCI repository. The dataset is originally designed for the letter classification task ($26$ classes) and it includes $n=20000$ samples, each presented with $16$ covariates. As we are interested in values of $k > 26$, we only use the covariates from this dataset and create synthetic reward functions as follows. We generate $k = 300$ arms with parameters $\theta_1, \theta_2, \ldots, \theta_k \in \IR^d$ ($d$ will be specified shortly) and generate reward of arm $i$ at time $t$ via
\[
Y_{it} = X_t^\top \theta_i + \varepsilon_{it}\,,
\]
where $X_t\in\IR^d$ is the context vector at time $t$ and $\varepsilon_{it}$ is iid noise.

Our goal is to show that there are two distinct sources of free exploration. The first one is due to the variation in covariates, and the other one is due to having large number of arms. Therefore, we consider two experiments with $d=2$ and $d=6$ and compare the performance of several algorithms in these two cases. As contexts are $16$-dimensional, we project them onto $d$-dimensional subspaces using SVD.

For each $d$, we generate $50$ different simulation instances, where we pick $T = 8000$ samples at random (from the original $20000$ samples) and generate the arm parameters according to the uniform distribution on the $\ell_2$-ball in $\IR^d$, i.e., $\theta_i \sim \Unif_d = \{u \in \IR^d: \|u\|_2 \leq 1 \}$.  We plot the distribution of the per-instance regret in each case, for each algorithm; note the mean of this distribution is (an estimate of) the Bayesian regret.  We study the following algorithms and also their subsampled versions (with subsampling $m = \sqrt{T}$; subsampling is denoted by ``SS''):
\begin{enumerate}
	\item Greedy,
	\item OFUL Algorithm \citep{abbasi2011improved}, and
	\item TS \citep{thompson1933likelihood, russo2014learning}.
\end{enumerate}

\paragraph{Results.}The results are depicted in Figure \ref{fig:contextual}. We can make the following observations. First, subsampling is an important concept in the design of low-regret algorithms, and indeed, SS-Greedy outperforms all other algorithms in both settings. Second, Greedy performs well compared to OFUL and TS, and it benefits from the same free exploration provided by a large number of arms that we identified in the non-contextual setting: if it drops an arm $a$ due to poor empirical performance, it is likely that another arm with parameter close to $\theta_a$ is kept active, leading to low regret. Third, we find that SS-TS actually performs reasonably well; it has a higher average regret compared to SS-Greedy, but smaller variance.
Finally, by just focusing on Greedy, OFUL, and TS (ignoring their SS versions) and comparing with Figure \ref{fig:contextual-small-k}, we see empirical evidence that the aforementioned source of free exploration is different from that observed in recent literature on contextual bandits (see, e.g., \citep{bastani2017mostly, kannan2018smoothed, raghavan2018externalities, hao2019adaptive}), where free exploration arises due to diversity in the context distribution.  For example, simulations of \cite{bastani2017mostly} show that  when $d$ is too small compared to $k$, the context diversity is small and performance of Greedy deteriorates which leads to a high variance for its regret. Figure \ref{fig:contextual-small-k} which shows results of the above simulation, but using only $k=8$ arms, underscores the same phenomena -- by reducing the number of arms, the performance of Greedy substantially deteriorates compared to TS and OFUL. But in the case of $k=300$, even with $d=2$, Greedy has the best average and median regret.
\begin{figure*}[ht]
  \centering
  \subfigure[Large number of arms.]{\includegraphics[width=0.47\textwidth]{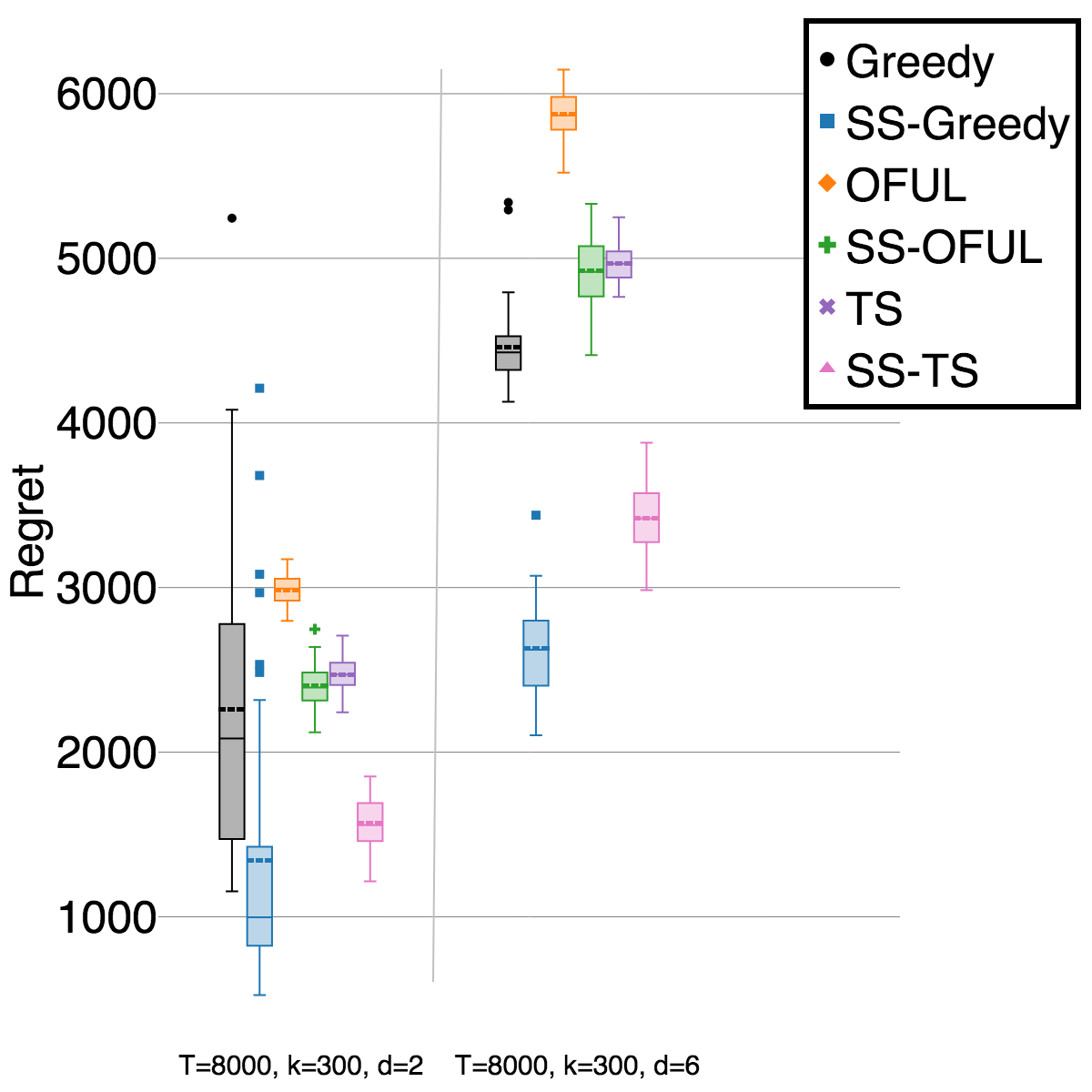}\label{fig:contextual}}
\hfill
  \subfigure[Small number of arms. ]{\includegraphics[width=0.47\textwidth]{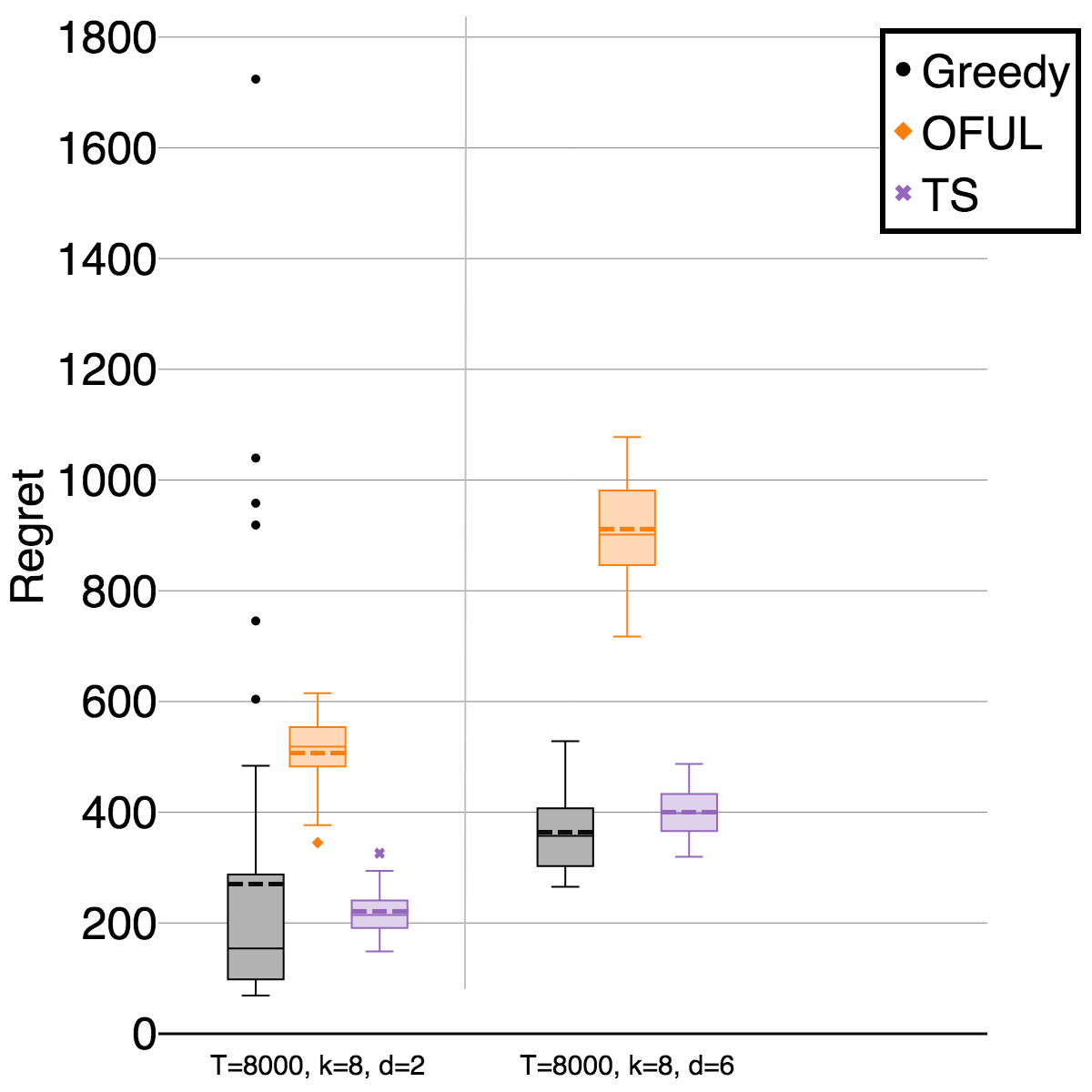}\label{fig:contextual-small-k}}
  \caption{
  	Distribution of the per-instance regret for the contextual setting with real data. For $k=8$, subsampled algorithms are omitted as subsampling leads to a poor performance. In these figures, the dashed lines indicate the average regret. 
  }
\end{figure*}

%% file: generalizations.tex
\section{Generalizations}\label{sec:generalizations}

In this section we generalize results of \S \ref{sec:LB-and-opt}-\ref{sec:greedy} in two ways. First, in \S \ref{subsec:beta-regular}, we relax Assumption \ref{ass:prior} and consider more general prior distributions. We prove that all results that appeared in \S \ref{sec:LB-and-opt}-\ref{sec:greedy} can be generalized to the broader family of $\beta$-regular prior distributions, introduced in Definition \ref{def:beta-prior}. Second, in \S \ref{subsec:seq-greedy}, we consider a more effective version of the SS-Greedy algorithm, called Seq-Greedy, and prove that our results for SS-Greedy also hold for Seq-Greedy.

\subsection{General $\beta$-regular Priors} \label{subsec:beta-regular}

Assume that the prior reward distribution is $\beta$-regular per Definition \ref{def:beta-prior}.
Before stating the results, we need to update the notion of  ``many armed regime" as well as subsampling rate of various algorithms.

\paragraph{Small versus large $k$ regime.} It can be shown (see \S \ref{app:generalizations} for details) that the transition from small to large $k$ regime, for lower bound, depends on whether $k < T^{\beta/(\beta+1)}$ or not. However, for the upper bounds, the transition point is algorithm dependent (depending on which term of the regret is dominant). For example, when $\beta = 1$ and the reward is Gaussian, we saw in \S \ref{sec:LB-and-opt}-\ref{sec:greedy} that for the lower bounds (and even some upper bounds) the transition from small to large $k$ occurs at $\sqrt{T}$. However, our result for Greedy in \S \ref{sec:greedy} would give
\[
\BR_{T,k}(\text{Greedy}) \approx k + Tk^{-1/2}\,,
\]
which means the small versus large $k$ transition for this upper bound occurs at  $T^{2/3}$.

\paragraph{Subsampling rate.} Similar to \S \ref{sec:LB-and-opt}-\ref{sec:greedy}, SS-UCB and SS-Greedy use all arms whenever $k < m$. For these algorithms the optimal values for $m$ are used, which per above discussion, is algorithm dependent and happens when the transition between small $k$ and large $k$ occurs for UCB and Greedy, respectively.

\paragraph{Conditions on the density.} Similar to Theorem \ref{thm:upp-small-k}, for $\beta \geq 1$ we require a bounded density for UCB and SS-UCB in the small $k$ regime. For $\beta < 1$, it is easy to observe that the density cannot be bounded near $1$ as $\P[\mu \geq 1-\vep] = \Theta(\vep^\beta)$. Hence, in this setting we instead assume that
\[
g(x) \leq D_0 (1-x)^{\beta-1}\,.
\]

\paragraph{A new notation for greedy algorithms.} We introduce a new notation, $\betaF$, that will be used for Greedy and SS-Greedy. To motivate the notation, recall that results in \S \ref{sec:greedy} qualitatively depend on the exponent of $\delta$ in $q_{1-2\delta}(\mu)$. On the other hand, since the term $\P \bp{ \mu \geq 1-\delta}$ by definition is of order $\Theta(\delta^\beta)$, we could argue the exponent of $\delta$ in  the expression
\begin{equation}\label{eq:first-term-of-regret}
\E_\Gamma \bb{ \Ind \bp{ \mu \geq 1-\delta} q_{1-2\delta}(\mu)}
\end{equation}
that appeared in Lemma \ref{lem:greedy-gen}, dictates the regret upper bounds. With that motivation in mind, let us define
$\betaF$ to be a lower bound on the exponent of $\delta$ in \eqref{eq:first-term-of-regret}. More concretely, we assume $\betaF$ is a constant that satisfies\footnote{One can be more precise and define $\betaF$ to be the supremum of all constants that satisfy \eqref{eq:gen-prior-delta-scaling}, but for our purposes this level of precision is not necessary.} the following:
\begin{equation}\label{eq:gen-prior-delta-scaling}
\E_\Gamma \bb{ \Ind \bp{ \mu \geq 1-\delta} q_{1-2\delta}(\mu)} = \tilde{\Omega}(\delta^\betaF)\,.
\end{equation}
Clearly, any such constant $\betaF$ must satisfy the inequality $\betaF\ge\beta$, given Definition \ref{def:beta-prior}. But for some reward distributions one may need to pick $\betaF$ strictly larger than $\beta$. In fact, we use the arguments of \S \ref{sec:greedy} and show that the constants
\begin{itemize}
\item $\beta_\mathcal{\text{Bernoulli}} = \beta$,
\item $\beta_\mathcal{\text{$1$-subgaussian}} = \beta + 2$, and
\item $\beta_\mathcal{\text{Upward-Looking}} = \beta + 1$
\end{itemize}
satisfy \eqref{eq:gen-prior-delta-scaling}.

\subsubsection{Results for $\beta$-regular priors}

First, we show that a similar argument as in \S \ref{subsec:low} can be used to show that regret of any policy is at least
\[
\min\left(k, T^{\beta/(\beta+1)} \right)\,.
\]
This highlights that from the perspective of the lower bounds small $k$ versus large $k$ is determined when $k < T^{\beta/(\beta+1)}$ versus $k \ge T^{\beta/(\beta+1)}$ respectively. For the upper bounds, the results are presented in Table \ref{tab:gen-beta} below. Specifically, Table \ref{tab:gen-beta} shows upper bounds for each of UCB and Greedy and their subsampling versions, for both small versus large $k$ and when  $\beta$ is smaller or larger than $1$. A more detailed version of all these results with corresponding proofs is deferred to \S \ref{app:generalizations}.
\begin{table*}[hptb]
    \centering
    \caption{Upper bounds for regret of various algorithms for different ranges of $\beta$ and for small versus large $k$}\label{tab:gen-beta}
    \rowcolors{5}{}{gray!10}
    \begin{tabular}{m{0.15\textwidth}m{0.13\textwidth}m{0.35\textwidth}m{0.15\textwidth}m{0.15\textwidth}}
        \toprule
        & \multicolumn{2}{c}{$\beta < 1$} &\multicolumn{2}{c}{$\beta \ge 1$} \\
        \cmidrule(lr){2-5}
         & \textbf{small $k$} & \textbf{large $k$} & \textbf{small $k$} & \textbf{large $k$} \\
          \textbf{Algorithm} &  &      &      \\
        \midrule
        			UCB & $k^{1/\beta}$ & $k T^{(1-\beta)/2}$ & $k$ & $k$  \\
  	SS-UCB & $k^{1/\beta}$ & $\sqrt{T}$ & $k$ & $ T^{\beta/(\beta+1)}$ \\
Greedy & $T k^{-1/\betaF}$ & $\min(k^{(\beta_\mathcal{F}-\beta+1)/\beta_\mathcal{F}}, k T^{(1-\beta)/2})$ & $T k^{-1/\betaF}$ & $k$ \\
SS-Greedy & $T k^{-1/\betaF}$ & $T^{(\betaF - \beta + 1)/(\betaF - \beta + 2)}$ & $T k^{-1/\betaF}$ & $T^{\betaF/(\betaF+1)}$ \\
        \bottomrule
    \end{tabular}
\end{table*}

\subsection{Sequential Greedy} \label{subsec:seq-greedy}

When $k$ is large, allocating the first $k$ (or $T^{\beta/(\beta+1)}$ arms in case of subsampling) time-periods for exploration before exploiting the good arms may be inefficient. We can design a sequential greedy algorithm, in which an arm is selected and pulled until its sample average drops below a certain threshold, say $1 - \upsilon$. Once that happens, a new arm is selected and a similar routine is performed; we denote this algorithm by Seq-Greedy and its peseudo-code is shown in Algorithm \ref{alg:seq-greedy}.
\begin{algorithm}
	\caption{Sequential Greedy Algorithm (Seq-Greedy)}\label{alg:seq-greedy}
	\begin{algorithmic}[1]
		\STATE \textbf{Input:} $\upsilon$
		\STATE Initialize $i = 1$.
		\STATE Pull arm $a_1 = i$
		\FOR {$t \geq 2$}
		\IF {$\hmu_i(t-1) \geq 1 - \upsilon$}
		\STATE Pull arm $a_t = i$ ~~\COMMENT{current arm is good}
		\ELSIF{$i < k$}
		\STATE Pull arm $a_t = i+1$ ~~\COMMENT{use a new arm}
		\STATE Update the index $i \gets i+1$
		\ELSE
		\STATE Pull arm $a_t = \argmax_{i\in[k]} \hmu_i(t-1)$ ~~\COMMENT{execute greedy}
		\ENDIF
		\ENDFOR
	\end{algorithmic}
\end{algorithm}
In this section we show that for an appropriate choice of $1-\upsilon$, Bayesian regret of Seq-Greedy is similar to that of SS-Greedy.

First, we state the following result on the regret of Seq-Greedy.
\begin{lemma}[Bayesian Regret of Seq-Greedy]\label{lem:seq-greedy-gen}
	Let Assumption \ref{ass:reward} hold and suppose that Algorithm \ref{alg:seq-greedy} is executed with parameter $\upsilon = 2 \delta$. Define $q_\theta(\mu)$ as the probability that sample average of an i.i.d. sequence distributed according to $P_\mu$ never crosses $\theta$ (see also Eq. \eqref{eqn:max-cross}).
Then, for any integer $1 \leq k_1 \leq k$, the following bound holds on the Bayesian regret of the Seq-Greedy algorithm:
\begin{align}\label{eqn:reg-seq-greedy}
		\BR_{T,k} \bp{\text{Seq-Greedy}}
		&\leq  T \left(1 - \E_\Gamma \Big[\Ind \bp{\mu \geq 1-\delta} q_{1-2\delta}(\mu)\Big]\right)^{k_1}
		+3T \delta \nonumber \\
		&+ k_1 \E_\Gamma \bb{\Ind \bp{\mu < 1-3\delta} \min \bp{1+ \frac{3}{C_1(1-2\delta-\mu)}, T(1-\mu)}}\,.
\end{align}
\end{lemma}
Note that the above result shows that the Bayesian Regret of Seq-Greedy is at most
\begin{align*}
	\BR_{T,k} \bp{\text{Seq-Greedy}}
	&\leq \min_{1 \leq k_1 \leq k} \Bigg \{ T \left(1 - \E_\Gamma \Big[\Ind \bp{\mu \geq 1-\delta} q_{1-2\delta}(\mu)\Big]\right)^{k_1}
	+3T \delta \\
	&+ k_1 \E_\Gamma \bb{\Ind \bp{\mu < 1-3\delta} \min \bp{1+ \frac{3}{C_1(1-2\delta-\mu)}, T(1-\mu)}} \Bigg \}\,,
\end{align*}
i.e., the best possible rate one can achieve by optimizing $k$ in the upper bound of regret of Greedy presented in Lemma \ref{lem:greedy-gen}.

While the above result can be used to prove similar results as those presented in Table \ref{tab:gen-beta}, we only prove this for the case $\beta = 1$ (similar to \S \ref{sec:LB-and-opt}-\ref{sec:greedy}) as the proofs of other cases are  similar.
\begin{theorem}[Bayesian Regret of Seq-Greedy]\label{thm:seq-greedy}
	Suppose that Assumptions  \ref{ass:prior} and \ref{ass:reward} hold and that Algorithm \ref{alg:seq-greedy} has been executed with parameter $\upsilon = 2 \delta $. Then,
	\begin{itemize}
		\item[(a)] For Bernoulli reward and $\delta= \tilde{\Theta} \bp{\max(k^{-1}, T^{-1/2})}$ we have
		\begin{equation*}
			\BR_{T,k}(\text{Seq-Greedy}) \leq \tilde{O} \bp{\max \bb{\sqrt{T}, \, \frac{T}{k}} }\,.
		\end{equation*}
		\item[(b)] For general $1$-subgaussian rewards and $\delta = \tilde{\Theta} \bp{\max(k^{-1/3}, T^{-1/4})}$ we have
		\begin{equation*}
			\BR_{T,k}(\text{Seq-Greedy}) \leq \tilde{O} \bp{\max \bb{T^{3/4}, \, \frac{T}{k^{1/3}} } }.
		\end{equation*}
		\item[(c)] For uniformly upward-looking distributions and $\delta = \tilde{\Theta} \bp{\max(k^{-1/2}, T^{-1/3})}$ we have
		\begin{equation*}
			\BR_{T,k}(\text{Seq-Greedy}) \leq \tilde{O} \bp{\max \bb{T^{2/3}, \, \frac{T}{k^{1/2} } } } \,.
		\end{equation*}
	\end{itemize}
\end{theorem}
Proofs of Lemma \ref{lem:seq-greedy-gen} and Theorem \ref{thm:seq-greedy} are presented in \S \ref{app:generalizations}.

%% file: appendix.tex
\section{Proofs of \S \ref{sec:LB-and-opt}}\label{sec:proofs-LB-and-opt}

\subsection{Proof of the Lower Bound}\label{app:prelim-lb}

\proof{Proof of Theorem \ref{thm:lb}.}
We start by proving that if $k \geq (6 C_0/c_0) \sqrt{T}$, the Bayesian regret of any policy is at least $c_L \sqrt{T}$ for a constant $c_L$.
	
	Note that, without loss of generality, we can only consider the policies that first pull arm $1$, then if needed arm $2$, then if needed arm $3$, and so on. This is because all arms have the same prior $\Gamma$. More precisely, for any permutation $\sigma$ on $[k]$ and any policy $\pi$, let $\sigma \circ \pi$ be the policy that pulls arm $\sigma(i)$ when policy $\pi$ pulls arm $i$. Then,  $\BR_{T,k}(\pi) = \BR_{T,k}(\sigma \circ \pi)$.
	
	With this observation in hand, we define a class of ``bad orderings'' that happen with a constant probability, for which the Bayesian regret cannot be better than $\sqrt{T}$. This would prove our lower bound. More precisely, let $m = \lfloor{\sqrt{T} \rfloor}$ and define event $A$ as the set of all realizations $\bmu$ that satisfy the following conditions:
	\begin{itemize}
		\item[(i)] $\max_{i=1}^k \mu_i \geq 1-\frac{1}{kc_0}$,
		\item[(ii)] $\max_{i=1}^m \mu_i \leq 1-\frac{1}{2C_0\sqrt{T}}$,
		\item[(iii)] $\sum_{i=1}^m \mu_i \leq \bp{1-\frac{1}{8C_0}} m$\,,
	\end{itemize}
	where $c_0$ and $C_0$ are the constants defined in Assumption \ref{ass:prior}.
	
	First, we claim that the event $A$ happens with a constant probability. Note that if $Z \sim \Gamma$, then
	\begin{equation}\label{eqn:Z_low}
	\P \bb{Z \leq 1- \frac{1}{kc_0}} = 1 - \P \bb{Z > 1 - \frac{1}{kc_0}} \leq 1-\frac{1}{k} \,,
	\end{equation}
	using the lower bound in Assumption \ref{ass:prior}. Hence,
	\begin{equation*}
	\P \bb{\max_{i=1}^k \mu_i \leq 1-\frac{1}{kc_0}} \leq \bp{1-\frac{1}{k}}^k \leq e^{-1}\,,
	\end{equation*}
	implying that event (i) happens with probability at least $1-e^{-1}$. Similar to Eq. \eqref{eqn:Z_low}, using the upper bound in Assumption \ref{ass:prior}, we have
	\begin{equation*}
	\P \bb{\max_{i=1}^m \mu_i \leq 1-\frac{1}{2C_0\sqrt{T}}} \geq \bp{1-\frac{1}{2\sqrt{T}}}^m \geq 0.5\,,
	\end{equation*}
	where the right hand side inequality follows from the generalized Bernoulli inequality; therefore event (ii) holds with probability at least $0.5$. Finally, we lower bound the probability of event (iii).  In fact, using
	\begin{align*}
	\E_{\Gamma}[\mu]
	&\leq 1 \cdot \P_\Gamma \bb{\mu \geq 1-\frac{1}{2C_0}} + \bp{1-\frac{1}{2C_0}} \P_\Gamma \bb{\mu < 1-\frac{1}{2C_0}} \leq \frac{1}{2} + \frac{1}{2} \bp{1-\frac{1}{2C_0}} \\
	&\leq 1 - \frac{1}{4C_0}\,,
	\end{align*}
	we have $\E_\Gamma[\mu_i] + 1/(8 C_0) \leq 1-1/(8C_0)$. Hence, Hoeffding's inequality implies that
	\begin{align*}
	\P \bb{\sum_{i=1}^m \mu_i \geq \bp{1-\frac{1}{8C_0}} m} \leq \P \bb{\frac{\sum_{i=1}^m (\mu_i-\E[\mu_i])}{m} \geq \frac{1}{8C_0}}
	\leq \exp\bp{- \frac{m}{32C_0^2}} \leq 0.1\,,
	\end{align*}
	when $m \geq 74 C_0^2$, i.e., if $T \geq (74 C_0^2)^2$. This means (iii) occurs with probability at least $0.9$.
A union bound implies that,
	\begin{equation*}
	    \P[A] \geq 1-e^{-1}-0.5-0.1 \geq 0.03\,.
	\end{equation*}
	In other words, the event $A$ happens at least for a constant fraction of realizations $\bmu$. We claim that on the realizations $\bmu$ that are in $A$, the expected regret of policy $\pi$ is lower bounded by $\Omega(\sqrt{T})$, i.e.,
\[
\E[R_T(\pi \mid \bmu) | \bmu \in A] \geq \Omega(\sqrt{T})\,.
\]
	
	Note that given the assumption about the policy $\pi$, the decision-maker starts by pulling arm $1$, and continues pulling that arm for some number of rounds. At some point the decision-maker pulls arm $2$ (if needed), and then pulls only arms $1$ or $2$ for some number of rounds, until at some point she pulls arm $3$ (if needed); and so on.   Although the choice of whether to try a new arm or keep pulling from the existing set of tried arms may depend on the observations $Y_{it}$, on any particular sample path 
	one of these two possibilities arise:

\begin{itemize}	
\item	{\it Case 1.} The decision-maker only pulls from the first $m$ arms during the $T$-period horizon; or
\item	{\it Case 2.} The decision-maker finishes pulling all arms in $[m]$, and starts pulling some (or all) arms in the set $[k]\setminus[m]$.
\end{itemize}
	
	We claim that in both cases, the decision-maker will incur $\Omega(\sqrt{T})$ regret.  This is argued by considering each case in turn below.
	
\paragraph{Regret in Case 1.} In this case the regret incurred in each period $t$, for $t\in[T]$, is at least $\max_{i=1}^k \mu_i - \max_{i=1}^m \mu_i$. Therefore,
	\begin{align*}
	R_T(\pi\mid\bmu)
	\geq T \bb{\bp{1-\frac{1}{c_0k}} - \bp{1-\frac{1}{2C_0 \sqrt{T}}}}
	\geq \frac{T}{2C_0\sqrt{T}} - \frac{T}{c_0 k}
	\geq \frac{\sqrt{T}}{3C_0}\,.
	\end{align*}
	where the last inequality is true once $k \geq 6 \sqrt{T} C_0/c_0$.

\paragraph{Regret in Case 2.} In this case the algorithm pulls each of arms $1,2,\ldots, m$ at least once and hence using $m \geq \sqrt{T}/2$
	\begin{align*}
	R_T(\pi\mid\bmu)
	\geq \sum_{i=1}^m \bp{\max_{j=1}^k \mu_j - \mu_i}
	%
	%
	\geq m \bb{\bp{1-\frac{1}{c_0 k}} - \bp{1-\frac{1}{8C_0}}}
	\geq \frac{\sqrt{T}}{24C_0}\,,
	\end{align*}
	which holds when $k \geq 24 C_0/c_0$.

	Thus regardless of the observations $Y_{it}$, and whether the decision-maker decides to try all arms in $[k]$ or only the  arms in $[m]$, the regret is lower bounded by $l_b \sqrt{T}$, where $l_b = 1/(24C_0)$. For finishing the proof note that we have
	\begin{align*}
	\BR_{T,k} (\pi)
	&= \int \E[R_T(\pi \mid {\bmu})] d {\bmu}
	\geq \int \E[R_T(\pi \mid {\bmu})] \Ind(\bmu \in A) d {\bmu} \\
	&\geq l_b \sqrt{T} \int \Ind(\bmu \in A) d {\bmu}
	=l_b \sqrt{T} \,\P[A]
	\geq 0.03 l_b \sqrt{T}\,.
	\end{align*}
	As a conclusion, if $T \geq \max \bp{16, (74 C_0)^2} = C_1$, then for any $k \geq 6 \sqrt{T}C_0/c_0$, the Bayesian regret of $\pi$ is at least $0.03 l_b \sqrt{T}$.
	
	Now if $k \leq 6 \sqrt{T} C_0/c_0$, then we can exploit the fact that the Bayesian regret is nondecreasing in time horizon. In other words, letting $r = \lfloor{\bp{c_0k/ (6C_0)}^2 \rfloor}$, we have $k \geq 6 \sqrt{r} C_0/c_0$ and hence
	\begin{equation*}
	\BR_{T,k}(\pi) \geq \BR_{r^2,k}(\pi) \geq 0.03 l_b \sqrt{r} = 0.03 l_b \frac{c_0 k}{12 C_0}\,,
	\end{equation*}
	as long as $r \geq  \max \bp{16, (74 C_0)^2} = C_1$, meaning $k \geq \lceil{(6C_0/c_0) \sqrt{C_1} \rceil}$. Hence, letting $c_L = (0.03 l_b c_0)/(12C_0)$ and $c_D = \max \bp{C_1, (6C_0/c_0) \sqrt{C_1}}$ implies the result.
\Halmos
\endproof

\subsection{Proofs for Upper Bound on SS-UCB}\label{app:prelim-up}

\proof{Proof of Theorem \ref{thm:upp-small-k}}
	The proof follows from the analysis of asymptotically optimal UCB algorithm which can be found in \cite{lattimore2018bandit} with slight modifications.
	
	Suppose that $\bmu = (\mu_1, \mu_2, \ldots, \mu_k)$ are drawn from $\Gamma$ and let $\mu_{(k)} \leq \mu_{(k-1)} \leq \ldots \leq \mu_{(1)}$ be a sorted version of these means. Denote $\bmu_{()}=(\mu_{(1)},\mu_{(2)},\ldots,\mu_{(k)})$ and note that conditioned on $\bmu_{()}$, we can first derive an upper bound on the expected regret of the asymptotically optimal UCB algorithm and then take an expectation to derive an upper bound on the Bayes regret of this algorithm. Denote $\Delta_{(i)} = \mu_{(1)} - \mu_{(i)}$ and decompose the expected regret as
	\begin{equation*}
	R_T(\text{UCB} \mid \bmu_{()}) = \sum_{i=2}^k \Delta_{(i)} \E[N_{(i)}(T)] \,.
	\end{equation*}
	where $N_{(i)}(T)$ is the number of pulls of arm with mean $\mu_{(i)}$ till time $T$. Theorem 8.1 of \cite{lattimore2018bandit} establishes an upper bound on $\E[N_{(i)}(T)]$. Specifically, for any $\epsilon \in (0, \Delta_{(i)})$, we have
	\begin{equation*}
	\E[N_{(i)}(T)] \leq 1 + \frac{5}{\epsilon^2} + \frac{2 \log f(T) + \sqrt{\pi \log f(T)} + 1}{(\Delta_{(i)}-\epsilon)^2}\,.
	\end{equation*}
	Take $\epsilon = \Delta_{(i)}/2$ and note that whenever $T \geq 4$, $\sqrt{\pi \log f(T)} \leq 3 \log f(T)$ and $1 \leq \log f(T)$. Hence,
	\begin{equation*}
	\E[N_{(i)}(T)] \leq 1 + \frac{20 + 36 \log f(T)}{\Delta_{(i)}^2}\,.
	\end{equation*}
	Note that as the number of pulls for each arm is upper bounded by $T$, we can write
	\begin{equation*}
	\E[N_{(i)}(T)] \leq \min \bp{1 + \frac{20 + 36 \log f(T)}{\Delta_{(i)}^2},T}\,.
	\end{equation*}
	Plugging this into the regret decomposition, we have
	\begin{align*}
	R_T(\text{UCB} \mid \bmu_{()} ) =  \sum_{i=2}^k  \min \bp{\Delta_{(i)} + \frac{20 + 36 \log f(T)}{\Delta_{(i)}},T \Delta_{(i)}}\,.
	\end{align*}
	We need to take expectation of the above expression over the distribution of $\bmu_{()}=(\mu_{(1)}, \mu_{(2)}, \cdots \mu_{(k)})$. As it can be observed from the above representation, only the distribution of $\Delta_{(2)}, \Delta_{(3)}, \ldots, \Delta_{(k)}$ matters. 
	We separate the case of $i=2$ and $i\geq 3$, and also the terms $\Delta_{(i)}$ from $1/\Delta_{(i)}$ to write
	\begin{align*}
	\BR_{T,k}(\text{UCB}) &\leq \sum_{i=2}^k \E[\Delta_{(i)}] + \sum_{i=3}^k \E\bb{\frac{20 + 36 \log f(T)}{\Delta_{(i)}}} + \E \bb {\min \bp{\frac{20 + 36 \log f(T)}{\Delta_{(2)}},T \Delta_{(2)}}} \,.
	\end{align*}
	Note that for the first term above, we can write
	\begin{equation*}
		\sum_{i=2}^k \E[\Delta_{(i)}] \leq k \E[\mu_{(1)}] - \E \bb{\sum_{i=1}^k \mu_i} \leq k \bp{\E[\mu_{(1)}] - \E_{\mu \sim \Gamma}[\mu]} \leq k\,.
	\end{equation*}
	
	Lemma \ref{lem:Deltas} shows that for any $i \geq 3$, we have $\E[\frac{1}{\Delta_{(i)}}] \leq \frac{D_0 k}{i-2}$, where $D_0$ is the upper bound on $g(x)$. Now looking at the term describing $\Delta_{(2)}$, we divide the expectation into two regimes: $\Delta_{(2)} \leq 1/\sqrt{T}$ and $\Delta_{(2)} > 1/\sqrt{T}$. As it is shown in Lemma \ref{lem:Deltas}, $U = \Delta_{(2)}$ has the following density
	\begin{align*}
	g_{U}(u) &= k(k-1) \int_0^{1-u} g(x)g(x+u) G(x)^{k-2} dx \\
	&\leq k (k-1) D_0 \int_0^{1-u} g(x) G(x)^{k-2} dx
	= k D_0 G(1-u)^{k-1}
	\leq k D_0\,,
	\end{align*}
	where we used the fact that $\frac{d}{dz} G^{k-1}(z) = (k-1) g(z) G^{k-2}(z)$ together with the fundamental theorem of calculus. Note that here $G$ is the cumulative of distribution $\Gamma$ defined according to $G(z) = \int_0^z g(x) dx$. Hence, we can write
	\begin{align*}
	\E \bb{\min \bp{\frac{20 + 36 \log f(T)}{\Delta_{(2)}},T \Delta_{(2)}}}
	&\leq D_0 k \Bigg[ \int_{0}^{1/\sqrt{T}} T z dz + \int_{1/\sqrt{T}}^{1} \bp{\frac{20 + 36 \log f(T)}{z}} dz \Bigg] \\
	&\leq \frac{D_0 k}{2} + \frac{D_0 k}{2} + D_0 k \bp{20 + 36 \log f(T)} \int_{1/\sqrt{T}}^{1} \frac{1}{z} dz \\
	&= D_0k + D_0k \bp{10+18\log f(T)} \log T\,.
	\end{align*}
	Summing up all these terms, implies the following Bayesian regret bound for UCB:
	\begin{align*}
	\BR_{T,k}(\text{UCB})
	&\leq k + D_0k (20 + 36 \log f(T))  \sum_{i=3}^k \frac{1}{i-2}  + D_0k
	+ D_0 k \bp{10+18\log f(T)} \log T \\
	&\leq k +D_0 k(20+36\log f(T))\bp{1+\log k}
	+ D_0k + D_0k(10+18\log f(T)) \log T \\
	&\leq k \bp{1 + D_0 + D_0(10 + 18\log f(T)) (2+ 2 \log k + \log T)} \,,
	\end{align*}
	as desired.
\Halmos
\endproof

\proof{Proof of Theorem \ref{thm:upp-large-k}}
	While this theorem can be proved from Theorem \ref{thm:upp-small-k}, here we provide a proof that does not require an upper bound on the density. First note that the regret on any instance $\bmu=\bp{\mu_1, \mu_2, \ldots, \mu_k}$, with subsampled arms $\bmu_m=(\mu_{i_1}, \mu_{i_2}, \cdots, \mu_{i_m})$ can be written as
	\begin{equation*}
		R_T(\text{SS-UCB} \mid \bmu) = R_T(\text{SS-UCB}\mid \bmu_m) + T(\max_{j=1}^k \mu_j - \max_{l=1}^m \mu_{i_l})\,.	
	\end{equation*}
	Taking an expectation implies that
	\begin{align*}
	\BR_{T,k}(\text{SS-UCB})
	&= \BR_{T,m}(\text{SS-UCB}) + T \bp{\E \bb{\max_{j=1}^k \mu_j} - \E \bb{\max_{l=1}^m \mu_l}} \\
	&\leq \BR_{T,m}(\text{SS-UCB}) + T \bp{1 - \bb{\max_{l=1}^m \mu_{i_l}}}\,.
	\end{align*}
	Now, as $m$ arms $i_1, i_2, \ldots, i_m$ have the same distribution as $\mu_1, \mu_2, \ldots, \mu_k$ (i.e., $\Gamma$), the first term is Bayesian regret of UCB algorithm with $m$ arms. From here we replace them with $\bmu_m=(\mu_1, \mu_2, \cdots, \mu_m)$ for simplicity. Our goal is to prove that both terms are $\tilde{O} \bp{\sqrt T}$. For the second term, let $\theta = \log T/(m c_0)$ and write
	\begin{equation*}
		\P \bp{\max_{l=1}^m \mu_l < 1-\theta} \leq \bp{1-c_0 \theta}^m \leq \exp(-c_0 \theta m) = \frac{1}{T} \,.
	\end{equation*}
	Hence,
	\begin{equation*}
		T \bp{1 - \bb{\max_{l=1}^m \mu_l}} \leq 1 + T \theta \leq 1 + \frac{\log T \sqrt{T}}{c_0}\,.
	\end{equation*}
	We now proceed to proving that the first term is also small. Note that with probability at least $1-1/T$, $\max_{l=1}^m \mu_l$ is larger than $1-2 \theta$. Let this event be $A$. Now note that for these instance in $A$ we can write
	\begin{equation*}
		R_T \bp{\text{SS-UCB} \mid \bmu_m} = \sum_{l=1}^m \Delta_i \E[N_{i_l}(T)] \leq \sum_{i: \mu_i < 1-3 \theta} \Delta_i \E[N_i(T)] + \sum_{i: \mu_i \geq 1 - 3\theta} \E[N_i(T)] 3\theta\,.
	\end{equation*}
	Furthermore, the upper bound on the pulls of UCB (see Theorem 8.1 of \cite{lattimore2018bandit}) for any arm with $\mu < 1 - 3 \theta$
	\begin{equation*}
		\Delta_i \E \bb{N_i(T)} \leq 1 + \frac{20 + 36 \log f(T)}{\Delta_i} \leq 1 + \frac{20 + 36 \log f(T)}{1-2\theta - \mu_i}\,.
	\end{equation*}
	Hence, we can write
	\begin{equation*}
		R_T \bp{\text{SS-UCB} \mid \bmu_m} \leq \Ind \bp{\bar{A}} T + \Ind \bp{A} \bp{3 \theta T \sum_{\mu_i \geq 1-3\theta} 1 + \sum_{\mu_i < 1 - 3\theta} \bp{1 + \frac{20 + 36 \log f(T)}{1-2\theta - \mu_i}}}\,.
	\end{equation*}
	Taking expectation with respect to $\bmu_m$ (and hence $A$) implies that
	\begin{equation*}
		\BR_{T,m} \bp{\text{SS-UCB}} \leq 1 + 3 \theta Tm \P \bp{\mu \geq 1-3 \theta} + \bp{20 + 36 \log f(T)} m \E \bb{\Ind \bp{\mu < 1-3\theta} \bp{1 + \frac{1}{1-2\theta-\mu}}}\,.
	\end{equation*}
	Now, note that $\P(\mu \geq 1-3\theta) \leq 3 C_0 \theta$. Hence, Lemma \ref{lem:1gaps} implies that
	\begin{equation*}
		\BR_{T,m} \bp{\text{SS-UCB}} \leq 1 + 9 C_0 \theta^2 m T + C_0 (20 + 36 \log f(T)) \bp{5 + \log(1/\theta)} m \,.
	\end{equation*}
	Replacing $\theta$ and $m$ implies the result.
\Halmos
\endproof

\section{Proofs of Section \ref{sec:greedy}}\label{app:greedy-proofs}

\subsection{Proof of Lemma \ref{lem:greedy-gen}: A General Upper Bound for Greedy and SS-Greedy}
\label{app:greedy-general-upper}

Before proceeding with the proof, we need some notations. Note that we have
\begin{equation*}
R_T \bp{\text{Greedy} \mid \bmu} = \sum_{i=1}^T \Delta_i \E[N_i(T)]\,,
\end{equation*}
where $\Delta_i = \max_{l=1}^k \mu_l - \mu_i$ is the suboptimality gap of arm $i$. Define $\tilde{R}_T(\text{Greedy}) = \sum_{i=1}^k (1-\mu_i) \E[N_i(T)]$. In other words, $\tilde{R}_T(\text{Greedy})$ is the regret of greedy when arms are compared to $1$. Note that this implies that for any realization $\bmu$ we have
\begin{equation*}
R_T \bp{\text{Greedy} \mid \bmu} \leq \tilde{R}_T(\text{Greedy} \mid \bmu)\,.
\end{equation*}
Note that $\BR_{T,k}(\text{Greedy})$ is the expectation of $R_T(\text{Greedy} \mid \bmu)$. Similarly, we can define the expectation of $\tilde{R}_T(\text{Greedy})$ as $\widetilde{\BR}_{T,k}(\text{Greedy})= \E[\tilde{R}_T(\text{Greedy})]$ which implies that
\begin{equation*}
\BR_{T,k}(\text{Greedy}) \leq \widetilde{\BR}_{T,k}(\text{Greedy})\,.
\end{equation*}
Our goal is to prove that $\widetilde{\BR}_{T,k}(\text{Greedy})$ satisfies the inequality given in the statement of Lemma \ref{lem:greedy-gen} and hence the result. Note that the replacement of $\max_{i=1}^k \mu_i$ with a constant $1$ will be helpful in proving the second part of this lemma.
	
Fix the realization $\bmu = (\mu_1, \mu_2, \cdots, \mu_k)$ and partition the interval $(0,1)$ into sub-intervals of size $\delta$. In particular, let $I_1, I_2, \ldots, I_h$ be as $I_1 = [1-\delta, 1]$ and $I_j = [1-j\delta, 1-(j-1)\delta)$, for $j \geq 2$, where $h = \lceil{1/\delta \rceil}$. For any $1 \leq j \leq h$, let $A_j$ denote the set of arms with means belonging to $I_j$. Suppose that for each arm $i \in [k]$, a sequence of i.i.d. samples from $P_{\mu_i}$ is generated. Then, using definition of function $q_{1-2\delta}$ in Eq. \eqref{eqn:max-cross}, for any arm $i \in A_1$ we have
\begin{equation*}
	\P \bb{\exists t \geq 1: \hmu_{i}(t) \leq 1-2\delta} = 1 - q_{1-2\delta}(\mu_i)\,,
\end{equation*}
where $\hmu_{i}(t) = \sum_{i=1}^t Y_{it}/t$ is the empirical average of rewards of arm $i$ at time $t$. Now, consider all arms that belong to $A_1$. As the above events are independent, this implies that for the (bad) event $\bar{G} = \cap_{i \in A_1} \bb{\exists t \geq 1: \hmu_{i}(t) \leq 1-2\delta}$ we have
\begin{equation*}
\P \bb{\bar{G}} \leq \prod_{i \in A_1} (1-q_{1-2\delta}(\mu_i))\,.
\end{equation*}
The above bounds the probability that $\bar{G}$, our bad event happens. Now note that if the (good) event $G$ happens, meaning that there exists an arm in $A_1$ such that its empirical average never drops below $1-2\delta$ we can bound the regret as follows. First note that
\begin{equation*}
\tilde{R}_T(\text{Greedy} \mid \bmu) = \sum_{i=1}^k (1-\mu_i) \E[N_i(T)]\,,
\end{equation*}
Our goal is to bound $\E[N_i(T)]$ for any arm that belongs to $\cup_{j=4}^h A_j$. Indeed, a standard argument based on subgaussianity of arm $i$ implies that
\begin{align*}
\mathbb{P}(N_i(T) \geq t + 1 \mid \bmu)
\leq
\mathbb{P}(\hmu_i(t) \geq 1- 2\delta \mid \bmu)
\leq
\exp \bp{-t(1-2\delta-\mu_i)^2/2},
\end{align*}
where we used the fact that a suboptimal arm $i$ (with mean less than $1-3\delta$) only would be pulled for the $t+1$ time if its estimate after $t$ samples is larger than $1-2\delta$ and that (centered version of) $\hmu_i(t)$ is $1/t$-subgaussian. Now note that for any discrete random variable $Z$ that only takes positive values, we have $\mathbb{E}[Z] = \sum_{t=1}^{\infty} \mathbb{P}(Z \geq t)$. Hence,
\begin{align*}
\mathbb{E}[N_i(T) \mid \bmu]
= 1 + \sum_{t \geq 1} \mathbb{P}(N_i(T) \geq t+1 \mid \bmu)
&\leq 1 + \sum_{t=1}^{\infty} \exp\bp{\frac{-t(1-2\delta-\mu_i)^2}{2}} \\
&= 1 + \frac{1}{1-\exp \bp{\frac{(1-2\delta-\mu_i)^2}{2}}}\,.
\end{align*}
Now note that for any $z \in [0,1]$ we have $\exp(-z) \leq 1 - 2C_1z$ where $C_1=(1-\exp(-1))/2$. Hence, we have
\begin{equation*}
\mathbb{E}[N_i(T) \mid \bmu] \leq 1 + \frac{1}{C_1(1-2\delta-\mu_i)^2}\,,
\end{equation*}
which implies that
\begin{align*}
(1-\mu_i) \mathbb{E}[N_i(T) \mid \bmu]
= (1-\mu_i) + \frac{1-\mu_i}{C_1(1-2\delta-\mu_i)^2}
\leq 1 + \frac{3}{C_1(1-2\delta-\mu_i)}\,,
\end{align*}
where we used the inequality $1-\mu_i \leq 3(1-2\delta-\mu_i)$ which is true as $\mu_i \leq 1-3\delta$. Note that the above is valid for any arm that belongs to $\cup_{j=4}^h A_j$. Furthermore, the expected number of pulls of arm $i$ cannot exceed $T$. Hence,
\begin{equation*}
(1-\mu_i) \mathbb{E}[N_i(T) \mid \bmu]
\leq \min \bp{1 + \frac{3}{C_1(1-2\delta-\mu_i)}, T(1-\mu_i)} \,.
\end{equation*}
As a result,
\begin{equation*}
	R_T \bp{\text{Seq-Greedy} \mid \bmu} \leq T \P (\bar{G})  + 3T \delta + \sum_{j=4}^h \sum_{i \in B_j}  \min \bp{1 + \frac{3}{C_1(1-2\delta-\mu_i)}, T(1-\mu_i)}\,.
\end{equation*}
The last step is to take expectation with respect to the prior. Note that for the first term we can also be rewrite it as $\prod_{i=1}^k \bp{1-\Ind(\mu \geq 1-\delta) q_{1-2\delta}(\mu_i)}$. Hence, taking expectation with respect to $\mu_1, \mu_2, \cdots, \mu_k \sim \Gamma$ implies
\begin{align*}
	\widetilde{\BR}_{T,k} \bp{\text{Greedy}}
	&\leq  T \bp{1 - \E_\Gamma \bb{\Ind \bp{\mu \geq 1-\delta} q_{1-2\delta}(\mu))}}^k
	+3T \delta \\
	&+ k \E_\Gamma \bb{\Ind \bp{\mu < 1-3\delta} \min \bp{1+ \frac{3}{C_1(1-2\delta-\mu)}, T(1-\mu)}}\,,
\end{align*}
as desired. The proof of second part follows from the fact that $\widetilde{\BR}_{T,k} \bp{\text{SS-Greedy}} = \widetilde{\BR}_{T, m} \bp{\text{Greedy}}$ and that $\BR_{T,k} \bp{\text{SS-Greedy}} \leq \widetilde{\BR}_{T,k} \bp{\text{SS-Greedy}}$. \Halmos

\subsection{Proofs for Bernoulli Rewards}\label{app:greedy-bern}

\subsubsection{Proof of Lemma \ref{lem:bern-crossing}.}
Here we provide a more general version of Lemma \ref{lem:bern-crossing} and prove it.

\begin{lemma}\label{lem:bern-crossing-general} Suppose $P_\mu$ is the Bernoulli distribution with mean $\mu$ and that $m \geq 4$ exists such that $1/m \leq 1 - \theta < 1/(m-1)$. Then, for any $\mu > \theta$ we have
		\begin{equation*}
		q_{\theta}(\mu) \geq \mu^m \bp{1-\bb{\frac{\mu}{(1-\mu)(m-1)}}^{-m/(m-1)}}\,.
		\end{equation*}
	Furthermore, if $\mu \geq (1+\theta)/2$ the above quantity is at least $\exp(-0.5)/3$.
\end{lemma}

First note that the above result implies Lemma \ref{lem:bern-crossing} as condition $\theta > 2/3$ ensures the existence of $m \geq 4$ satisfying $1/m \leq 1 - \theta < 1/(m-1)$. Now we prove the above result.
\proof{Proof of Lemma \ref{lem:bern-crossing-general}}
The proof uses Lundberg's inequality. Note that with probability $\mu^m$ we have $X_1 = X_2 = \cdots = X_m = 1$. In this case, the ruin will not occur during the first $m$ rounds and as $\theta \leq 1-1/m$ we can write:
\begin{align*}
\P[\exists n: M_n \leq \theta]
&\leq (1-\mu^m) + \mu^m \P \bb{\exists n \geq m+1: M_n \leq 1-1/m \mid X_1= X_2 = \ldots = X_m = 1} \\
&\leq (1-\mu^m) + \mu^m \P \bb{\exists n \geq m+1: \sum_{i=m+1}^n X_i \leq  (n-m) \theta - m(1-\theta)} \\
&\leq (1-\mu^m) + \mu^m \P \bb{\exists n \geq m+1: \sum_{i=m+1}^n Z_i \geq  1} \,,
\end{align*}
where $Z_i = -X_i + \theta$. The last term above falls into the framework defined by Proposition \ref{prop:lundberg} for the sequence $\bc{Z_i}_{i=m+1}^{\infty}$ and $u=1$. It is not very difficult to see that the other conditions are satisfied. Indeed, as $\E[Z_i] = \theta - \mu < 0$ and that $Z_i$ is subgaussian, the moment generating function $M_Z(s) = \E[\exp(sZ)]$ is defined for all values of $s$. Furthermore, $M_Z(0) = 1, M'_Z(0) = \E[Z_i] < 0$. Finally, we can assume that $\lim_{s \rightarrow \infty} M_Z(s) = +\infty$, which basically means that $\P[Z>0] > 0$. Note that if we assume the contrary, then $P[Z>0] = 0$, which means that the probability that $\sum_{i=m+1}^n Z_i$ crosses $1$ is actually zero, making the claim obvious. Hence, due to continuity of $M_Z(s)$ there exists $\gamma>0$ such that $M_Z(\gamma) = 1$. Finally, note that on the set that $\eta(1) = \infty$ (meaning that crossing $1$ never happens), using the fact that $\E[Z_i] < 0$ and that all moments of $Z$ exist and are bounded, Strong Law of Large Numbers implies that $\sum_{i=m+1}^n Z_i \rightarrow -\infty$ almost surely. This shows that the conditions of Proposition \ref{prop:lundberg} are satisfied. Hence,
\begin{equation*}
\P \bb{\exists n \geq m+1: \sum_{i=m+1}^n Z_i \geq  1} \leq \exp \bp{- \gamma}\,,
\end{equation*}
where $\gamma$ is the Lundberg coefficient of the distribution $-X + \theta$. Lemma \ref{lem:bern-gamma} which is stated below provides a lower bound on $\gamma$. In particular, it shows that
\begin{equation*}
\gamma \geq \frac{m}{m-1} \log \bp{\frac{\mu}{(1-\mu)(m-1)}}\,,
\end{equation*}
Substituting this and noting that $q_{\theta}(\mu) = 1- \P[\exists n: M_n \leq \theta]$ implies the first part.

For the second part, note that we have
\begin{equation*}
q_{\theta}(\mu) \geq \mu^m - \mu^m \bp{\frac{\mu}{(1-\mu)(m-1)}}^{-\frac{m}{m-1}}\,.
\end{equation*}
Now note as $m/(m-1) > 1$ we can write:
\begin{align*}
q_{\theta}(\mu)
\geq \mu^m - \mu^m \bp{\frac{(1-\mu)(m-1)}{\mu} }
\geq \mu^{m-1} \bp{\mu-(1-\mu)(m-1)}\,.
\end{align*}
Hence, as $\mu \geq (1+\theta)/2 \geq 1 - 1/2(m-1)$ which implies that
\begin{align*}
q_{\theta}(\mu)
\geq \bp{1-\frac{1}{2(m-1)}}^{m-1} \bp{1-\frac{m}{2(m-1)}} \geq \exp(-0.5) \bp{1-\frac{2}{3}}
\geq \exp(-0.5)/3\,,
\end{align*}
as desired.

\begin{lemma}\label{lem:bern-gamma}
Let $Z$ be distributed according to $\theta - \mathcal{B}(\mu)$, with $\theta = 1-1/m < \mu$. Then, the following bound on $\gamma$, the Lundberg coefficient of distribution $Z$ holds:
\begin{equation*}
\frac{m}{m-1} \log \bp{\frac{\mu}{(1-\mu)(m-1)}} \leq \gamma \leq  \frac{2m}{m-1} \log \bp{\frac{\mu}{(1-\mu)(m-1)}}
\end{equation*}
\end{lemma}
\proof{Proof.}
Note that $\gamma$ is the non-zero solution of equation $\E[\exp(\gamma Z)] = 1$. Therefore, $\gamma$ satisfies
\begin{equation*}
\mu \exp(\gamma (\theta-1)) + (1-\mu) \exp(\gamma \theta) = 1 \rightarrow \mu \exp \bp{-\frac\gamma m} + (1-\mu) \exp \bp{\frac{m-1}{m} \gamma} =1 \,.
\end{equation*}
Let $x = \exp(\gamma/m)$. Then, $x$ solves $\mu + (1-\mu) x^m = x$, which implies that
\begin{equation*}
\mu - (1-\mu) \bp{x+x^2+x^3+\ldots+x^{m-1}} = 0 \rightarrow x+x^2+\ldots+x^{m-1} = \frac{\mu}{1-\mu} \,.
\end{equation*}
Note that as $\mu > \theta = 1-1/m$, it implies that $x>1$ (this could also be derived from $\gamma>0$). As a result
\begin{equation*}
x \geq \bp{\frac{\mu}{(1-\mu)(m-1)}}^{1/(m-1)}\,.
\end{equation*}	
On the other hand, the AM-GM inequality implies that $\mu/(1-\mu) \geq (m-1) x^{m/2}$. Hence,
\begin{equation*}
x \leq \bp{\frac{\mu}{(1-\mu)(m-1)}}^{2/(m-1)}\,.
\end{equation*}
This concludes the proof. \Halmos
\endproof

%

\subsubsection{Proof of Theorem \ref{thm:greedy-opt-bern}.}
	Note that using Lemma \ref{lem:bern-crossing} for any $\delta < 1/6$ and $\mu \geq 1-\delta$ we have $q_{1-2\delta}(\mu) \geq \exp(-0.5)/3 = 0.2$. Furthermore, Assumption \ref{ass:prior} implies that $\P \bp{\mu \geq 1 - \mu} \geq c_0 \delta$. Finally, Lemma \ref{lem:1gaps} shows that
	\begin{equation*}
	\E_\Gamma \bb{\Ind \bp{\mu < 1-3\delta} \min \bp{1+ \frac{3}{C_1(1-2\delta-\mu)}, T(1-\mu)}} \leq \frac{3C_0}{C_1} (5 + \log(1/\delta))\,.
	\end{equation*}
	Replacing in Lemma \ref{lem:greedy-gen}, using $1-x \leq \exp(-x)$, and choosing $\delta = 5 \log T/(kc_0)$ implies
	\begin{equation*}
	\BR_{T,k}(\text{Greedy}) \leq 1 + \frac{15T \log T}{kc_0} + \frac{3C_0 k}{C_1} \bp{5+ \log(kc_0/5) - \log \log T}\,,
	\end{equation*}
	as desired. The proof of second part follows from the second part of Lemma \ref{lem:greedy-gen}. \Halmos

\subsection{Proofs for Subgaussian Rewards}\label{app:greedy-subg}
	
\proof{Proof of Lemma \ref{lem:subg-crossing}}
The proof of this Lemma uses Lundberg's inequality. Let $\xi = \mu-\theta$ and define $Z_i = X_i - \mu + \xi/2$. Then, we are interested in bounding the probability
\begin{equation*}
\P \bb{\exists n \geq 1: \sum_{i=1}^n Z_i < - n \xi/2} \leq \P \bb{\exists n \geq 1: \sum_{i=1}^n Z_i < -\xi/2} \,.
\end{equation*}
We claim that the latter is upper bounded by $\exp(-\xi^2/2)$. To prove this, first we claim that the conditions of Lundberg's inequality stated in Proposition \ref{prop:lundberg} are satisfied for $S_n = \sum_{i=1}^n (-Z_i)$. Denote $M_Z(s) = \E[ \exp(sZ)]$ and note that as $Z_i - \xi/2$ is $1$-subgaussian, $M_Z(s)$ exists for all values of $s$. Further, without any loss in generality, we can assume $\P[Z_i < 0] > 0$, since otherwise $\P \bb{\exists n \geq 1: \sum_{i=1}^n Z_i < - n \xi/2} = 0 \leq \exp(-\xi^2/2)$ is satisfied. Hence, as $\E[Z_i] = \xi/2$, we have that $\lim_{s \rightarrow -\infty} M_z(s)$ goes to $\infty$. Furthermore, $M_Z(0) = 1$ and $M'_Z(0) = \E[Z_1] = \xi/2 > 0$. Therefore, we conclude that there exists a $\gamma > 0$ such that $M_Z(-\gamma) = 0$. Note that $Z_i$ has finite moments and therefore the Strong Law of Large Numbers (SLLN) implies that $S_n$ diverges almost surely to $-\infty$. Therefore, looking at $\tau(u) = \inf \bc{n \geq 1: S_n > u}$, $S_n$ diverges almost surely to $-\infty$ on the set $\bc{\tau(u) = \infty}$. Hence, the conditions of Lundberg's inequality are satisfied and therefore

\begin{equation*}
 \P \bb{\exists n \geq 1: \sum_{i=1}^n Z_i < -\xi/2} \leq \exp(-\gamma \xi/2)\,.
\end{equation*}
Now we claim that $\gamma \geq \xi$ which concludes the proof. Note that $\exp(-\gamma Z_1) = 1$ and $Z_1 - \xi/2$ is $1$-subgaussian. Hence,
\begin{equation*}
\E \bb{\exp \bp{-\gamma (Z_1 - \xi/2)}} = \exp(\gamma \xi /2) \leq \exp(\gamma^2/2)\,,
\end{equation*}
hence $\gamma \geq \xi$. For finishing the proof, note that $q_{\theta}(\mu) = 1 - \P \bb{\exists n \geq 1: \sum_{i=1}^n Z_i < -\xi/2} $ and that for any $0 \leq z \leq 1$ we have $\exp(-z) \geq 1 - e^{-1}z$. \Halmos
\endproof

\proof{Proof of Theorem \ref{thm:greedy-reg-1sg}}
	Note that the result of Lemma \ref{lem:subg-crossing} holds for all $1$-subgaussian distributions. Therefore, we can apply Lemma \ref{lem:greedy-gen} for the choice $q_{1-2\delta}(\mu) \geq 1 - \exp(-\delta^2/2)$. Note that for $0 \leq x \leq 1$ we have $\exp(-x) \leq 1 - e^{-1}x$ and as a result $q_{1-2\delta}(\mu) \geq (e^{-1}/2)\delta^2$. Therefore, combining this with Lemma \ref{lem:1gaps} and inequality $\exp(-x) \geq 1-x$ implies that
	\begin{equation*}
	\BR_{T,k}(\text{Greedy}) \leq T \exp\bp{-\frac{k c_0 \delta^3 e^{-1}}{2}} + 3T \delta + \frac{3k}{C_1} C_0 (5 + \log(1/\delta))\,.
	\end{equation*}
	Picking $\delta =  (2e \log T/kc_0)^{1/3}$ finishes the proof. The second part follows immediately. \Halmos
\endproof

\subsection{Proofs for Uniformly Upward-Looking Rewards}\label{app:greedy-unif-upw}

Similar to previous cases, we first prove the result stated in Lemma \ref{lem:max-crossing}, providing a lower bound on $q_{\theta}(\mu)$ for uniformly upward-looking rewards.

\subsubsection{Proof of Lemma \ref{lem:max-crossing}}

Recall that $M_n = \sum_{i=1}^n X_i/n$ and we are interested in bounding the probability that $M_n$ falls below $\mu - \delta$. In Definition \ref{def:upward-looking}, pick $\theta = \mu -\delta$ and define $Z_i = X_i - \theta$ and note that
\begin{align*}
    1 - q_{\theta}(\mu) = \P[\exists n: M_n < \mu - \delta]
    &=\P \bb{\exists n \geq 1: \frac{\sum_{i=1}^n X_i}{n} < \mu - \delta} \\
    &=\P \bb{\exists n \geq 1: \sum_{i=1}^n Z_i < 0} \,.
\end{align*}
 Note that, without any loss in generality, we can assume that $\P[Z_i < 0] > 0$, since otherwise the above probability is $0$ which is of course less than $\exp(-p_0 \delta/4)$. Note that $Q$ is $p_0$ upward-looking, it implies that one of the following holds:
 \begin{itemize}
 \item $\P[R_{\tau(\theta)}(\theta) \geq 1 ] \geq p_0$ for $R_n(\theta) = S_n - n \theta = \sum_{i=1}^n Z_i$.
 \item $\E[(X_1 - \theta)) \Ind(X_1 \geq \theta)] \geq p_0$.
 \end{itemize}

We consider each of these cases separately.

\textbf{Case 1.} In this case, the random walk $R_n(\theta) = \sum_{I=1}^n X_i - n\theta = \sum_{i=1}^n Z_i$ hits $1$ before going below $0$ with probability at least $p_0$. 
Hence, denoting $\P \bb{R_{\tau(\theta)}(\theta) \geq 1} = p$, letting $N = \tau(\theta)$, by conditioning on value $R_N(\theta)$ we can write
\begin{align*}
\P & \bb{\exists n : R_n(\theta) < 0 } \\
&\leq \P \bb{R_N(\theta) < 0 \text{~~or~~} N = \infty} + \P \bb{R_N(\theta) \geq 1} \times \P \bb{\exists n \geq 1: R_n < 0 \mid R_N(\theta) \geq 1} \\
&\leq (1-p) + p \P \bb{\exists n \geq 1: R_n(\theta) < 0 \mid R_1(\theta) \geq 0, R_2(\theta) \geq 0, \cdots, R_{N-1}(\theta) \geq 0, R_N(\theta)\geq 1} \\
&= (1-p) + p \P \bb{\exists n \geq N+1: R_N(\theta) + \sum_{i=N+1}^n Z_i < 0 \mid R_1(\theta) \geq 0, R_2(\theta) \geq 0, \cdots, R_{N-1}(\theta) \geq 0, R_N(\theta)\geq 1} \\
&\leq (1-p) + p \P \bb{\exists n \geq N+1: 1 + \sum_{i=N+1}^n  Z_i < 0 \mid R_1(\theta) \geq 0, R_2(\theta) \geq 0, \cdots, R_{N-1}(\theta) \geq 0, R_N(\theta)\geq 1} \\
&\leq  (1-p) + p \P \bb{\exists n \geq N+1: \sum_{i=N+1}^n (-Z_i) > 1}
\end{align*}
where in the above we used the fact that $R_n(\theta)$ only depends on $\bc{Z_i}_{i=1}^n$ and that the distribution of $\bc{Z_i}_{i=N+1}^{\infty}$ is independent of $\bc{Z_i}_{i=1}^N$. For the last argument we want to use Lundberg's inequality for the summation $\sum_{i=N+1}^n (-Z_i)$. First note that, $\E[-Z_i]  = \theta - \mu = -\delta < 0$. Therefore, by SLLN (note that $-Z_i$ has finite moments, as its centered version is sub-gaussian) $\lim \sum_{i=N+1}^{\infty} (-Z_i) \rightarrow -\infty$ almost surely. Hence, for using Lundberg's inequality stated in Proposition \ref{prop:lundberg}, we only need to show the existence of $\gamma > 0$ such that $\E[\exp(\gamma (-Z_1))] = 1$. For proving this, let $M(s) = \E[\exp(sZ_1)]$ be the moment generating function of $Z_1$ which exists for all values of $s$ (due to sub-gaussianity of $Z_1 - \E[Z_1]$). Note that $M(0) = 1$ and $M'(0) = \E[Z_1] = \delta > 0$. Note that as $P[Z_1<0]>0$ and as $M(s)$ is defined for all values of $s$, there exists $\gamma>0$ such that $M(-\gamma) = 1$. In other words, given the condition $P[Z_1<0]>0$, $\lim_{s \rightarrow -\infty} M(s)$ goes to $+\infty$. Hence, due to the continuity of $M(s)$, there exists $\gamma>0$ such that $\E[\exp(-\gamma Z_1)] = 1$. Hence, we can apply Lundberg's inequality which states that:
\begin{align*}
\P \bb{\exists n \geq N: \sum_{i=N+1}^n (-Z_i) > 1}
= \P \bb{\exists n \geq N: \sum_{i=1}^n (-Z_i) > 1}
= \psi(1) \leq \exp(-\gamma) \,.
\end{align*}
 Now we claim that $\gamma \geq 2 \E[Z_1] = 2\delta$. This is true according to $Z_1 - \E[Z_1]$ being $1$-subgaussian and that
 \begin{equation*}
 \E[\exp \bp{-\gamma (Z_1-\E[Z_1]}] =\exp(\gamma \delta) \leq \exp(\gamma^2/2) \,,
 \end{equation*}
 which proves our claim. Combining all these results and using $p \geq p_0$ we have
 \begin{align*}
 	1 - q_{\theta}(\mu) = \P[\exists n: M_n < \mu - \delta]
	&\leq (1-p_0) + p_0 \exp(-2 \delta) \\
	&\leq (1-p_0) + p_0 (1-\delta/4) \\
	&\leq \exp(-p_0 \delta/4) \,,
 \end{align*}
 where we used the inequality $1-z \leq \exp(-z) \leq 1-z/8$ for $z = 2\delta$ which is true for any $z \leq 2$ (or equivalently, $\delta \leq 1$).Thus, using inequality $\exp(-x) \leq 1-e^{-1}x$ for $x \in [0,1]$ we have
 \begin{equation*}
 	q_\theta(\mu) \geq 1-\exp(-p_0 \delta/4) \geq e^{-1}p_0 \delta/4\,,
 \end{equation*}
 as desired.

 \textbf{Case 2.} Proof of this part is similar and relies again on Lundberg's inequality. Here, we are going to condition on the value of $Z_1$ which is independent of $Z_i$ for $i \geq 2$. Our goal is to relate the desired probability to $S'_n = -\sum_{i=2}^n Z_i$. Hence,
\begin{align*}
    1-q_\theta(\mu) = \P \bb{\exists n: \sum_{i=1}^n Z_i < 0}
    &=\P[Z_1 < 0] + \int_{Z_1 \geq 0} \P \bb{\exists n \geq 2: Z_1 + \sum_{i=2}^n Z_i > 0 \mid Z_1} dP_{Z_1} \\
    &=\P[Z_1 < 0] +\int_{Z_1 \geq 0} \P \bb{\exists n \geq 2: Z_1 - S'_n < 0} dP_{Z_1} \\
    &=\P[Z_1 < 0] +\int_{Z_1 \geq 0} \P \bb{\exists n \geq 2: S'_n > Z_1} dP_{Z_1} \\
    &=\P[Z_1 < 0] + \int \psi(Z_1) \Ind(Z_1 \geq 0) dP_{Z_1} \,,
\end{align*}
where $\psi(\cdot)$ is the ultimate ruin probability for the random walk $S_n = -\sum_{i=2}^n Z_i$. Now we want to apply Proposition \ref{prop:lundberg} to provide an upper bound on $\psi(z)$. Similar to the previous case, we can show that the conditions of Lundberg's inequality holds and hence Proposition \ref{prop:lundberg} implies that $\psi(z) \leq \exp(-\gamma z)$ where $\gamma$ is the Lundberg coefficient of distribution $-Z_1$ that satisfies $\E[\exp(-\gamma Z_1)] = 1$.
Similar to the previous case we can show that $\gamma \geq 2 \E[Z_1] \geq 2 \delta$ and therefore, $\psi(z) \leq \exp(-\gamma x) \leq \exp(-2 \delta z)$ holds for all $z \geq 0$. Hence,
\begin{equation*}
    1-q_\theta(\mu) = \P[\exists n: M_n < \theta] \leq
    \P[Z_1 < 0] + \E[\exp(-2 \delta Z_1) \Ind \bp{Z_1 \geq 0}] \,.
\end{equation*}
Now we can use the inequality $\exp(-t) \leq 1 - t/e$ which is true for $0 \leq t \leq 1$. Replacing $t = 2 \delta z$, implies that for $z \leq \frac{1}{2\delta}$ we have
\begin{equation*}
    \exp(-2 \delta z) \leq 1 - \frac{2 \delta z}{e} \,.
\end{equation*}
Therefore,
\begin{align*}
    \int \exp(-2 \delta Z_1) \Ind \bp{Z_1 \geq 0} dP_{Z_1}
    &= \int \exp(-2 \delta Z_1) \Ind \bp{0 \leq Z_1 \leq \frac{1}{2\delta}} dP_{Z_1} + \int \exp(-2 \delta Z_1) \Ind \bp{Z_1 > \frac{1}{2\delta}} dP_{Z_1} \\
    \quad&\leq P \bp{Z_1 \geq 0} - \frac{2 \delta}{e} \int \bp{0 \leq Z_1 \leq \frac{1}{2\delta}} Z dP_{Z_1}\,.
\end{align*}
Hence,
\begin{align*}
    1-q_\theta(\mu) = \P[\exists n: M_n < \theta] \leq 1 - \frac{2\delta}{e} \E\bb{Z_1 \Ind \bp{0 \leq Z_1 \leq\frac{1}{2\delta}}}\,.
\end{align*}
Our goal is to show that if $\delta \leq 0.05$, then
\begin{equation}\label{eqn:half-exp}
    \E\bb{Z_1 \Ind \bp{0 \leq Z_1 \leq\frac{1}{2\delta}}} \geq \frac{\E \bb{Z_1 \Ind \bp{Z_1 \geq 0}}}{2} \,.
\end{equation}
For proving this note that for $\delta \leq 1/2$ we have $\frac{1}{2\delta} \geq \delta + \frac{1}{4\delta}$. Hence, according to $1$-subgaussianity of $Z_1 - \delta$ we have
\begin{align*}
    \E\bb{Z_1 \Ind \bp{Z_1 > \frac{1}{2\delta}}}
    &=\frac{1}{2\delta} \P \bb{Z_1 > \frac{1}{2\delta}} +\int_{\frac{1}{2\delta}}^\infty \P[Z_1 > t] dt \\
    &\leq \frac{1}{2 \delta} \exp\bp{-\frac{1}{2}\frac{1}{16 \delta^2}} + \int_{\delta+\frac{1}{4\delta}}^\infty \P \bb{Z_1 > t} dt \\
    &\leq \frac{1}{2 \delta} \exp \bp{-\frac{1}{32\delta^2}} + \int_{\frac{1}{4\delta}} \exp(-t^2/2) dt \\
    &\leq \bp{\frac{1}{2 \delta}+\sqrt{2 \pi}} \exp \bp{-\frac{1}{32\delta^2}}\,.
\end{align*}
We claim that the above probability is less than $\delta/2$. In fact, a simple numerical calculation shows that for all $\delta \leq 0.05$ we have
\begin{equation*}
    \bp{\frac{1}{2 \delta}+\sqrt{2 \pi}} \exp \bp{-\frac{1}{32\delta^2}} \leq \frac{\delta}{2}\,.
\end{equation*}
Hence, this implies that
\begin{align*}
    \E\bb{Z_1 \Ind \bp{0 \leq Z_1 \leq\frac{1}{2\delta}}}
    &= \E \bb{Z_1 \Ind \bp{Z_1 \geq 0}} - \E \bb{Z_1 \Ind \bp{Z_1 > \frac{1}{2\delta}}} \\
    &\geq \E \bb{Z_1\Ind \bp{Z_1 \geq 0}} - \frac{\delta}{2} \,.
\end{align*}
The final inequality we need to show is that $\E \bb{Z_1 \Ind \bp{Z_1 \geq 0}} \geq \delta$ which is obvious as $\E[Z_1] = \mu - \theta = \delta$, meaning
\begin{align*}
    \delta = \E \bb{Z_1 \Ind \bp{Z_1 \geq 0}} + \E \bb{Z_1 \Ind \bp{Z_1 < 0}} \leq \E \bb{Z_1 \Ind \bp{Z_1 \geq 0}} \,,
\end{align*}
as $\E \bb{Z_1 \Ind \bp{Z_1 < 0}} \leq 0$. Putting all these results together we have proved our claim in Eq. \eqref{eqn:half-exp}. Hence, using inequality $1-t \leq \exp(-t)$ and that $Z_1 = X_1 - \theta$ we have
\begin{align*}
1-q_\theta(\mu) = \P [\exists n: M_n < \theta]
&\leq 1 - \frac{2\delta}{e} \E\bb{Z_1 \Ind \bp{0 \leq Z_1 \leq\frac{1}{2\delta}}}
\leq 1 - \frac{\delta}{e} \E \bb{(X_1- \theta) \Ind \bp{X_1 \geq \theta}} \\
&\leq \exp \bp{-\frac{p_0}{e} \delta}\,.
\end{align*}
As $1/e \geq 1/4$, and that $\exp(-x) \leq 1-e^{-1} x$ for all $x \in [0,1]$, the conclusion follows.
\Halmos

\subsubsection{Proof of Theorem \ref{thm:greedy-reg}}
	Note that the result of Lemma \ref{lem:max-crossing} holds for all upward-looking distributions. Note that as $\mathcal{F}$ is uniformly $(p_0, \delta_0)$ upward-looking, this result holds if $\delta \leq \delta_1 = \min \bp{\delta_0, 0.05}$. Further, if $\delta_0 > 0.05$, the distribution $\mathcal{F}$ is upward-looking for $(p_0, 0.05)$ as well. Suppose that $\mu \geq 1 -\delta$. Then, we can apply Lemma \ref{lem:greedy-gen} for the choice $q_{1-2\delta}(\mu) \geq 1 - \exp(-p_0 \delta/4) \geq (p_0 e^{-1}/4) \delta$ according to the inequality $\exp(-x) \leq 1 - e^{-1}x$, for any $\delta < \delta_1$. Therefore, combining this with Lemma \ref{lem:1gaps} and inequality $\exp(-x) \geq 1-x$ implies that
	\begin{equation*}
	\BR_{T,k}(\text{Greedy}) \leq T \exp\bp{-\frac{k p_0 c_0 \delta^2 e^{-1}}{4}} + 3T \delta + \frac{3k}{C_1} C_0 (5 + \log(1/\delta))\,.
	\end{equation*}
	Picking $\delta =  \bp{(4e \log T)/(k p_0 c_0)}^{1/2}$ finishes the proof. Note that the condition on $k$ also implies that $\delta < \delta_1$, as desired. The second part follows from the second part of Lemma \ref{lem:greedy-gen} together with $m = \Theta \bp{T^{2/3}}.$

\section{Additional Simulations}\label{app:add-sim}

In this section, we provide a wide range of simulations to validate our claims in the paper. We first start by the stochastic case (which is the main focus of the paper) and then provide some additional simulations with the contextual case as well.

\subsection{Stochastic reward}\label{app:add-sim-stoch}

\paragraph{Setting.} We repeat the analysis described in the introduction and Figure \ref{fig:alph_75} for the wide range of beta priors and for both Gaussian and Bernoulli rewards. We fix $T=20000$ in all our simulations and use two different values for $k$. For each setting, we generate $100$ instances, where in each instance the means $\mu_1, \mu_2, \ldots, \mu_k \sim \Gamma$. Denote our general prior to be $\mathcal{\beta}(a, b)$ (note that uniform corresponds to $a=b=1$). The subsampling rates chosen for various algorithms are based on our theory and Table \ref{tab:gen-beta}. For SS-TS, we use a similar rate as of SS-UCB. Finally, TS and SS-TS use the correct prior information (i.e., $\mathcal{\beta}(a,b)$) for the Bernoulli setting. For the Gaussian case, TS and SS-TS suppose that the prior distribution of all arms are $\mathcal{N}(1/2,1/16)$. In other words, for Gaussian rewards, TS uses a mismatched (but narrow) prior for all arms. The following algorithms are included in our simulations:
\begin{itemize}
	\item UCB: Asymptotically optimal UCB (Algorithm \ref{alg:ucb-asymp}).
	\item SS-UCB: Subsampled UCB algorithm in Algorithm \ref{alg:subs-ucb}, with $m = T^{b/2}$ if $b < 1$ and $m = T^{b/(b+1)}$ if $b \geq 1$.
	\item Greedy: Greedy (Algorithm \ref{alg:greedy}).
	\item SS-Greedy: Subsampled Greedy (Algorithm \ref{alg:subs-greedy}) with $m = T^{b/2}$ if $b < 1$ and $m = T^{b/(b+1)}$ if the reward is Bernoulli (suggested by Theorem \ref{thm:greedy-opt-bern}). For Gaussian rewards, $m=T^{(b+1)/3}$ for $b < 1$ and $m=T^{(b+1)/(b+2)}$ for $b \geq 1$  and (suggested by Theorem \ref{thm:greedy-reg-1sg}).
	\item UCB-F algorithm \cite{wang2009algorithms} with the choice of confidence set $\mathcal{E}_t = 2 \log (10 \log t)$. Note that UCB-F also subsamples $m = T^{b/2}$ if $b < 1$ and $m=T^{b/(b+1)}$ if $b \geq 1$ arms.
	\item TS: Thompson Sampling algorithm \cite{thompson1933likelihood, russo2014learning, agrawal2012analysis}.
	\item SS-TS: Subsampled TS with $m = T^{b/2}$ if $b < 1$ and $m=T^{b/(b+1)}$ if $b \geq 1$.
\end{itemize}

\paragraph{Priors.} For each of Bernoulli and Gaussian case, we consider $\mathcal{\beta}(a,b)$ priors with $a \in \bc{0.5, 1, 2}$ and $b \in \bc{0.8, 1, 1.5}$, leading to $9$ different choices of priors.

\paragraph{Results.} The following two tables summarize the results of these simulations. In fact, Table \ref{tab:add-sim-gauss} contains the result for Gaussian rewards and Table \ref{tab:add-sim-bern} contains the result for Bernoulli rewards. For each simulation, we normalize the mean per-instance regret (an estimate of Bayesian regret) of all algorithms by that of SS-Greedy. In other words, each entry gives an estimate of $\BR_{T,k}(\text{Algorithm})/\BR_{T,k}(\text{SS-Greedy})$.

We also generate $6$ figures (Figures 3-8), similar to Figure \ref{fig:alph_75}, where we plot the distribution of per-instance regret of algorithms discussed above. As can be observed from Table \ref{tab:gen-beta}, only the choice of $b$ changes the rates of various algorithms. For this reason, we only generate the plots for $a=1$ when $b$ is in $\bc{0.8, 1, 1.5}$ for Bernoulli and Gaussian rewards ($6$ in total).

As can be observed from Tables \ref{tab:add-sim-gauss} and \ref{tab:add-sim-bern}, SS-Greedy is the best performer in almost all the cases. Two other competitive algorithms are Greedy and SS-TS. It is clear from this table that subsampling leads to a great improvement (compare UCB with SS-UCB, Greedy with SS-Greedy, and TS with SS-TS) uniformly across all algorithms. For SS-TS, the performance in the Bernoulli case is much better. This is due to using correct information about prior and reward distributions. These two tables show that the superior performance of SS-Greedy is robust to the choice of prior and reward distributions.

\begin{table*}[h]
    \centering
 	\caption{Estimated ratio $\BR_{T,k}(\text{Algorithm})/\BR_{T,k}(\text{SS-Greedy})$ for Gaussian rewards.}
	\label{tab:add-sim-gauss}
    \rowcolors{1}{}{gray!10}
    \begin{tabular}{*8c}
        \toprule
		\textbf{Setting} & \textbf{UCB} & \textbf{SS-UCB} & \textbf{Greedy} & \textbf{SS-Greedy} & \textbf{UCB-F} & \textbf{TS} & \textbf{SS-TS} \\
        \midrule
			$a=0.5, b=0.8, k = 400$ & $7.44$ & $2.74$  & ${\bf 1}$ & ${\bf 1}$ & $4.17$ & $2.86$ & $1.12$ \\
			$a=0.5, b=0.8, k = 1000$ & $9.65$ & $2.64$ & $1.47$  & ${\bf 1}$ & $4.11$ & $4.34$ & $1.16$ \\
			$a=1, b=0.8, k = 400$ & $7.87$ & $3.44$ & $1.13$ & ${\bf 1}$ & $5.01$ & $3.73$ & $1.57$ \\
			$a=1, b=0.8, k = 1000$ & $8.81$  & $3.21$ & $1.62$  & ${\bf 1}$ & $4.65$ & $4.53$ & $1.29$ \\
			$a=2, b=0.8, k = 400$ & $6.21$ & $3.38$ & $1.06$ & ${\bf 1}$ & $4.56$ & $3.59$ & $1.75$ \\
			$a=2, b=0.8, k = 1000$  & $7.79$  & $3.73$  & $1.75$  & ${\bf 1}$ & $4.76$ & $4.82$ & $1.85$ \\
        \midrule
			$a=0.5, b=1, k = 1000$ & $7.39$  & $3.48$  & $1.15$  & ${\bf 1}$ & $9.26$ & $3.37$ & $1.3$ \\
			$a=0.5, b=1, k = 3000$ & $8.31$ & $3.34$  & $2.47$  & ${\bf 1}$ & $9.34$ & $5.01$ & $1.24$ \\
			$a=1, b=1, k = 1000$ & $6.74$ & $3.79$ & $1.2$  & ${\bf 1}$ & $5.56$ & $3.61$ & $1.61$ \\
			$a=1, b=1, k = 3000$ & $7.51$ & $3.63$ & $2.68$  & ${\bf 1}$ & $5.32$ & $4.81$ & $1.55$ \\
			$a=2, b=1, k = 1000$ & $5.35$ & $3.48$ & $1.2$  & ${\bf 1}$ & $6.41$ & $3.35$ & $1.83$ \\
			$a=2, b=1, k = 3000$ & $5.99$ & $3.44$ & $2.68$  & ${\bf 1}$ & $6.71$ & $4.05$ & $1.81$ \\
        \midrule
			$a=0.5, b=1.5, k = 1000$ & $5.54$ & $4.31$ & $1.01$  & ${\bf 1}$ & $6.52$ & $2.53$ & $1.47$ \\
			$a=0.5, b=1.5, k = 3000$ & $5.89$ & $4.05$ & $1.71$  & ${\bf 1}$ & $6.36$ & $3.82$ & $1.49$ \\
			$a=1, b=1.5, k = 1000$ & $5.4$ & $4.43$ & $1.03$  & ${\bf 1}$ & $6.39$ & $2.96$ & $1.94$ \\
			$a=1, b=1.5, k = 3000$ & $5.26$ & $3.83$ & $1.73$  & ${\bf 1}$ & $5.76$ & $3.62$ & $1.73$ \\
			$a=2, b=1.5, k = 1000$ & $4.49$ & $3.82$ & ${\bf 0.95}$  & $1$ & $5.23$ & $2.86$ & $2.13$ \\
			$a=2, b=1.5, k = 3000$ & $4.64$ & $3.57$ & $1.9$  & ${\bf 1}$ & $5.08$ & $3.37$ & $2.03$ \\        
        \bottomrule
    \end{tabular}
\end{table*}

\begin{table*}[h]
    \centering
 	\caption{Estimated ratio $\BR_{T,k}(\text{Algorithm})/\BR_{T,k}(\text{SS-Greedy})$ for Bernoulli rewards.}
	\label{tab:add-sim-bern}
    \rowcolors{1}{}{gray!10}
    \begin{tabular}{*8c}
        \toprule
		\textbf{Setting} & \textbf{UCB} & \textbf{SS-UCB} & \textbf{Greedy} & \textbf{SS-Greedy} & \textbf{UCB-F} & \textbf{TS} & \textbf{SS-TS} \\
        \midrule
			$a=0.5, b=0.8, k = 400$ & $8.13$ & $2.41$ & ${\bf 0.94}$  & $1$ & $3.09$ & $1.12$ & $0.97$ \\
			$a=0.5, b=0.8, k = 1000$ & $14.24$ & $2.67$ & $2.34$  & ${\bf 1}$ & $3.32$ & $2.64$ & $1.14$ \\
			$a=1, b=0.8, k = 400$ & $15.12$ & $5.29$ & $1.87$  & ${\bf 1}$ & $25.87$ & $2.28$ & $1.42$ \\
			$a=1, b=0.8, k = 1000$ & $19.54$ & $4.59$ & $3.69$  & ${\bf 1}$ & $23.04$ & $4.44$ & $1.17$ \\
			$a=2, b=0.8, k = 400$ & $16.2$ & $6.8$ & $2.13$  & ${\bf 1}$ & $25.24$ & $2.76$ & $1.05$ \\
			$a=2, b=0.8, k = 1000$ & $21.12$ & $6.58$ & $3.91$  & ${\bf 1}$ & $30.86$ & $6.28$ & $1.19$ \\
        \midrule
           $a=0.5, b=1, k = 1000$ & $12.19$ & $3.75$ & $1.9$  & ${\bf 1}$ & $20.22$ & $2.2$ & $1.02$ \\
			$a=0.5, b=1, k = 3000$ & $16.47$ & $3.55$ & $5.14$  & $1$ & $21.39$ & $5.35$ & ${\bf 0.94}$ \\
			$a=1, b=1, k = 1000$ & $17.47$ & $6.74$ & $3.22$  & ${\bf 1}$ & $40.16$ & $3.61$ & $1.14$ \\
			$a=1, b=1, k = 3000$ & $26.85$ & $7.56$ & $9.29$  & ${\bf 1}$ & $34.13$ & $10.93$ & $1.53$ \\
			$a=2, b=1, k = 1000$ & $20.52$ & $10.12$ & $4.66$  & ${\bf 1}$ & $29.29$ & $5.33$ & $1.52$ \\
			$a=2, b=1, k = 3000$ & $21.63$ & $8.45$ & $7.25$  & ${\bf 1}$ & $26.84$ & $11.79$ & $1.38$ \\
        \midrule
           $a=0.5, b=1.5, k = 1000$ & $8.34$ & $4.45$ & $1.15$  & ${\bf 1}$ & $13.27$ & $1.51$ & $1.06$ \\
			$a=0.5, b=1.5, k = 3000$ & $11.17$ & $4.06$ & $2.84$  & ${\bf 1}$ & $13.4$ & $3.17$ & $1.19$ \\
			$a=1, b=1.5, k = 1000$ & $10.63$ & $6.62$ & $1.6$  & ${\bf 1}$ & $15.62$ & $2.14$ & $1.35$ \\
			$a=1, b=1.5, k = 3000$ & $12.66$ & $5.93$ & $4.18$  & ${\bf 1}$ & $15.68$ & $4.49$ & $1.31$ \\
			$a=2, b=1.5, k = 1000$ & $10.67$ & $7.8$ & $1.99$  & ${\bf 1}$ & $14.63$ & $2.55$ & $1.47$ \\
			$a=2, b=1.5, k = 3000$ & $11.93$ & $6.86$ & $5.11$  & ${\bf 1}$ & $14.1$ & $5.36$ & $1.38$ \\
        \bottomrule
    \end{tabular}
\end{table*}
	
 \begin{figure}[h]
	\centering
	\includegraphics[width = .85\textwidth]{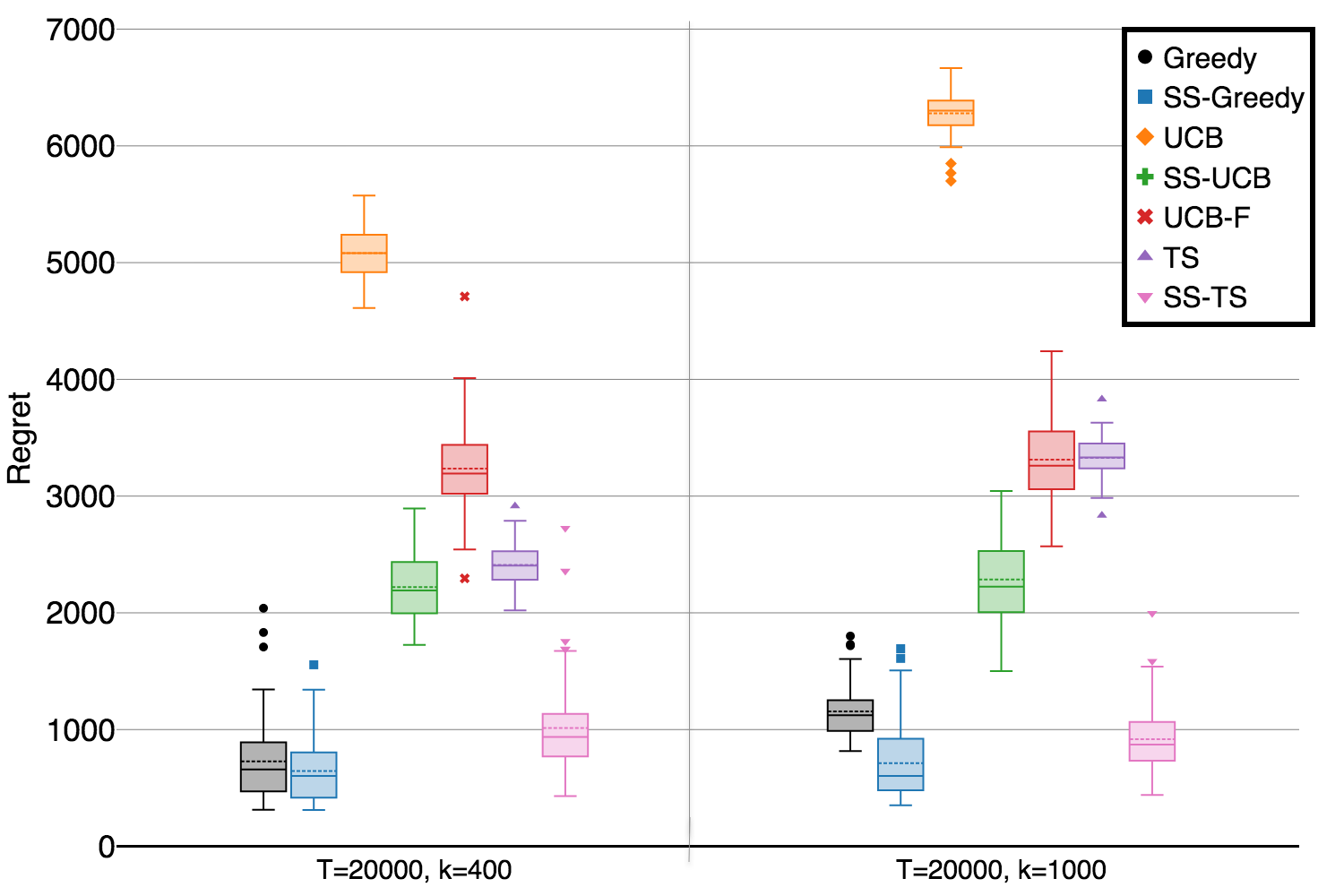}
	\label{fig:norm_b_8}
	\caption{Distribution of the per-instance regret for Gaussian rewards and prior $\Gamma = \mathcal{\beta}(1,0.8).$}
\end{figure}

 \begin{figure*}[h]
	\centering
	\includegraphics[width = .85\textwidth]{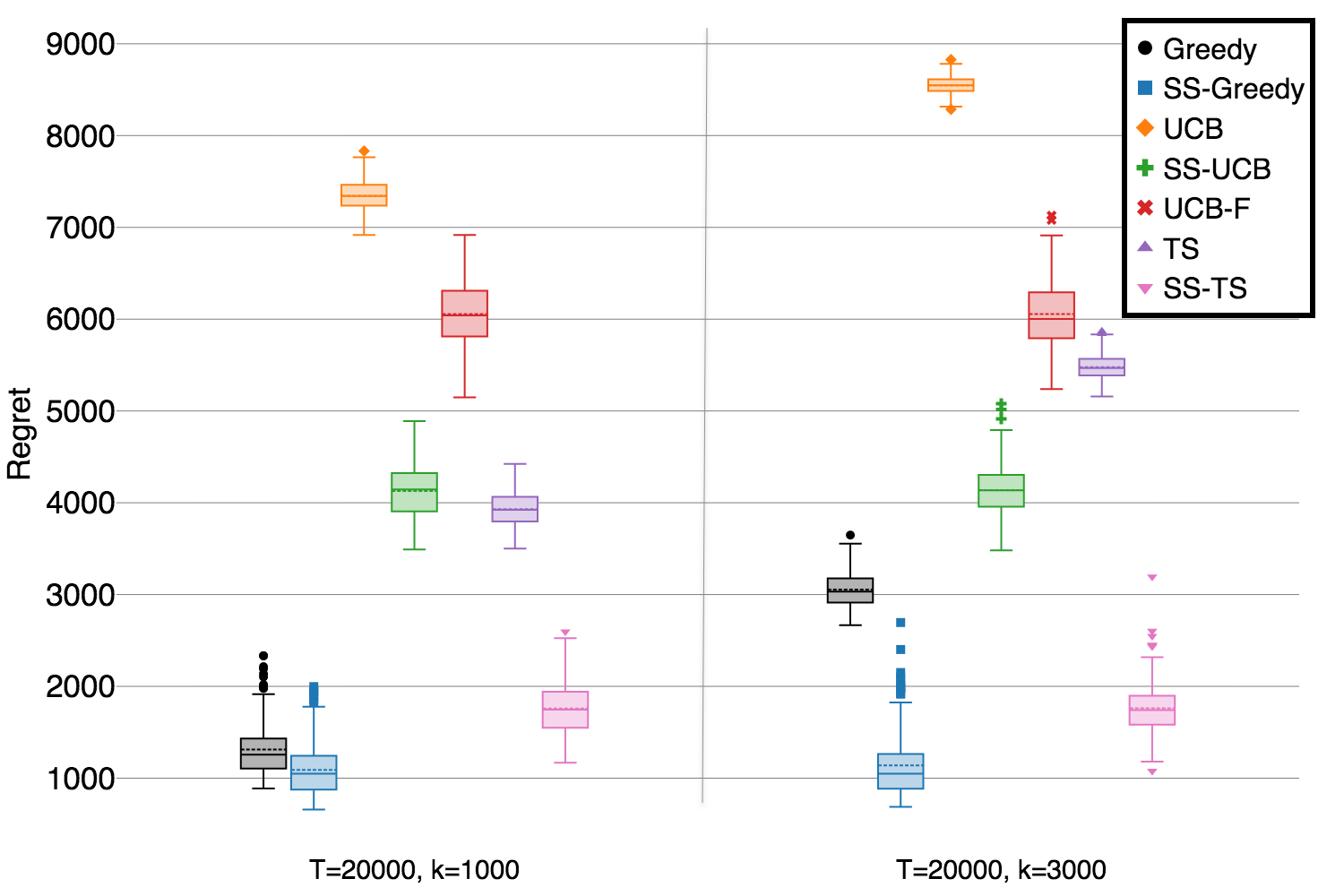}
	\label{fig:norm_b_1}
	\caption{Distribution of the per-instance regret for Gaussian rewards and prior $\Gamma = \mathcal{\beta}(1,1).$}
\end{figure*}

 \begin{figure}[h]
	\centering
	\includegraphics[width = .85\textwidth]{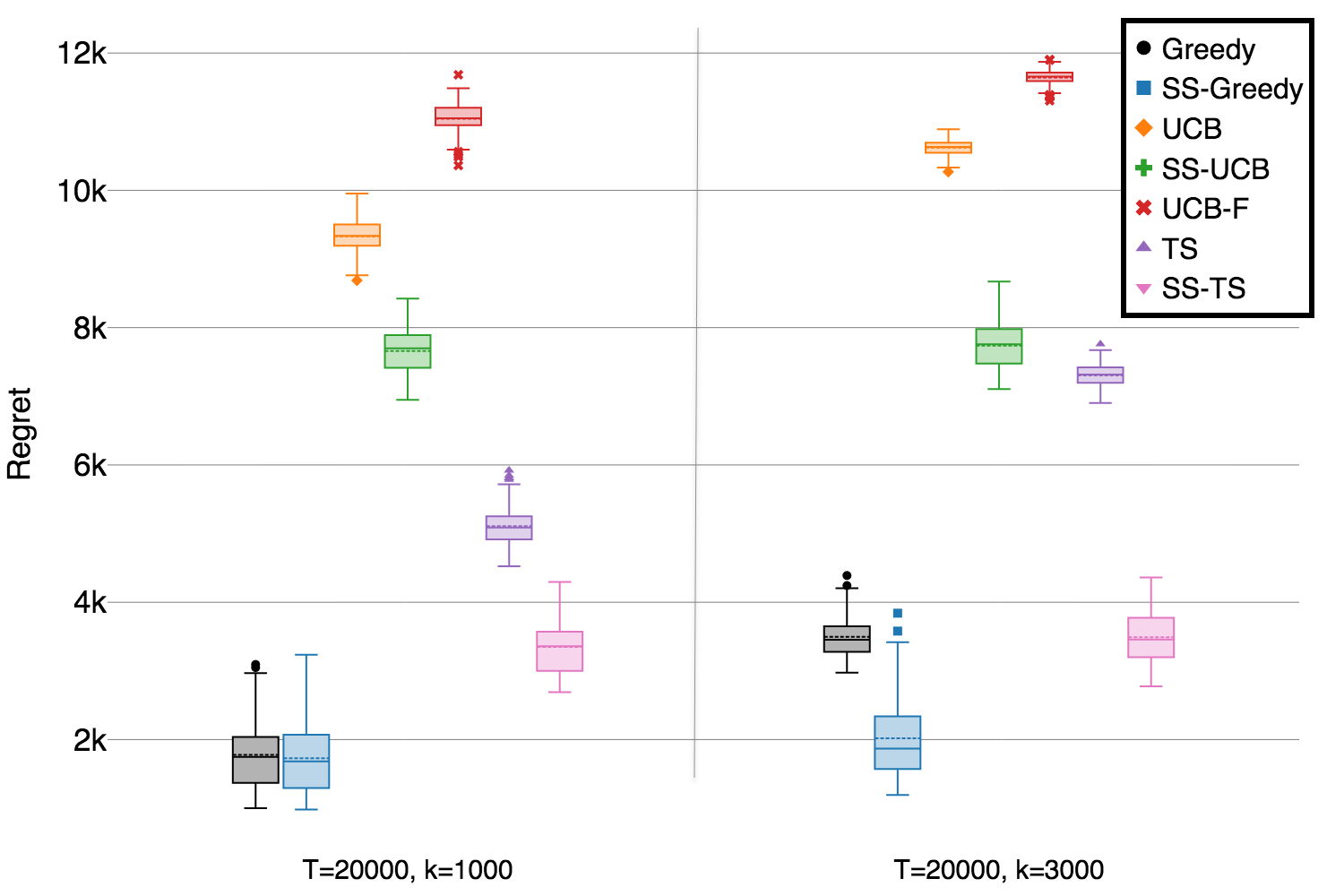}
	\label{fig:norm_b_15}
	\caption{Distribution of the per-instance regret for Gaussian rewards and prior $\Gamma = \mathcal{\beta}(1,1.5).$}
\end{figure}

 \begin{figure}[h]
	\centering
	\includegraphics[width = .85\textwidth]{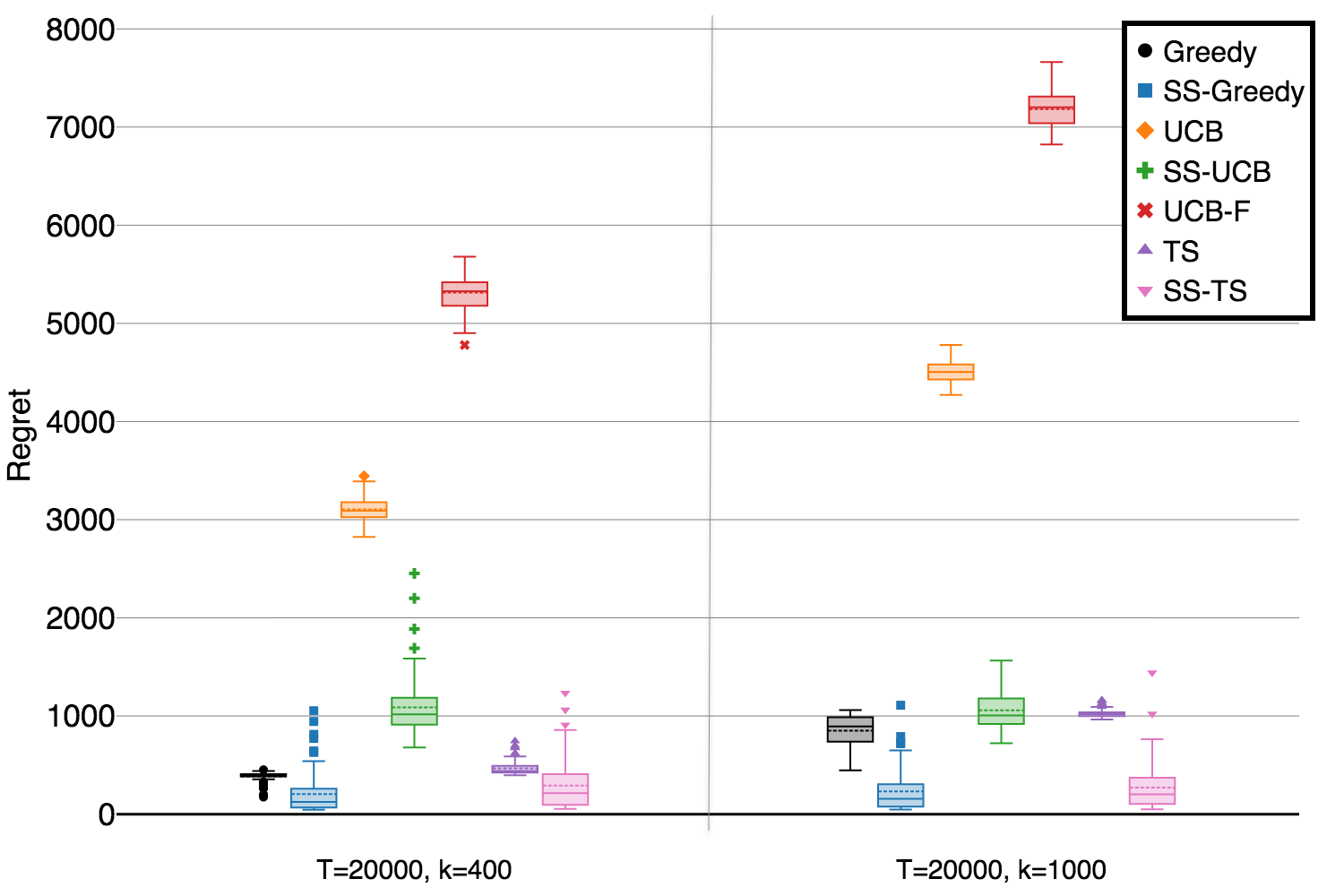}
	\label{fig:bern_b_8}
	\caption{Distribution of the per-instance regret for Bernoulli rewards and prior $\Gamma = \mathcal{\beta}(1,0.8).$}
\end{figure}

\begin{figure}[h]
	\centering
	\includegraphics[width = .85\textwidth]{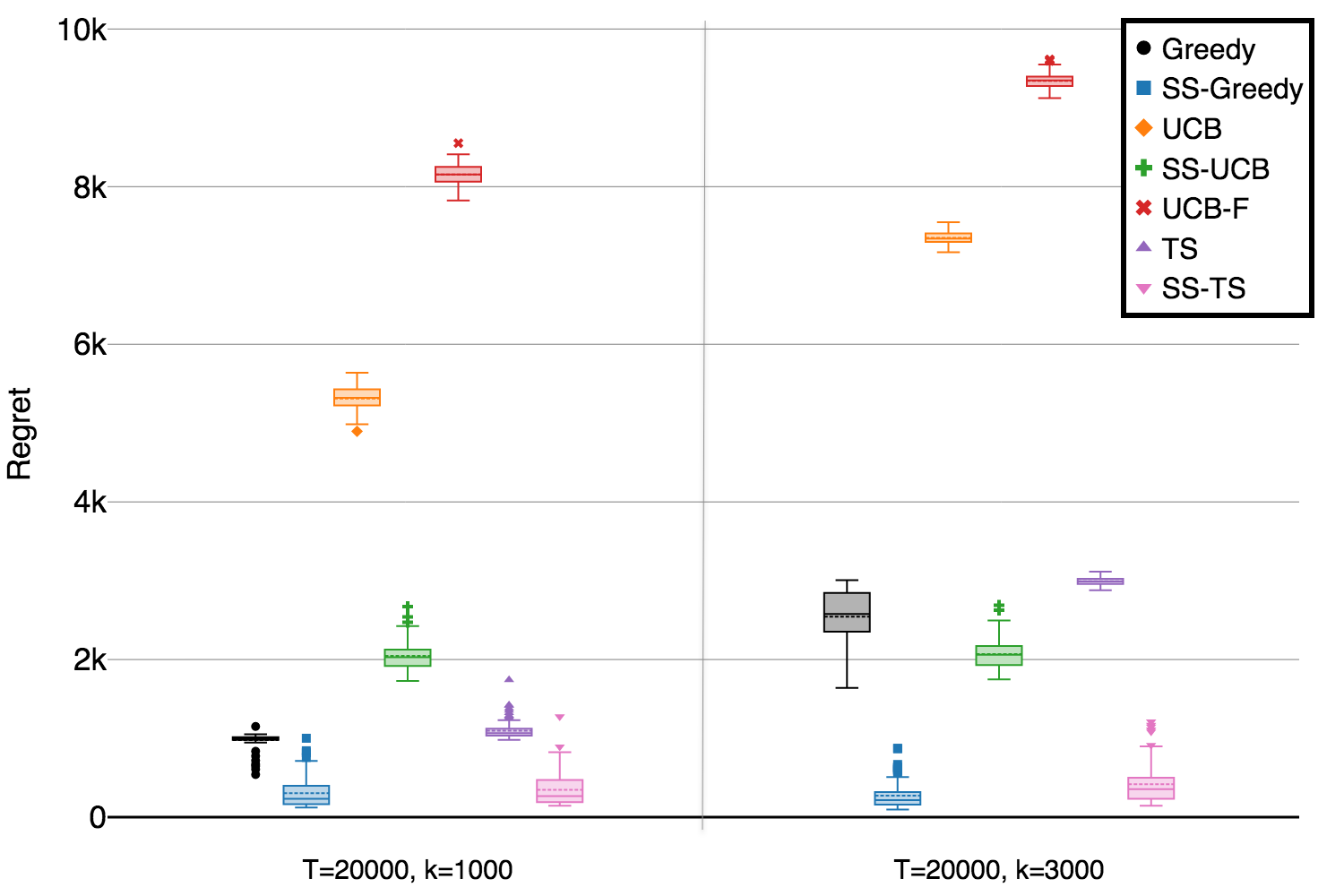}
	\label{fig:bern_b_1}
	\caption{Distribution of the per-instance regret for Bernoulli rewards and prior $\Gamma =\mathcal{\beta}(1,1).$}
\end{figure}

\begin{figure}[h]
	\centering
	\includegraphics[width = .85\textwidth]{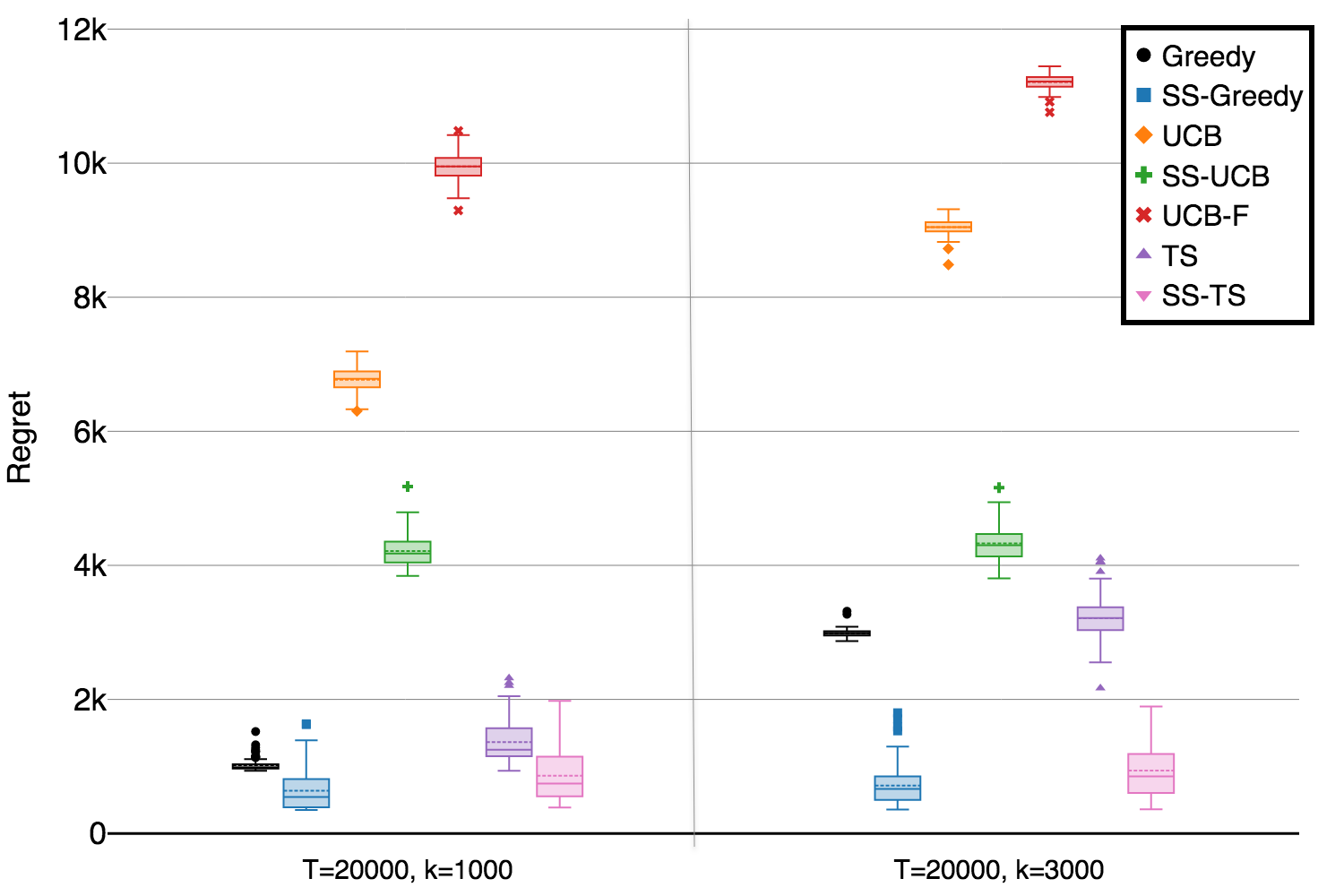}
	\label{fig:bern_b_15}
	\caption{Distribution of the per-instance regret for Bernoulli rewards and prior $\Gamma = \mathcal{\beta}(1,1.5).$}
\end{figure}

\subsection{Contextual reward}

\paragraph{Setting.} We repeat the analysis described in \S \ref{sec:simulations}, where instead of real data, we work with synthetically generated data. We consider the contextual rewards given by $Y_{it} = X_t^\top \theta_i + \vep_{it}$, where $X_t \in \IR^d$ is the (shared) context observed at time $t$ and $\theta_i \in \IR^d$ is the parameter of arm $i$. We fix $T=8000$ and $d=2$ and consider two settings: $k = 200$ and $k=400$. We suppose that $\theta_i \sim \Unif_d = \{u \in \mathbb{R}^d: \|u\|_2 \leq 1\}$ and draw all contexts according to $X_t \sim \mathcal{N}(0,I_d/\sqrt{d})$. Note that the normalization with $\sqrt{d}$ implies that $\E[ \|X_t\|_2^2] = 1$. Also, the noise terms $\vep_{it} \sim \mathcal{N}(0,\sigma^2)$, where $\sigma = 0.5$. We again compare, OFUL, Greedy, and TS, and their subsampled versions. We generate $50$ instances of the problem described above and compare the per-instance regret of these algorithms.

\begin{figure}[hptb]
	\centering
	\includegraphics[height = 6cm]{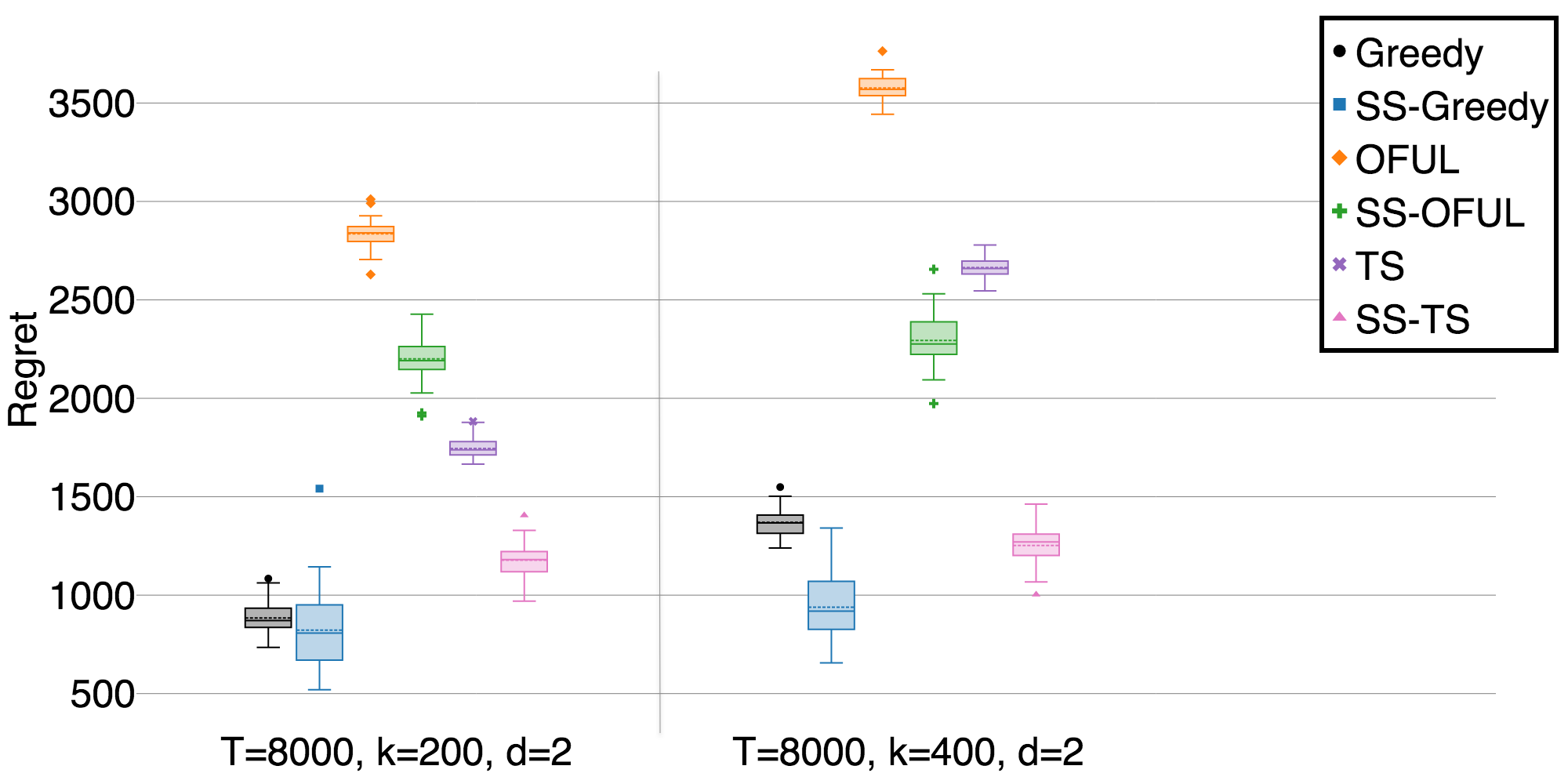}
	\label{fig:context_synt}
	\caption{Distribution of the per-instance regret for the synthetic contextual simulation}
\end{figure}

\paragraph{Results.} The results are depicted in Figure 9. As can be observed: (1) SS-Greedy is the best algorithm in both cases, and (2) subsampling improves the performance of all algorithms. These results are consistent with our findings throughout the paper.

\section{Useful Lemmas}\label{app:useful-lem}

\begin{lemma}[Gautschi's Inequality]\label{lem:gautschi}
	Let $x$ be a positive real number and let $0 < s < 1$. Then,
	\begin{equation*}
		x^{1-s} \leq \frac{\Gamma(x+1)}{\Gamma(x+s)} \leq (x+1)^{1-s}\,.
	\end{equation*}
\end{lemma}

\begin{lemma}[Rearrangement Inequality]\label{lem:rearg-ineq} Let $a_1 \leq a_2 \leq \cdots \leq a_n$ and $b_1 \leq b_2 \leq \cdots \leq b_n$ be two sequence of real numbers. Then for any permutation $\sigma:[n] \rightarrow [n]$ the following holds:
	\begin{equation*}
	a_1 b_n + a_2 b_{n-1} + \cdots + a_n b_1 \leq a_1 b_{\sigma(1)} + a_2 b_{\sigma(2)} + \cdots + a_n b_{\sigma(n)} \leq a_1 b_1 + a_2 b_2 + \cdots a_nb_n\,.
	\end{equation*}
	
\end{lemma}

\begin{lemma}\label{lem:1gaps}
Let $\delta > 0$ be given and let $\Gamma$ be $\beta$-regular (See Definition \ref{def:beta-prior}).Then, the following holds:

\[
	\E_\Gamma \bb{\Ind \bp{\mu < 1-3\delta} \min \bp{1+ \frac{1}{1-2\delta-\mu}, T(1-\mu)}} \leq C_0 \begin{cases}
	5 + \log(1/\delta), \text{~~~if~} \beta = 1, \\
	C(\beta), \text{~~if~} \beta > 1, \\
	C(\beta) \min \bp{\sqrt{T}, 1/\delta}^{1-\beta} \text{~~if~} \beta < 1
	\end{cases}\,.
\]
In above, $C(\beta)$ is a constant that only depends on $\beta$.

\end{lemma}

\proof{Proof of Lemma \ref{lem:1gaps}.}
	Divide the interval $[0,1]$ into subintervals of size $\delta$. In particular, let $h = \lceil {1/\delta \rceil}$ and suppose intervals $I_1, I_2, \cdots, I_h$ are defined according to $I_1 = [1-\delta, 1]$, and for $j \geq 2, I_j = [1-j \delta, (j-1)\delta)$. Let $p_j = \P[\mu_i \in I_j]$, then the following holds:
	\begin{align*}
		p_1 &\leq C_0 \delta^\beta \\
		p_1 + p_2 &\leq C_0 (2\delta)^\beta \\
		\vdots & \vdots \\
		p_1 + p_2 + \cdots+ p_h &\leq C_0 (h\delta)^\beta\,.
	\end{align*}
	Now we can write
	\begin{equation*}
		\E_\Gamma \bb{\Ind \bp{\mu < 1-3\delta} \bp{1+\frac{1}{1-2\delta-\mu}}} \leq \sum_{j=4}^h p_j \bp{1+\frac{1}{(j-3)\delta}} \,.
	\end{equation*}
  	Let $a_j = 1 + 1/( (j-3) \delta)$ and note that the sequence $a_j$ is decreasing in $j$. Hence, applying Rearrangement Inequality (Lemma \ref{lem:rearg-ineq}) to sequences $(a_4-a_5, a_5 - a_6, \cdots, a_{h-1}-a_h, a_h)$ and $(p_4, p_4+p_5, \cdots, p_4+p_5+\cdots+p_h)$, together with inequalities above on probability terms $p_i$ implies
  	\begin{equation}\label{eqn:gen-beta}
  		\sum_{j=4}^h p_j a_j \leq C_0 (4 \delta)^{\beta}  a_4 + \sum_{j=5}^h C_0 \bp{j^\beta - (j-1)^\beta} \delta^\beta a_j \leq C_0 + C_0 \delta^{\beta-1} \bp{4^\beta + \sum_{j=5}^h \frac{j^\beta - (j-1)^\beta}{j-3}} \,.
  	\end{equation}
  	Note that for $\beta = 1$ the right hand-side of above inequality turns into
  	\begin{equation*}
  		C_0 + C_0 \bp{4 + \sum_{j=5}^h 1/(j-3)} \leq C_0 + C_0 \bp{4 + \log (h-1)} \leq C_0 \bp{5 + \log(1/\delta)}\,.
  	\end{equation*}
  	For $\beta > 1$, we can write $j^\beta - (j-1)^\beta \leq \beta j^{\beta-1}$ using the mean-value theorem on the function $f(x) = x^\beta$. Hence, by using $j^{\beta-1}/(j-3) \leq 4 j^{\beta-2}$ which holds for any $j \geq 5$, the right hand-side of Eq. \ref{eqn:gen-beta} is at most
  	\begin{equation*}
  		C_0 + C_0 \delta^{\beta-1} \bp{4^\beta + 4 \beta \sum_{j=5}^h j^{\beta-2}} \leq C_0 + C_0 \delta^{\beta-1} \bp{4^\beta + 4 \beta D(\beta) (1/\delta)^{\beta-1}}\,,
  	\end{equation*}
  	where we used the fact that $\sum_{j=1}^n j^{\beta-2}$ is bounded above by $D(\beta) n^{\beta-1}$. Hence, letting $C(\beta) = 1 + 4^\beta + 4 D(\beta) \beta$, implies that our quantity of interest is bounded from above by $C(\beta)$.
  	
  	For $\beta < 1$, we can still use the mean-value theorem for the function $f(x)=x^\beta$ (here $j^\beta - (j-1)^\beta \leq \beta (j-1)^{\beta-1}$), together with $(j-1)^{\beta-1}/(j-3) \leq 2 (j-1)^{\beta-2}$ which is true for $j \geq 5$. Hence, the bound on the right hand-side of Eq. \ref{eqn:gen-beta} turns into
  	\begin{equation*}
  	C_0 + C_0 \delta^{\beta-1} \bp{4^\beta + 2 \beta \sum_{j=5}^h (j-1)^{\beta-2}} \leq C_0 + C_0 \delta^{\beta-1} \bp{4^\beta + 2 D(\beta) \beta}\,, 	
  	\end{equation*}
  	where we used the fact that when $\beta <1$, the summation $\sum_{j=1}^{\infty} j^{\beta-2}$ is bounded above by a constant $D(\beta)$. Hence, letting $C'(\beta) = 1 + 4^\beta + 2 D(\beta) \beta$ implies half of the result for $\beta < 1$.
  	
  	The only remaining part is to show that if $\delta < 1/\sqrt{T}$ we can show a similar bound for $\beta<1$ where $1/\delta$ is replaced with $\sqrt{T}$. In this case, the idea is to use the second term $T(1-\mu)$ for $\mu$ terms that are closer to zero. In particular, we can write
  	\begin{equation*}
  		\E_\Gamma \bb{\Ind \bp{\mu < 1-3\delta} \bp{1+\frac{1}{1-2\delta-\mu}}} \leq \sum_{j=4}^l p_j (T j \delta) + \sum_{j=l+1}^h p_j \bp{1 + \frac{1}{(j-3) \delta}}\,,
  	\end{equation*}
  	where $l$ is chosen as the largest integer for which $T l (l-3) \delta^2 \geq 1$. Note that the first sum is upper bounded by
  	\begin{equation*}
  	T l \delta \bp{\sum_{j=1}^l p_j} \leq C_0 T l^{\beta+1} \delta^{\beta+1}.
  	\end{equation*}
  	For the second term, we can use a similar analysis to the one we used above. The only difference here is that the summation of terms $j^{\beta-2}$ is from $l+1$ to $h$ (or $\infty$) and hence is upper bounded by $D(\beta) l^{\beta-1}$.
  	\begin{equation*}
  		\sum_{j=l+1}^h p_j \bp{1 + \frac{1}{(j-3)\delta}} \leq C_0 + C_0 \delta^{\beta-1} \bp{ 4 (l+1)^{\beta-1} + 2 \beta D(\beta) (l+1)^{\beta-1}}\,.
  	\end{equation*}
  	Now using the fact that $l \delta \approx 1/\sqrt{T}$ implies the the quantity of interest is at most
  	\begin{equation*}
  		C_0 \bp{T T^{-(\beta+1/2)} + 1 + 4 T^{(1-\beta)/2} + 2 \beta D \beta T^{(1-\beta)/2}} \leq C_0 C'(\beta) T^{(1-\beta)/2}\,,
  	\end{equation*}
  	Letting $C(\beta) = \max \bp{C'(\beta), C''(\beta)}$ finishes the proof.\Halmos
  	\endproof

\begin{lemma}\label{lem:Deltas}
Suppose that $\bmu = (\mu_1, \mu_2, \ldots, \mu_k)$ and $\mu_i \sim \Gamma$ which satisfies the conditions given in Assumption \ref{ass:prior}. Let $\mu_{(k)} \leq \mu_{(k-1)} \leq \cdots \leq \mu_{(1)}$ denote the order statistics of $\bmu$ and define $\Delta_{(i)} = \mu_{(1)} - \mu_{(i)}$. Then, we have
\begin{enumerate}
    \item $U_i =\Delta_{(i)}$ has the following density function
    \begin{align*}
        g_{U_i}(u) = k(k-1) \binom{k-2}{i-2} \int_0^{1-u} & \bp{G(u+z)-G(z)}^{i-2} \\
        &g(u+z) g(z) G(z)^{k-i} dz\,,
    \end{align*}
    where $G(\cdot)$ is the cumulative distribution function of $\Gamma$, defined as $G(z) = P_{Z \sim \Gamma}\bb{Z \leq z} = \int_0^z g(x) dx$
    \item If $g(x) \leq D_0$ for all $x \in [0,1]$, then any $i \geq 3$,
    \begin{equation*}
        \E[1/\Delta_{(i)}] \leq \frac{D_0k}{i-2}\,.
    \end{equation*}
\end{enumerate}
\end{lemma}
\proof{Proof of Lemma \ref{lem:Deltas}.} The proof is as follows.
\begin{enumerate}
\item The first part follows from basic probability calculations. In fact, conditioned on $\mu_{(i)} = z$, the density of $U_i$ around $u$ can be computed according to the fact that we need $k-i$ of $\mu$s to be less than $z$ and $i-2$ of them to belong to $[z, u+z]$. Note that here $\mu_{(1)}$ is equal to $u+z$. Considering all the different permutations that lead to the same realization of the order statistics and integrating $z$ from $0$ to $1-u$ (possible values for $z$) yields the desired formula.

\item Note that for any $u > 0$ we have
\begin{equation*}
    G(u+z) - G(z) \leq D_0 u \rightarrow \frac{1}{u} \leq \frac{D_0}{G(u+z)-G(z)}\,.
\end{equation*}
Now using this inequality we can write
\begin{align*}
    \E[\frac{1}{\Delta^{(k)}_{(i)}}]
    &= \int_0^1 \frac{1}{u} \int_0^{1-u} k(k-1) \binom{k-2}{i-2}
    \bp{G(u+z)-G(z)}^{i-2} g(u+z) g(z) G(z)^{k-i} dz du \\
    &\leq D_0 k(k-1) \binom{k-2}{i-2} \int_0^1 \int_0^{1-u}
    \bp{G(u+z)-G(z)}^{i-3} g(u+z) g(z) G(z)^{k-i} dz du  \\
    &= D_0 \frac{k(k-1) \binom{k-2}{i-2}}{(k-1)(k-2) \binom{k-3}{i-3}} (k-1)(k-2) \binom{k-3}{i-3} \\
    &\int_0^1 \int_0^{1-u} \bp{G(u+z)-G(z)}^{i-3}
    g(u+z) g(z) G(z)^{k-i} dz du
    = \frac{D_0\,k}{i-2}\,.
\end{align*}
where we used the fact that the last integral is equivalent to $\int_{0}^1 g_V(v) dv = 1$ for $V = \Delta^{(k-1)}_{(i-1)}$ which is true for $i \geq 3$.
\end{enumerate}
\Halmos
\endproof

\section{Proofs of \S \ref{sec:generalizations}: $\beta$-Regular Priors and Sequential Greedy }\label{app:generalizations}

We prove results presented in \S \ref{sec:generalizations} here. This section is divided into two parts. In the first part, we prove results for general $\beta$-regular priors and in the second part, we prove results for the sequential greedy (Seq-Greedy) algorithm.

\subsection{General $\beta$-regular priors}\label{app:gen-prior-details}
Here, we prove results presented in Table \ref{tab:gen-beta}. This includes generalizing the result of Theorem \ref{thm:lb} on lower bounds, Theorems \ref{thm:upp-small-k} and \ref{thm:upp-large-k} for UCB and SS-UCB, and finally Theorems \ref{thm:greedy-opt-bern}, \ref{thm:greedy-reg}, and \ref{thm:greedy-reg-1sg} for Greedy and SS-Greedy. For each situation, we highlight the steps in the proofs that require significant changes and prove such modified results.

\subsubsection{Lower Bounds}

Note that the proof presented in \S \ref{app:prelim-lb} goes through without much change. Indeed, consider the large $k$ regime defined by $k>c_D(\beta) T^{\beta/(\beta+1)}$, where $c_D(\beta)$ only depends on $c_0, C_0$ and $\beta$. Our goal is to prove that there exists a class of ``bad orderings'' that happen with a constant probability, for which the Bayesian is $\Omega\bp{T^{\beta/(\beta+1)}}$. Let $m=\lfloor{T^{\beta/(\beta+1)} \rfloor}$ and define the following class over orderings of arms:

\begin{itemize}
	\item[(i)] $\max_{i=1}^k \mu_i \geq 1-\bp{c_0 k}^{-1/\beta}$,
	\item[(ii)] $\max_{i=1}^m \mu_i \leq 1-\bp{2C_0}^{-1/\beta} T^{-1/(\beta+1)},$
	\item[(iii)] $\sum_{i=1}^m \mu_i \leq \bp{1-\frac{1}{2^{\beta+2}C_0^\beta}} m$\,.
\end{itemize}

Now one can follow the proof steps of Theorem \ref{thm:lb} to prove the result. Indeed, it is not difficult to see that the upper bound on the probability of the events remain the same. For example, the first term happens with probability
\begin{align*}
	\P \bb{\max_{i=1}^k \mu_i \geq 1-\bp{c_0k}^{-1/\beta}} \geq 1 - \P \bb{\mu \leq 1-\bp{c_0k}^{-1/\beta}}^k \geq 1 - \bp{1-c_0 (c_0k)^{-1}}^k \geq 1-e^{-1}\,.
\end{align*}
Other events also hold with similar probabilities as in the proof Theorem \ref{thm:lb} and one can follow the exact line of proof. Similarly, two things can happen: either the decision-maker only pulls from the first $m$ arms during the $T$-period horizon or starts pulling some (or all) arms in the set $[k] - [m]$. In the former case, the regret is at least
\begin{equation*}
T \bb{\bp{2C_0}^{-1/\beta} T^{-1/(\beta+1)} - \bp{c_0 k}^{-1/\beta} } \geq c_L(\beta) T^{\beta/(\beta+1)}\,.
\end{equation*}
In the latter case, the regret is at least
\begin{equation*}
	\sum_{i=1}^m \bp{\max_{j=1}^k \mu_j - \mu_i} \geq m \bb{1-\bp{c_0 k}^{-1/\beta} -\bp{1-\frac{1}{2^{\beta+2}C_0^\beta}} } \geq c_L(\beta) T^{\beta/(\beta+1)}\,,
\end{equation*}
as desired. Note that the values of $c_D(\beta)$ and $c_L(\beta)$ can be calculated from the above inequalities.

\subsubsection{Upper Bounds}

We divide the discussion of upper bounds into two parts. We first focus on the case of UCB and then consider SS-UCB. Note that for UCB, under the assumption that the density $g$ is bounded above by $D_0$, the same proof goes through without any modification. However, an upper bound on the density is only possible when $\beta \geq 1$. In what follows we prove the bounds provided in Table \ref{tab:gen-beta} for small $k$ and large $k$, when $\beta < 1$.

\paragraph{UCB with small $k$.} We here provide a proof for $\beta < 1$ under the assumption that $g(x) \leq D_0 (1-x)^{\beta-1}$ for a constant $D_0$. Our goal is to prove that an analogous form of Theorem \ref{thm:upp-small-k} holds, which leads to $O\bp{k^{1/\beta}}$ regret. Recalling the proof of Theorem \ref{thm:upp-small-k} we have
\begin{equation*}
\BR_{T,k}(\text{UCB}) \leq k + \sum_{i=2}^k \E \bb{\min \bp{\frac{20+36\log f(T)}{\Delta_{(i)}}, T \Delta_{(i)} } }\,.
\end{equation*}
Now, we first need to derive a bound on $\E[1/\Delta_{(i)}]$ when $\beta < 1$. In particular, we want to see how the result of Lemma \ref{lem:Deltas} would change for $\beta < 1$ under the assumption that $g(x) \leq D_0 (1-x)^{\beta-1}$. Suppose that integer $q$ exists such that $q \leq 1/\beta < q+1$. We claim that there exists a constant $D(\beta)$ such that for any integer $i \geq q + 3$ we have
\begin{equation*}
\E[1/\Delta_{(i)}] \leq D_0(\beta) \bp{\frac{k}{i-q-2}}^{1/\beta}\,.
\end{equation*}
Note that part 1 of Lemma \ref{lem:Deltas} remains valid. However, for part 2, instead of $G(u+z)-G(z) \leq D_0 u$ we can write
\begin{equation*}
	G(u+z) - G(z) \leq D_0 \int_{1-u}^1 (1-x)^{\beta-1} \leq \frac{D_0}{\beta} u^\beta.
\end{equation*}
This means that for $D_0(\beta) = D_0/\beta$ we have $1/u \leq D_0(\beta) (G(u+z)-G(u))^{-1/\beta}$. Replacing this in the upper bound and following same ideas as the proof of Lemma \ref{lem:Deltas} implies
\begin{align*}
\E[\frac{1}{\Delta_{(i)}}]
&\leq D_0(\beta) k(k-1) \binom{k-2}{i-2} \int_0^1 \int_0^{1-u}
\bp{G(u+z)-G(z)}^{i-2-1/\beta} g(u+z) g(z) G(z)^{k-i} dz du  \\
&= D_0(\beta) \frac{k(k-1)\binom{k-2}{k-i}}{(k-1/\beta)(k-1-1/\beta) \binom{k-2-1/\beta}{k-i}} \\
&= D_0(\beta) \frac{\frac{\Gamma(k+1)}{\Gamma(i-1)}}{\frac{\Gamma(k-1/\beta+1)}{\Gamma(i-1-1/\beta)}}\,,
\end{align*}
where $\Gamma$ is the gamma function defined by. Note that Gautschi's Inequality (see Lemma \ref{lem:gautschi}) implies that if $q \leq 1/\beta < q+1$ for an integer $q$
\begin{align*}
\frac{\Gamma(k+1)}{\Gamma(k+1-1/\beta)}
&= \frac{\Gamma(k+1)}{\Gamma(k-q+1)} \cdot \frac{\Gamma(k-q+1)}{\Gamma(k-q+(1+q-1/\beta))} \\
&\leq k(k-1) \cdots (k-q+1) (k-q+1)^{1/\beta-q} \,.
\end{align*}
Similarly, using the other direction in the Gautschi's inequality one can show that
\begin{equation*}
\frac{\Gamma(i-1)}{\Gamma(i-1-1/\beta)} \leq (i-2) (i-3) \cdots (i-1-q)  (i-2-q)^{1/\beta - q}
\end{equation*}
 Therefore, for any index $i \geq q+3$ we have
 \begin{equation*}
 \E[\frac{1}{\Delta_{(i)}}] \leq \bp{\frac{k}{i-q-2}}^{1/\beta}\,.
 \end{equation*}
 Now moving back to the proof of Lemma \ref{thm:upp-small-k}, we can divide the arms into two sets, integer indices $i \leq q + 2$ and $i \geq q + 3$. Hence,
\begin{align*}
\BR_{T,k}(\text{UCB})
&\leq k + \sum_{i \leq q+2} \E \bb{\min \bp{\frac{20+36\log f(T)}{\Delta_{(i)}}, T \Delta_{(i)} } } \\
&+ D_0(\beta) (20 + 36 \log f(T)) k^{1/\beta} \sum_{i \geq q+3} \bp{\frac{1}{i-q-2}}^{1/\beta} \,.
\end{align*}
As $\beta <1$, the term $\sum_{i \geq 1} (1/i)^{1/\beta}$ is bounded from above by a constant. Hence, it only remains to bound the first term. For the first term we use the same technique as of proof of Theorem \ref{thm:upp-small-k}. In particular, we first want to establish an upper bound on the density. Note that for $i=2$, using Lemma \ref{lem:Deltas}, the density of $\Delta_{(2)} = g_{U_2}$ is upper bounded by
\begin{align*}
	g_{U_2}(u) &= k(k-1) \int_{0}^{1-u} g(u+z) g(z) G(z)^{k-2} dz \leq D_0 u^{\beta-1} \int_0^{1-u} g(u+z) G(z)^{k-2} dz \\
	&\leq D_0 u^{\beta-1} \int_0^{1-u} g(u+z) G(u+z)^{k-2} \leq D_0 k u^{\beta-1}\,.
\end{align*}
For $i \geq 3$, we can write
\begin{align*}
g_{U_i}(u)
&= k(k-1) \binom{k-2}{i-2} \int_0^{1-u} \bp{G(u+z) - G(z)}^{i-2} g(u+z) g(z) G(z)^{k-i} dz\\
&\leq k(k-1) \binom{k-2}{i-2} \frac{D_0}{\beta} u^\beta \int_0^{1-u} \bp{G(u+z)-G(z)}^{i-3} g(u+z) g(z) G(z)^{k-i} dz \\
& = D_0(\beta) u^\beta \frac{k(k-1) \binom{k-2}{i-2}}{(k-1)(k-2)\binom{k-3}{i-3}} \\
& \leq D_0(\beta) \frac{k}{i-2} \\
& \leq D_0(\beta) k \,.
\end{align*}
Having these upper bounds on the densities $g_{U_i}$ it is not difficult to see that the proof of Theorem \ref{thm:upp-small-k} for indices $3 \leq i \leq q+2$ goes through similarly and the contribution of each term is bounded above by $D_0(\beta) k (10 + 18 \log f(T)) \log T$. For $i=2$ the situation is a bit different, but we still can write
\begin{align*}
\E \bb{\min \bp{\frac{20+36\log f(T)}{\Delta_{(i)}}, T \Delta_{(i)} } }
&\leq  D_0 k \bb{\int_0^{1/\sqrt{T}} Tz z^{\beta-1}  dz+ (20+36 \log f(T)) \int_{1/\sqrt{T}}^1 \frac{1}{z} z^{\beta-1} dz} \\
 &\leq \frac{D_0 k T^{(1-\beta)/2} (40 + 72 \log f(T))}{1-\beta} \\
 &=\tilde{O} \bp{k T^{(1-\beta)/2}}\,.
\end{align*}
Putting all the terms together and factoring all constants together as $D_{\text{all}}(\beta)$ we have
\begin{equation*}
	\BR_{T,k}(\text{UCB}) \leq D_{\text{all}}(\beta) \bp{k + \log f(T) k^{1/\beta} + k \log f(T) \log T + k T^{(1-\beta)/2}}\,.
\end{equation*}
Note that for small $k$, the term $k^{1/\beta}$ is dominant in above bound. However, once $k \geq \sqrt{T}$, the term $k T^{(1-\beta)/2}$ is the dominant term.

\paragraph{UCB with large $k$.} From the above regret bound it is clear that for large values of $k$ (specifically, $k \geq \sqrt{T}$) the term $k T^{(1-\beta)}$ is dominant which finishes the proof.

\paragraph{SS-UCB.} For SS-UCB, the results for small $k$ (in both cases $\beta \geq 1$ and $\beta < 1$) are under similar assumption as UCB, i.e., having a bounded density for $\beta \geq 1$ and being bounded from above by $D_0 (1-x)^{\beta-1}$, for $\beta < 1$. Hence, we only focus on proving that if the subsampling rate is chosen correctly, for large values of $k$, $\BR_{T,k}(\text{SS-UCB})$ is of order $T^{\beta/(\beta+1)}$ for $\beta \geq 1$ and of order $\sqrt{T}$ for $\beta < 1$. These results can be proved without putting an assumption on density, but rather under (more relaxed) $\beta$-regular priors. The proof steps are similar to those of Theorem \ref{thm:upp-large-k}. We briefly explain which parts need to be changed below.
\begin{itemize}
	\item $\beta \geq 1$: In this case, we wish to prove that if $m = T^{\beta/(\beta+1)}$ then Bayesian regret of SS-UCB is upper bounded by $\tilde{O}(T^{\beta/(\beta+1)})$. Similar to Theorem \ref{thm:upp-large-k}, we can upper bound $\BR_{T,k}(\text{SS-UCB})$ as sum of $\BR_{T,m}(\text{SS-UCB})$ and $T(1-\max_{l=1}^m \mu_{i_l})$. Pick $\theta = \bb{\log T/(mc_0)}^{1/\beta}$ which implies again that
	\begin{equation*}
		T(1-\max_{l=1}^m \mu_{i_l}) \leq 1 + T \theta = \tilde{O} \bp{T^{\beta/(\beta+1)}}\,.
	\end{equation*}
	For the first term, we can follow the rest of the proof. In the last line, however, we need to use
\[
P(\mu \geq 1 - 3\theta) \leq 3^\beta C_0 \theta^\beta
\]
and the result of Lemma \ref{lem:1gaps} for a general $\beta$. Note that the result  follows from the result of Lemma \ref{lem:1gaps} for general $\beta \geq 1$ instead of $\beta = 1$. In particular, Lemma \ref{lem:1gaps} shows that the difference between the upper bound on the term
	\begin{equation*}
	\E \bb{\Ind(\mu < 1- 3\theta) \bp{1+\frac{1}{1-2\theta-\mu}}}
	\end{equation*}
	for the cases $\beta = 1$ and $\beta > 1$ is only in $\log(1/\theta)$, which translates to a factor $\log T$. This finishes the proof.
	\item $\beta < 1$: Proof is similar to the proof above (case $\beta \geq 1$). Indeed, let $m = T^{\beta/2}$ and follow the steps provided in Theorem \ref{thm:upp-large-k}. Following these steps, we can see that $\theta = (\log T/mc_0)^{1/\beta}  = \tilde{O}(T^{-1/2})$ and hence, $T(1-\max_{l=1}^m \mu_{i_l}) = \tilde{O}(\sqrt{T})$. Our goal is to show that the other term, coming from $\BR_{T,m}(\text{SS-UCB})$ is also $\tilde{O}(\sqrt{T})$. Note that again the last two lines of the proof of Theorem \ref{thm:upp-large-k} need to be modified. Indeed, replacing $P(\mu \geq 1 - 3\theta) \leq 3^\beta C_0 \theta^\beta$ and noting that $\theta = \tilde{O}(T^{-1/2})$ shows that the term $3 \theta T m \P(\mu \geq 1 - 3\theta)$ is $\tilde{O} (\sqrt{T})$. Finally, for the last term we can use Lemma \ref{lem:1gaps}. Here, note that as $1/\theta = T^{-\beta/2} > T^{-1/2}$, this last term is upper bounded by
	\begin{equation*}
	C_0  C(\beta) (20 + 36 \log f(T)) m T^{(1-\beta)/2} = \tilde{O}(\sqrt{T})\,,
	\end{equation*}
	as desired.
\end{itemize}

\subsubsection{Results for Greedy and SS-Greedy}
	Note that the proof of results for Greedy all follow from the generic bound presented in Lemma \ref{lem:greedy-gen}. Note that as discussed in \S \ref{sec:generalizations}, $\betaF$ is lower bound for the exponent of $\delta$ in $\E_\Gamma \bb{ \Ind \bp{ \mu \geq 1-\delta} q_{1-2\delta}(\mu)}$. In particular, Lemmas \ref{lem:bern-crossing}, \ref{lem:max-crossing}, and \ref{lem:subg-crossing} show that:
	\begin{itemize}
	\item For Bernoulli rewards: $\inf_{\mu \geq 1-\delta} q_{1-2\delta}(\mu) = \Omega(1)$ and hence for $\beta$-regular priors we have
	\begin{equation*}
	\E_\Gamma \bb{ \Ind \bp{ \mu \geq 1-\delta} q_{1-2\delta}(\mu)} = \tilde{\Omega}(\delta^\beta) \,.
	\end{equation*}
	Hence, $\betaF = \beta$ satisfies \eqref{eq:gen-prior-delta-scaling}.
	\item For uniformly upward-looking rewards: $\inf_{\mu \geq 1-\delta} q_{1-2\delta}(\mu) = \tilde{\Omega}(\delta)$ and hence for $\beta$-regular priors we have
	\begin{equation*}
	\E_\Gamma \bb{ \Ind \bp{ \mu \geq 1-\delta} q_{1-2\delta}(\mu)} = \tilde{\Omega}(\delta^{\beta+1}) \,.
	\end{equation*}
	Hence, $\betaF  = \beta+1$ satisfies \eqref{eq:gen-prior-delta-scaling}.
	\item For general $1$-subgaussian rewards: $\inf_{\mu \geq 1-\delta} q_{1-2\delta}(\mu) = \Omega(\delta^2)$ and hence for $\beta$-regular priors we have
	\begin{equation*}
	\E_\Gamma \bb{ \Ind \bp{ \mu \geq 1-\delta} q_{1-2\delta}(\mu)} = \tilde{\Omega}(\delta^{\beta+2}) \,.
	\end{equation*}
	Hence, $\betaF = \beta+2$ satisfies \eqref{eq:gen-prior-delta-scaling}.
	\end{itemize}
	Having these bounds on the first term and using inequality $(1-x) \leq \exp(-x)$, the result of Lemma \ref{lem:greedy-gen} can be written as:
	\begin{align}\label{eqn:gen-beta-greedy}
	\BR_{T,k}(\text{Greedy})
	&\leq  T \exp \bp{-k c_{\mathcal{F}}(\beta) \delta^{\betaF}}+3T \delta \nonumber \\
	&+ k \E_\Gamma \bb{\Ind \bp{\mu < 1-3\delta} \min \bp{1+ \frac{3}{C_1(1-2\delta-\mu)}, T(1-\mu)}}\,,
	\end{align}
	where $c_{\mathcal{F}}(\beta)$ is a constant depending on the distribution $\mathcal{F}$ and also $\beta$. Furthermore, Lemma \ref{lem:greedy-gen} also implies that for SS-Greedy the same upper bound holds, with $k$ being replaced with $m$ in Eq. \ref{eqn:gen-beta-greedy}. From here, the results can be proved by choosing the optimal value for $\delta$. We explain these choices for $\beta \geq 1$ and $\beta < 1$.
 	
 	\paragraph{Case $\beta \geq 1$.} In this case, the best choice of $\delta$ is given by
 	\begin{equation*}
 		\delta =  \bp{\frac{\log T}{k c_{\mathcal{F}}(\beta)}}^{1/\betaF}\,.
 	\end{equation*}
Since $\beta \geq 1$, Lemma \ref{lem:1gaps} implies that the last term is also upper bounded by $k C_0 C(\beta)\log(1/\delta) = \tilde{O}(k)$. Therefore, the bound in Eq. \ref{eqn:gen-beta} turns into
 	\begin{align*}
 		\BR_{T,k}(\text{Greedy})
 		&\leq 1 + 3 T \bp{\frac{\log T}{k c_{\mathcal{F}}(\beta)}}^{1/\betaF}
 		+ k  \E_\Gamma \bb{\Ind \bp{\mu < 1-3\delta} \min \bp{1+ \frac{3}{C_1(1-2\delta-\mu)}, T(1-\mu)}} \\
 		&= 1 + \tilde{O}(Tk^{-1/\beta_F}) + \tilde{O}(k) \,.
 	\end{align*}
 	Now note that for small $k$ the first term is dominant, while for the large $k$ the second term is dominant. This proves the result presented in Table \ref{tab:gen-beta} for Greedy, for $\beta \geq 1$. For SS-Greedy, note that the same result applies, with $k$ being replaced with $m$. Note that for $\beta \geq 1$, the optimal $m$ is given by $m = \Theta(T^{\beta_F/(\beta_F+1)})$, which proves the result for the SS-Greedy, when $k$ is large and $\beta \geq 1$.
 	
 	 \paragraph{Case $\beta < 1$.} In this case, the regret bound of Greedy has three regimes. Again pick $\delta$ according to
 	\begin{equation*}
 	\delta =  \bp{\frac{\log T}{k c_{\mathcal{F}}(\beta)}}^{1/\betaF}\,.
 	\end{equation*}
 	As $\beta < 1$, Lemma \ref{lem:1gaps} implies that the last term in Eq. \eqref{eqn:gen-beta-greedy} is also upper bounded by $k C_0 C(\beta) \min \bp{\sqrt{T}, 1/\delta}^{1-\beta}$. Therefore, the bound in Eq. \eqref{eqn:gen-beta} turns into
 	\begin{align*}
 	\BR_{T,k}(\text{Greedy})
 	&\leq 1 + 3 T \bp{\frac{\log T}{k c_{\mathcal{F}}(\beta)}}^{1/\betaF}
 	+ k  \E_\Gamma \bb{\Ind \bp{\mu < 1-3\delta} \min \bp{1+ \frac{3}{C_1(1-2\delta-\mu)}, T(1-\mu)}} \\
 	&= 1 + \tilde{O}(Tk^{-1/\beta_F}) + \tilde{O}\bb{\min \bp{k T^{(1-\beta)/2},  k^{(1-\beta+\beta_F)/\beta_F}}} \,.
 	\end{align*}
 	Now note that for small $k$ the first term is dominant and regret is $\tilde{O} (Tk^{-1/\beta_\mathcal{F}})$ while for the large $k$, the second term is dominant. Note that between these two terms inside the minimum, for mid-range values for $k$, the regret is $k^{(1-\beta+\betaF)/\betaF}$, while for larger values of $k$ the term $k T^{(1-\beta)/2}$ is smaller.
 	
 	For SS-Greedy, note that as the regret in large $k$ is minimum of two terms, we need to carefully check at what $m$ the best regret is achieved. However, as $\beta_F \geq \beta$, it is not difficult to see that the best regret is achieved when $T m^{-1/\betaF} = m^{(\beta_\mathcal{F}-\beta+1)/\beta_F}$, or
 	\begin{equation*}
 		m = \Theta \bp{T^{\beta_F/(\beta_F-\beta+2)}}\,,
 	\end{equation*}
 	which proves $\BR_{T,k}(\text{SS-Greedy}) = \tilde{O} \bp{T^{(\beta_F-\beta+1)/(\beta_F-\beta+2)}}$, as desired.
 \subsection{Sequential Greedy}

\proof{Proof of Lemma \ref{lem:seq-greedy-gen}}
The proof is very similar to the proof of Lemma \ref{lem:greedy-gen}. In particular, the goal is to prove that, if $k_1$ is selected large enough (so that the probability of undesired events are much smaller than $1$), then with a very high probability the algorithm only tries the first $k_1$ and hence the term $k$ in the regret will be replaced with $k_1$. To formally prove this, we show that on most problem instances only first $k_1$ arms are tried and in those cases we can replace $k$ in the bound of Lemma \ref{lem:greedy-gen} with $k_1$, and the contribution of problems where more than $k_1$ arms are pulled is at most $T$ times the probability of such events. While, the proof steps are very similar to that of Lemma \ref{lem:greedy-gen}, we provide a proof for completeness below.

Fix a realization $\bmu = (\mu_1, \mu_2, \cdots, \mu_k)$  and partition the interval $(0,1)$ into sub-intervals of size $\delta$. In particular, let $I_1, I_2, \ldots, I_h$ be as $I_1 = [1-\delta, 1]$ and $I_j = [1-j\delta, 1-(j-1)\delta)$, for $j \geq 2$, where $h = \lceil{1/\delta \rceil}$. For any $1 \leq j \leq h$, let $A_j$ denote the set of arms with means belonging to $I_j$. Suppose that for each arm $i \in [k]$, a sequence of i.i.d. samples from $P_{\mu_i}$ is generated. Furthermore, for each $1 \leq j \leq h$ define the set $B_j = A_j \cap [k_1]$, which includes the set of all arms among the first $k_1$ indices that belong to $A_j$. Then, using definition of function $q_{1-2\delta}$ in Eq. \eqref{eqn:max-cross}, for any arm $i \in B_1$ we have
\begin{equation*}
\P \bb{\exists t \geq 1: \hmu_{i}(t) \leq 1-2\delta} = 1 - q_{1-2\delta}(\mu_i)\,,
\end{equation*}
where $\hmu_{i}(t)$ is the empirical average of rewards of arm $i$ at time $t$. Now, consider all arms that belong to $B_1$. Define the bad event as $\bar{G} = \cap_{i \in B_1} \bc{\exists t \geq 1: \hmu_{i}(t) \leq 1 - 2\delta}$ meaning that the sample average of all arms in $B_1$ drops below $1-2\delta$ at some time $t$. Similarly, define the good event $G$, as the complement of $\bar{G}$. We can write
\begin{equation}\label{eqn:barg}
\P \bb{\bar{G}} \leq \prod_{i \in B_1} (1-q_{1-2\delta}(\mu_i))\,.
\end{equation}
The above bounds the probability that $\bar{G}$, our bad event happens. Now note that if the (good) event $G$ happens, meaning that there exists an arm in $B_1$ such that its empirical average never drops below $1-2\delta$ we can bound the regret as follows. First note that in the case of good event, the algorithm only pulls the first $k_1$ arms and hence for all $i > k_1, N_i(T) = 0$. Therefore,
\begin{equation*}
R_T(\text{Seq-Greedy} \mid \bmu) \leq \sum_{i=1}^{k_1} (1-\mu_i) \E[N_i(T)]\,,
\end{equation*}
Our goal is to bound $\E[N_i(T)]$ for any arm that belongs to $\cup_{j=4}^h B_j$. Indeed, a standard argument based on subgaussianity of arm $i$ implies that
\begin{align*}
\P{N_i(T) \geq t + 1 \mid \bmu}
\leq
\P{\hmu_i(t) \geq 1- 2\delta \mid \bmu}
\leq
\exp \bp{-t(1-2\delta-\mu_i)^2/2},
\end{align*}
where we used the fact that a suboptimal arm $i$ (with mean less than $1-3\delta$) only would be pulled for the $t+1$ time if its estimate after $t$ samples is larger than $1-2\delta$ and that (centered version of) $\hmu_i(t)$ is $1/t$-subgaussian. Now note that for any discrete random variable $Z$ that only takes positive values, we have $\E[Z] = \sum_{t=1}^{\infty} \P(Z \geq t)$. Hence,
\begin{align*}
\E[N_i(T) \mid \bmu]
= 1 + \sum_{t \geq 1} \P(N_i(T) \geq t+1 \mid \bmu)
&\leq 1 + \sum_{t=1}^{\infty} \exp\bp{\frac{-t(1-2\delta-\mu_i)^2}{2}} \\
&= 1 + \frac{1}{1-\exp \bp{\frac{(1-2\delta-\mu_i)^2}{2}}}\,.
\end{align*}
Now note that for any $z \in [0,1]$ we have $\exp(-z) \leq 1 - 2C_1z$ where $C_1=(1-\exp(-1))/2$. Hence, we have
\begin{equation*}
\E[N_i(T) \mid \bmu] \leq 1 + \frac{1}{C_1(1-2\delta-\mu_i)^2}\,,
\end{equation*}
which implies that
\begin{align*}
(1-\mu_i) \E[N_i(T) \mid \bmu]
= (1-\mu_i) + \frac{1-\mu_i}{C_1(1-2\delta-\mu_i)^2}
\leq 1 + \frac{3}{C_1(1-2\delta-\mu_i)}\,,
\end{align*}
where we used the inequality $1-\mu_i \leq 3(1-2\delta-\mu_i)$ which is true as $\mu_i \leq 1-3\delta$. Note that the above is valid for any arm that belongs to $\cup_{j=4}^h B_j$. Furthermore, the expected number of pulls of arm $i$ cannot exceed $T$. Hence,
\begin{equation*}
(1-\mu_i) \E[N_i(T) \mid \bmu]
\leq \min \bp{1 + \frac{3}{C_1(1-2\delta-\mu_i)}, T(1-\mu_i)} \,.
\end{equation*}
As a result,
\begin{equation*}
R_T \bp{\text{Seq-Greedy} \mid \bmu} \leq T \P \bb{\bar{G}}  + 3T \delta + \sum_{j=4}^h \sum_{i \in B_j}  \min \bp{1 + \frac{3}{C_1(1-2\delta-\mu_i)}, T(1-\mu_i)}\,.
\end{equation*}
We can replace $\P \bb{\bar{G}}$ from Eq. \eqref{eqn:barg} and take expectation with respect to the prior. Note that the first term we can also be rewritten as $\prod_{i=1}^{k_1} \bb{1-\Ind(\mu \geq 1-\delta) q_{1-2\delta}(\mu_i)}$. Hence, taking an expectation with respect to $\mu_1, \mu_2, \cdots, \mu_k \sim \Gamma$ implies
\begin{align*}
\BR_{T,k} \bp{\text{Seq-Greedy}}
&\leq  T \bp{1 - \E_\Gamma \bb{\Ind \bp{\mu \geq 1-\delta} q_{1-2\delta}(\mu))}}^{k_1}
+3T \delta \\
&+ k_1 \E_\Gamma \bb{\Ind \bp{\mu < 1-3\delta} \min \bp{1+ \frac{3}{C_1(1-2\delta-\mu)}, T(1-\mu)}}\,,
\end{align*}
as desired. \Halmos
\endproof

\proof{Proof of Theorem \ref{thm:seq-greedy}}
	We prove the result only for the Bernoulli case, the other cases use similar arguments. We want to show that when $k < \sqrt{T}$, Bayesian regret of Seq-Greedy is at most $\tilde{O}(T/k)$ and if $k > \sqrt{T}$ it is at most $\tilde{O}(\sqrt{T})$. Consider the first case and note that in this case $\delta = \tilde{\Theta} \bp{k^{-1}}$. We claim that for this $\delta$ by selecting $k_1 = k$ in Lemma \ref{lem:seq-greedy-gen}, we have
	\begin{align*}
	\BR_{T,k}(\text{Seq-Greedy})
	&\leq T \bp{1 - \E_\Gamma \bb{\Ind \bp{\mu \geq 1-\delta} q_{1-2\delta}(\mu))}}^k
	+3T \delta \\
	&+ k\E_\Gamma \bb{\Ind \bp{\mu < 1-3\delta} \min \bp{1+ \frac{3}{C_1(1-2\delta-\mu)}}}\,,
	\end{align*}
	and the conclusion follows similar to the proof of Theorem \ref{thm:greedy-opt-bern}. Indeed, one can show that if $\delta = 5 \log T/(kc_0)$, then the above rate becomes $\tilde{O}(T/k)$.
	Now consider the other case that $k > \sqrt{T}$, meaning that $\delta = \tilde{\Theta} \bp{T^{-1/2}}$. In this case, by selecting $k_1 = \lfloor{ \sqrt{T} \rfloor} < k$, we have
	\begin{align*}
		\BR_{T,k}(\text{Seq-Greedy})
		&\leq T \bp{1 - \E_\Gamma \bb{\Ind \bp{\mu \geq 1-\delta} q_{1-2\delta}(\mu))}}^{k_1}
		+3T \delta \\
		&+ k_1 \E_\Gamma \bb{\Ind \bp{\mu < 1-3\delta} \min \bp{1+ \frac{3}{C_1(1-2\delta-\mu)}}}\,.
	\end{align*}
	which for the choice $\delta = \Theta \bp{T^{-1/2}}$ translates to the upper bound on $\BR_{T, k_1}(\text{Greedy}) \leq \tilde{O} (\sqrt{T})$, according to Theorem \ref{thm:greedy-opt-bern}. The proof for $1$-subgaussian and uniformly upward-looking rewards is similar. \Halmos
\endproof 

%% file: paper-main.bbl
\begin{thebibliography}{35}
\expandafter\ifx\csname natexlab\endcsname\relax\def\natexlab#1{#1}\fi
\expandafter\ifx\csname url\endcsname\relax
  \def\url#1{{\tt #1}}\fi
\expandafter\ifx\csname urlprefix\endcsname\relax\def\urlprefix{URL }\fi
\expandafter\ifx\csname urlstyle\endcsname\relax
  \expandafter\ifx\csname doi\endcsname\relax
  \def\doi#1{doi:\discretionary{}{}{}#1}\fi \else
  \expandafter\ifx\csname doi\endcsname\relax
  \def\doi{doi:\discretionary{}{}{}\begingroup \urlstyle{rm}\Url}\fi \fi

\bibitem[{Abbasi-Yadkori et~al.(2011)Abbasi-Yadkori, P{\'a}l, and
  Szepesv{\'a}ri}]{abbasi2011improved}
Abbasi-Yadkori, Yasin, D{\'a}vid P{\'a}l, Csaba Szepesv{\'a}ri. 2011.
\newblock Improved algorithms for linear stochastic bandits.
\newblock {\it Advances in Neural Information Processing Systems\/}.
  2312--2320.

\bibitem[{Agrawal and Goyal(2012)}]{agrawal2012analysis}
Agrawal, Shipra, Navin Goyal. 2012.
\newblock Analysis of thompson sampling for the multi-armed bandit problem.
\newblock {\it Conference on learning theory\/}. 39--1.

\bibitem[{Asmussen and Albrecher(2010)}]{asmussen2010ruin}
Asmussen, S{\o}ren, Hansj{\"o}rg Albrecher. 2010.
\newblock {\it Ruin probabilities\/}, vol.~14.
\newblock World scientific Singapore.

\bibitem[{Audibert and Bubeck(2009)}]{audibert2009minimax}
Audibert, Jean-Yves, S{\'e}bastien Bubeck. 2009.
\newblock Minimax policies for adversarial and stochastic bandits.

\bibitem[{Audibert et~al.(2007)Audibert, Munos, and
  Szepesv{\'a}ri}]{audibert2007tuning}
Audibert, Jean-Yves, R{\'e}mi Munos, Csaba Szepesv{\'a}ri. 2007.
\newblock Tuning bandit algorithms in stochastic environments.
\newblock {\it International conference on algorithmic learning theory\/}.
  Springer, 150--165.

\bibitem[{Auer et~al.(2002)Auer, Cesa-Bianchi, and Fischer}]{auer2002finite}
Auer, Peter, Nicolo Cesa-Bianchi, Paul Fischer. 2002.
\newblock Finite-time analysis of the multiarmed bandit problem.
\newblock {\it Machine learning\/} {\bf 47}(2-3) 235--256.

\bibitem[{Bastani et~al.(2020)Bastani, Bayati, and
  Khosravi}]{bastani2017mostly}
Bastani, Hamsa, Mohsen Bayati, Khashayar Khosravi. 2020.
\newblock Mostly exploration-free algorithms for contextual bandits.
\newblock {\it Management Science\/} .

\bibitem[{Berry et~al.(1997)Berry, Chen, Zame, Heath, and
  Shepp}]{berry1997bandit}
Berry, Donald~A, Robert~W Chen, Alan Zame, David~C Heath, Larry~A Shepp. 1997.
\newblock Bandit problems with infinitely many arms.
\newblock {\it The Annals of Statistics\/}  2103--2116.

\bibitem[{Bonald and Proutiere(2013)}]{bonald2013two}
Bonald, Thomas, Alexandre Proutiere. 2013.
\newblock Two-target algorithms for infinite-armed bandits with bernoulli
  rewards.
\newblock {\it Advances in Neural Information Processing Systems\/}.
  2184--2192.

\bibitem[{Carpentier and Valko(2015)}]{carpentier2015simple}
Carpentier, Alexandra, Michal Valko. 2015.
\newblock Simple regret for infinitely many armed bandits.
\newblock {\it International Conference on Machine Learning\/}. 1133--1141.

\bibitem[{Chan and Hu(2019)}]{chan2019optimal}
Chan, Hock~Peng, Shouri Hu. 2019.
\newblock Optimal ucb adjustments for large arm sizes.
\newblock {\it arXiv preprint arXiv:1909.02229\/} .

\bibitem[{Chaudhuri and Kalyanakrishnan(2018)}]{chaudhuri2018quantile}
Chaudhuri, Arghya~Roy, Shivaram Kalyanakrishnan. 2018.
\newblock Quantile-regret minimisation in infinitely many-armed bandits.
\newblock {\it UAI\/}. 425--434.

\bibitem[{Cooprider and Nassiri(2023)}]{cooprider2023science}
Cooprider, Joe, Shima Nassiri. 2023.
\newblock Science of price experimentation at amazon.
\newblock {\it AEA 2023, NABE 2023\/}.
\newblock
  \urlprefix\url{https://www.amazon.science/publications/science-of-price-experimentation-at-amazon}.

\bibitem[{de~Bruijn(1981)}]{de1981asymptotic}
de~Bruijn, N.G. 1981.
\newblock {\it Asymptotic Methods in Analysis\/}.
\newblock Bibliotheca mathematica, Dover Publications.
\newblock \urlprefix\url{https://books.google.com/books?id=Oqj9AgAAQBAJ}.

\bibitem[{Frey and Slate(1991)}]{frey1991letter}
Frey, Peter~W, David~J Slate. 1991.
\newblock Letter recognition using holland-style adaptive classifiers.
\newblock {\it Machine learning\/} {\bf 6}(2) 161--182.

\bibitem[{Gittins(1979)}]{gittins1979bandit}
Gittins, John~C. 1979.
\newblock Bandit processes and dynamic allocation indices.
\newblock {\it Journal of the Royal Statistical Society: Series B
  (Methodological)\/} {\bf 41}(2) 148--164.

\bibitem[{Hao et~al.(2020)Hao, Lattimore, and Szepesvari}]{hao2019adaptive}
Hao, Botao, Tor Lattimore, Csaba Szepesvari. 2020.
\newblock Adaptive exploration in linear contextual bandit.
\newblock {\it International Conference on Artificial Intelligence and
  Statistics\/}. PMLR, 3536--3545.

\bibitem[{{Jedor} et~al.(2021){Jedor}, {Lou{\"e}dec}, and
  {Perchet}}]{Jedor2021greedy}
{Jedor}, Matthieu, Jonathan {Lou{\"e}dec}, Vianney {Perchet}. 2021.
\newblock {Be Greedy in Multi-Armed Bandits}.
\newblock {\it arXiv e-prints\/}  arXiv:2101.01086.

\bibitem[{Kannan et~al.(2018)Kannan, Morgenstern, Roth, Waggoner, and
  Wu}]{kannan2018smoothed}
Kannan, Sampath, Jamie~H Morgenstern, Aaron Roth, Bo~Waggoner, Zhiwei~Steven
  Wu. 2018.
\newblock A smoothed analysis of the greedy algorithm for the linear contextual
  bandit problem.
\newblock {\it Advances in Neural Information Processing Systems\/}.
  2227--2236.

\bibitem[{Katz-Samuels and Jamieson(2020)}]{katzsamuels2020true}
Katz-Samuels, Julian, Kevin Jamieson. 2020.
\newblock The true sample complexity of identifying good arms.
\newblock Silvia Chiappa, Roberto Calandra, eds., {\it Proceedings of the
  Twenty Third International Conference on Artificial Intelligence and
  Statistics\/}, {\it Proceedings of Machine Learning Research\/}, vol. 108.
  PMLR, 1781--1791.

\bibitem[{Kaufmann(2018)}]{kaufmann2018bayesian}
Kaufmann, Emilie. 2018.
\newblock On bayesian index policies for sequential resource allocation.
\newblock {\it The Annals of Statistics\/} {\bf 46}(2) 842--865.

\bibitem[{Kohavi et~al.(2020)Kohavi, Tang, and Xu}]{kohavi2020trustworthy}
Kohavi, R., D.~Tang, Y.~Xu. 2020.
\newblock {\it Trustworthy Online Controlled Experiments: A Practical Guide to
  A/B Testing\/}.
\newblock Cambridge University Press.
\newblock \urlprefix\url{https://books.google.com/books?id=Gu-CEAAAQBAJ}.

\bibitem[{Lai and Robbins(1985)}]{lai1985asymptotically}
Lai, Tze~Leung, Herbert Robbins. 1985.
\newblock Asymptotically efficient adaptive allocation rules.
\newblock {\it Advances in applied mathematics\/} {\bf 6}(1) 4--22.

\bibitem[{Lai et~al.(1987)}]{lai1987adaptive}
Lai, Tze~Leung, et~al. 1987.
\newblock Adaptive treatment allocation and the multi-armed bandit problem.
\newblock {\it The Annals of Statistics\/} {\bf 15}(3) 1091--1114.

\bibitem[{Lattimore and Szepesv{\'a}ri(2020)}]{lattimore2018bandit}
Lattimore, Tor, Csaba Szepesv{\'a}ri. 2020.
\newblock {\it Bandit algorithms\/}.
\newblock Cambridge University Press.

\bibitem[{Raghavan et~al.(2018)Raghavan, Slivkins, Vaughan, and
  Wu}]{raghavan2018externalities}
Raghavan, Manish, Aleksandrs Slivkins, Jennifer~Wortman Vaughan, Zhiwei~Steven
  Wu. 2018.
\newblock The externalities of exploration and how data diversity helps
  exploitation.
\newblock {\it arXiv preprint arXiv:1806.00543\/} .

\bibitem[{Ren et~al.(2019)Ren, Liu, and Shroff}]{ren2019exploring}
Ren, Wenbo, Jia Liu, Ness~B. Shroff. 2019.
\newblock Exploring $k$ out of top $\rho$ fraction of arms in stochastic
  bandits.
\newblock Kamalika Chaudhuri, Masashi Sugiyama, eds., {\it Proceedings of the
  Twenty-Second International Conference on Artificial Intelligence and
  Statistics\/}, {\it Proceedings of Machine Learning Research\/}, vol.~89.
  PMLR, 2820--2828.

\bibitem[{Russo and Van~Roy(2014{\natexlab{a}})}]{russo2014inf}
Russo, Daniel, Benjamin Van~Roy. 2014{\natexlab{a}}.
\newblock Learning to optimize via information-directed sampling.
\newblock {\it Advances in Neural Information Processing Systems\/}.
  1583--1591.

\bibitem[{Russo and Van~Roy(2014{\natexlab{b}})}]{russo2014learning}
Russo, Daniel, Benjamin Van~Roy. 2014{\natexlab{b}}.
\newblock Learning to optimize via posterior sampling.
\newblock {\it Mathematics of Operations Research\/} {\bf 39}(4) 1221--1243.

\bibitem[{Russo and Van~Roy(2016)}]{russo2016information}
Russo, Daniel, Benjamin Van~Roy. 2016.
\newblock An information-theoretic analysis of thompson sampling.
\newblock {\it The Journal of Machine Learning Research\/} {\bf 17}(1)
  2442--2471.

\bibitem[{Russo and Van~Roy(2018)}]{russo2018satisficing}
Russo, Daniel, Benjamin Van~Roy. 2018.
\newblock Satisficing in time-sensitive bandit learning.
\newblock {\it arXiv preprint arXiv:1803.02855\/} .

\bibitem[{Slivkins(2019)}]{slivkins2019intro}
Slivkins, Aleksandrs. 2019.
\newblock Introduction to multi-armed bandits.
\newblock {\it Foundations and Trends® in Machine Learning\/} {\bf 12}(1-2)
  1--286.
\newblock \doi{10.1561/2200000068}.
\newblock \urlprefix\url{http://dx.doi.org/10.1561/2200000068}.

\bibitem[{Thompson(1933)}]{thompson1933likelihood}
Thompson, William~R. 1933.
\newblock On the likelihood that one unknown probability exceeds another in
  view of the evidence of two samples.
\newblock {\it Biometrika\/} {\bf 25}(3/4) 285--294.

\bibitem[{Wang et~al.(2009)Wang, yves Audibert, and Munos}]{wang2009algorithms}
Wang, Yizao, Jean yves Audibert, R\'{e}mi Munos. 2009.
\newblock Algorithms for infinitely many-armed bandits.
\newblock D.~Koller, D.~Schuurmans, Y.~Bengio, L.~Bottou, eds., {\it Advances
  in Neural Information Processing Systems 21\/}. 1729--1736.

\bibitem[{Zhu and Nowak(2020)}]{zhu2020onregret}
Zhu, Yinglun, Robert Nowak. 2020.
\newblock On regret with multiple best arms.
\newblock H.~Larochelle, M.~Ranzato, R.~Hadsell, M.F. Balcan, H.~Lin, eds.,
  {\it Advances in Neural Information Processing Systems\/}, vol.~33. Curran
  Associates, Inc., 9050--9060.
\newblock
  \urlprefix\url{https://proceedings.neurips.cc/paper_files/paper/2020/file/670c26185a3783678135b4697f7dbd1a-Paper.pdf}.

\end{thebibliography}
